\documentclass[twoside,11pt]{article}

\usepackage{natbib}

\usepackage{subfigure}
\usepackage{amsmath,amssymb}
\usepackage{algorithmic}
\usepackage{algorithm}
\usepackage{amsthm}

\DeclareMathOperator*{\argmax}{\mathrm{argmax}}

\newcommand{\eq}[1]{(eq.~\ref{#1})}
\newtheorem{theorem}{Theorem}
\newtheorem{lemma}{Lemma}

\usepackage{cite}
\usepackage{graphicx}
\usepackage{url}

\newcommand{\perm}[1]{\hspace{-1.3mm}\parbox[c]{0.12\textwidth}{\includegraphics[width=0.12\textwidth]{icmlfigures/perms/#1}\hspace{-1.3mm}}}

\begin{document}
\title{Faster Algorithms for Max-Product Message-Passing}

%

\author{Julian~J.~McAuley\thanks{The authors are with the Statistical Machine Learning Group at NICTA, and the Research School of Information Sciences and Engineering, Australian National University. Queries should be addressed to \texttt{julian.mcauley@nicta.com.au}.} and Tib\'erio~S.~Caetano}

\maketitle

\begin{abstract}
\emph{Maximum A Posteriori} inference in graphical models is often solved via message-passing algorithms, such as the junction-tree algorithm, or loopy belief-propagation. The exact solution to this problem is well known to be exponential in the size of the model's maximal cliques after it is triangulated, while approximate inference is typically exponential in the size of the model's factors. In this paper, we take advantage of the fact that many models have maximal cliques that are larger than their constituent factors, and also of the fact that many factors consist entirely of latent variables (i.e., they do not depend on an observation). This is a common case in a wide variety of applications, including grids, trees, and ring-structured models. In such cases, we are able to decrease the exponent of complexity for message-passing by $0.5$ for both exact \emph{and} approximate inference.
\end{abstract}


\section{Introduction}

It is well-known that exact inference in \emph{tree-structured} graphical models can be accomplished efficiently by message-passing operations following a simple protocol making use of the distributive law
\citep{thegdl,factorgraphs}. It is also well-known that exact inference in \emph{arbitrary} graphical models can be solved
by the junction-tree algorithm; its efficiency is determined by the size of the maximal cliques after triangulation, a quantity related to the treewidth of the graph.


Figure \ref{fig:examples_intro} illustrates an attempt to apply the junction-tree algorithm to some graphical models containing cycles. If the graphs are not chordal ((a) and (b)), they need to be triangulated, or made chordal (red edges in (c) and (d)). Their clique-graphs are then guaranteed to be \emph{junction-trees}, and the distributive law can be applied with the same protocol used for trees; see \citet{thegdl} for a beautiful tutorial on exact inference in arbitrary graphs. Although the models in this example contain only pairwise factors, triangulation has increased the size of their maximal cliques, making exact inference substantially more expensive. Hence approximate solutions in the original graph (such as loopy belief-propagation, or inference in a loopy factor-graph) are often preferred over an exact solution via the junction-tree Algorithm.

Even when the model's factors are the same size as its maximal cliques, neither exact nor approximate inference algorithms take advantage of the fact that many factors consist only of \emph{latent} variables. In many models, those factors that are conditioned upon the observation contain fewer latent variables than the purely latent cliques. Examples are shown in Figure \ref{fig:examps}. This encompasses a wide variety of models, including grid-structured models for optical flow and stereo disparity as well as chain and tree-structured models for text or speech.

\begin{figure}[ht]
\footnotesize
\begin{center}
 \begin{tabular}{cccc}
  \includegraphics[scale=0.25]{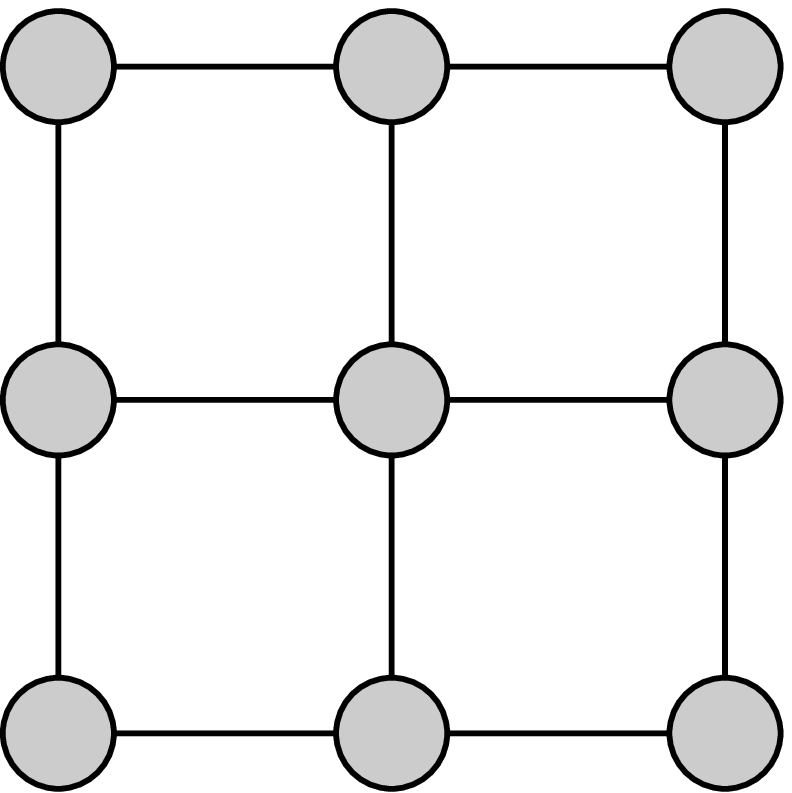} & \includegraphics[scale=0.25]{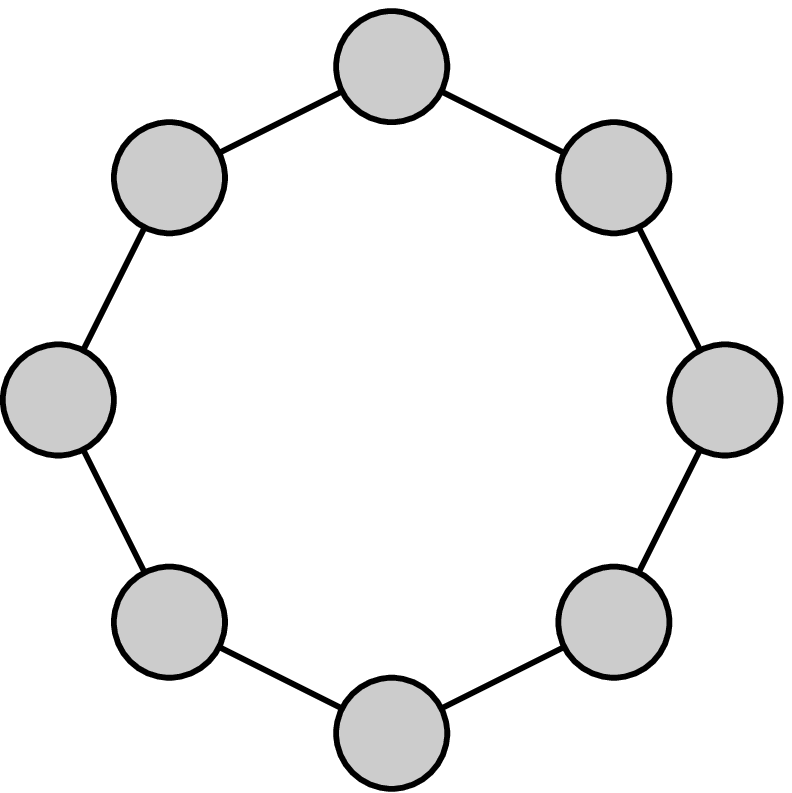} & \includegraphics[scale=0.25]{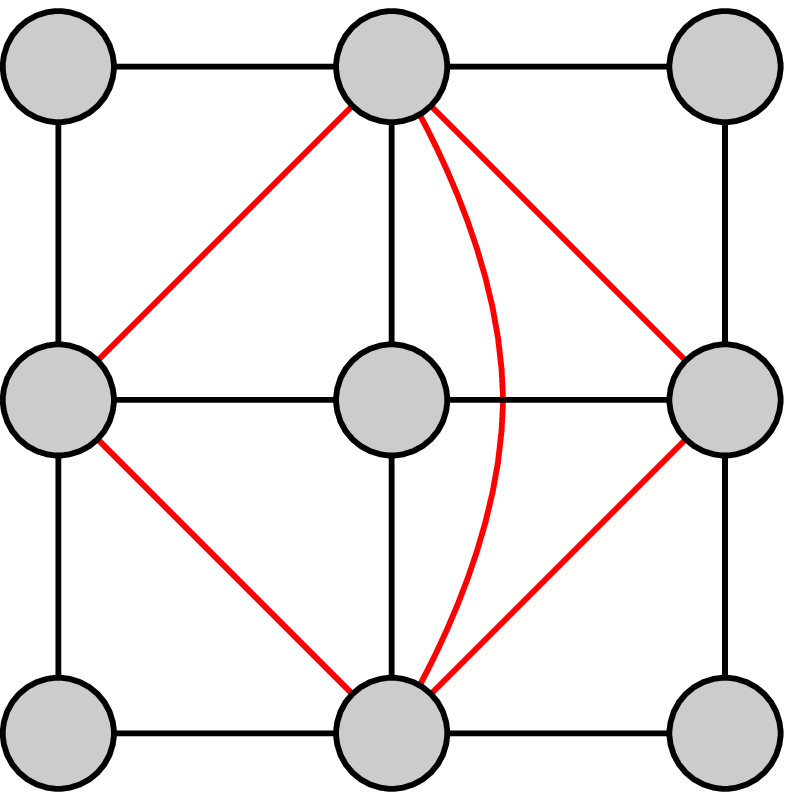} & \includegraphics[scale=0.25]{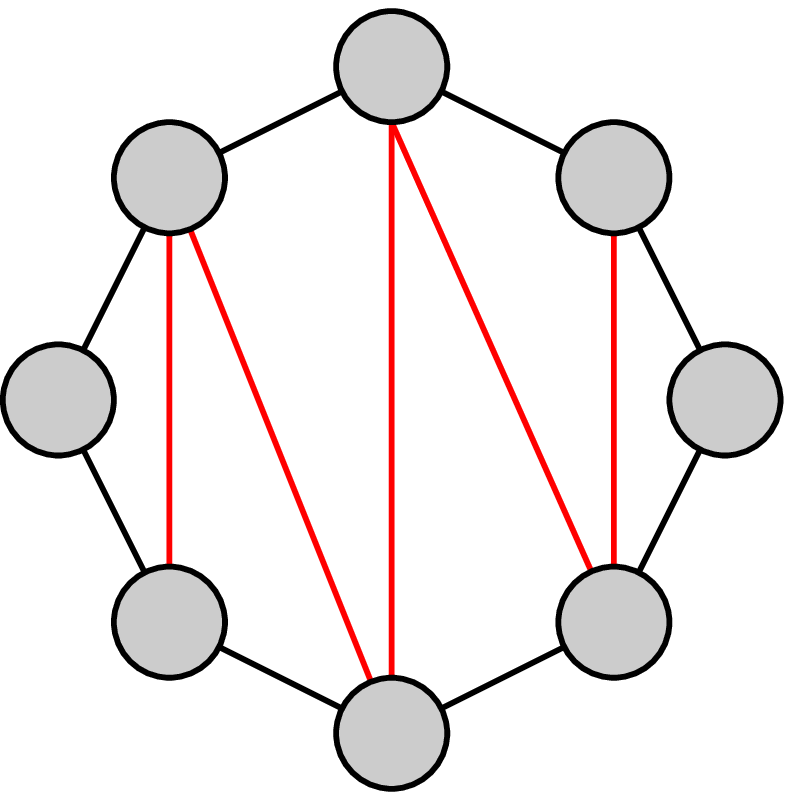}\\
  (a) & (b) & (c) & (d)\\
 \end{tabular}
\end{center}
\vspace{-3mm}
\caption{The models at left ((a) and (b)) can be triangulated  ((c) and (d)) so that the junction-tree algorithm can be applied. Despite the fact that the new models have larger maximal cliques, the corresponding potentials are still factored over pairs of nodes only. Our algorithms exploit this fact.}
\label{fig:examples_intro}
\end{figure}

\begin{figure}
\begin{center}
 \begin{tabular}{cccc}
\parbox[c]{0.18\columnwidth}{\centering\includegraphics[scale=0.35]{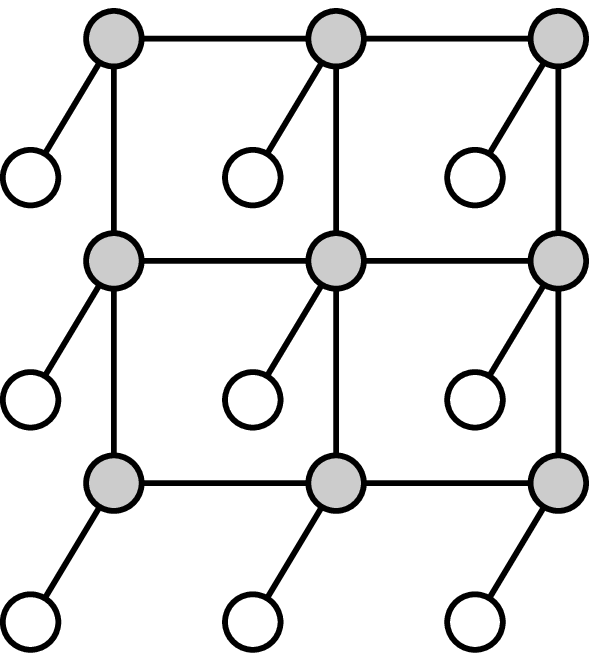}} & \parbox[c]{0.18\columnwidth}{\centering\includegraphics[scale=0.35]{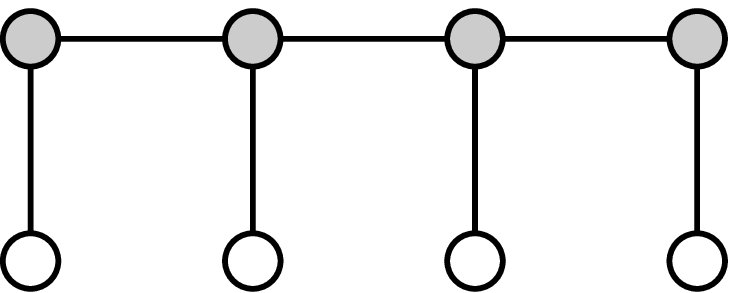}} & \parbox[c]{0.18\columnwidth}{\centering\includegraphics[scale=0.35]{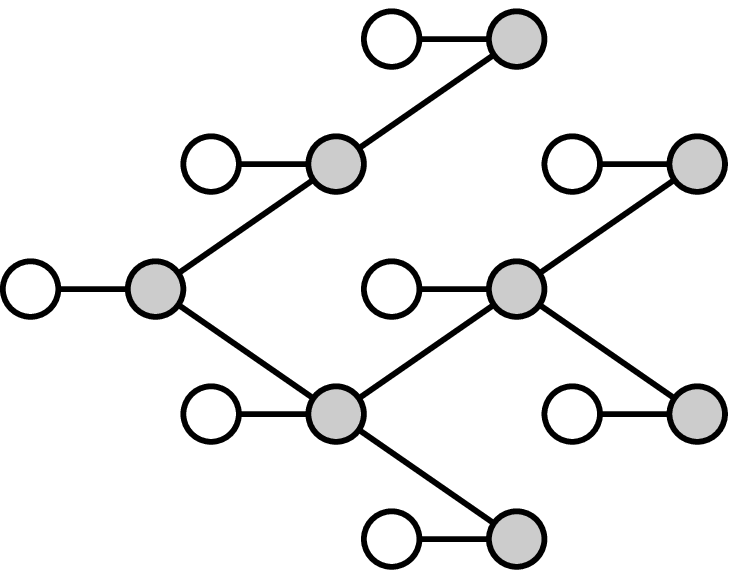}} & \parbox[c]{0.18\columnwidth}{\centering\includegraphics[scale=0.35]{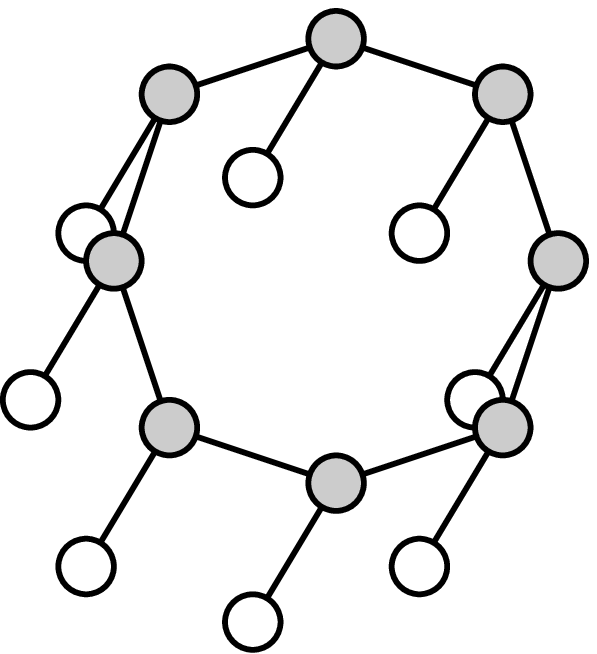}}\\
(a) & (b) & (c) & (d)
 \end{tabular}
\end{center}
 \caption{Some graphical models to which our results apply: \emph{cliques containing observations have fewer latent variables than purely latent cliques}. White nodes correspond to the observation, gray nodes to the labeling. In other words, cliques containing a white node encode the \emph{data likelihood}, whereas cliques containing only gray nodes encode \emph{priors}.
}
\label{fig:examps}
\end{figure}

In this paper, we exploit the fact that the maximal cliques (after triangulation) often have potentials that factor over subcliques, as illustrated in Figure \ref{fig:examples_intro}. We will show that whenever this is the case, the expected computational complexity of exact inference \emph{can be improved} (both the asymptotic upper-bound and the actual runtime).

Additionally, we will show that this result can be applied so long as \emph{those cliques that are conditioned upon an observation} contain fewer latent variables than those cliques consisting of purely latent variables; the `purely latent' cliques can be pre-processed \emph{offline}, allowing us to achieve the same benefits as described in the previous paragraph.

We show that these properties reveal themselves in a wide variety of real applications. Both of our improvements shall increase the class of problems for which inference via max-product belief-propagation is tractable.


A core operation encountered in the junction-tree algorithm is that of finding the index that chooses the largest product amongst two lists of length $N$:
\begin{equation}
 \hat{i} = \argmax_{i\in\lbrace 1 \ldots N \rbrace} \left\lbrace \mathbf{v}_a[i] \times \mathbf{v}_b[i] \right\rbrace.
\label{eq:hati}
\end{equation}
Our results stem from the realization that while \eq{eq:hati} appears to be a \emph{linear} time operation, it can be decreased to $O(\sqrt{N})$ (in the expected case) if we know the permutations that sort $\mathbf{v}_a$ and $\mathbf{v}_b$. 


A preliminary version of this work appeared in \citet{McACae10}.

\subsection{Summary of Results}

A selection of the results to be presented in the remainder of this paper can be summarized as follows:
\begin{itemize}
 \item We are able to lower the asymptotic expected running time of the max-product belief-propagation for \emph{any} graphical model whose cliques factorize into lower-order terms.
 \item The results obtained are exactly those that would be obtained by the traditional version of the algorithm, i.e., no approximations are used.
 \item Our algorithm also applies whenever cliques containing an observed variable contain fewer latent variables than purely latent cliques, as in Figure \ref{fig:examps} (meaning that certain computations can be taken offline).
 \item For any cliques composed of pairwise factors, we obtain an expected speed-up of \emph{at least} $\Omega(\sqrt{N})$ (assuming $N$ states per node; $\Omega$ denotes an \emph{asymptotic lower-bound}).
 \item For example, in models with third-order cliques containing pairwise terms, message-passing is reduced from $\Theta(N^3)$ to $O(N^2\sqrt{N})$, as in Figure \ref{fig:examples_intro}(d). For models containing pairwise (but purely latent) cliques, message-passing is reduced from $\Theta(N^2)$ to $O(N\sqrt{N})$, as in Figure \ref{fig:examps}.
 \item For cliques composed of $K$-ary factors, the expected speed-up generalizes to at least $\Omega(\frac{1}{K}N^\frac{1}{K})$, though it is \emph{never asymptotically slower} than the original solution.
 \item The expected-case improvement is derived under the assumption that the order-statistics of different factors are \emph{independent}.
 \item If the different factors have `similar' order statistics, the performance will be better than the expected case.
 \item If the different factors have `opposite' order statistics, the performance will be worse than the expected case, but is never asymptotically more expensive than the traditional version of the algorithm.
\end{itemize}

Our results do not apply for every semiring $S(+,\cdot)$, but only to those whose `addition' operation defines an order (for example, $\min$ or $\max$); we also assume that under this ordering, our `multiplication' operator satisfies
\begin{equation}
a < b \wedge c < d ~~\Rightarrow~~ a \cdot c < b \cdot d.
\label{eq:condition}
\end{equation}
Thus our results certainly apply to the \emph{max-sum} and \emph{min-sum} semirings (as well as \emph{max-product} and \emph{min-product}, assuming non-negative potentials), but not for \emph{sum-product} (for example). Consequently, our approach is useful for computing MAP-states, but cannot be used to compute marginal distributions. We also assume that the domain of each node is \emph{discrete}. 

We shall initially present our algorithm as it applies to models of the type shown in Figure \ref{fig:examples_intro}. The more general (and arguably more useful) application of our algorithm to those models in Figure \ref{fig:examps} shall be deferred until Section \ref{sec:latent}, where it can be seen as a straightforward generalization of our initial results.

\subsection{Related Work}

There has been previous work on speeding-up message-passing algorithms by exploiting some type of structure in certain graphical models. For example, \citet{KerAhmNat09} study the case where different cliques share the same potential function. In \citet{pedro_bp}, fast message-passing algorithms are provided for cases in which the potential of a 2-clique is only dependent on the \emph{difference} of the latent variables (which is common in some computer vision applications); they also show how the algorithm can be made faster if the graphical model is a bipartite graph. In \citet{KumTor}, the authors provide faster algorithms for the case in which the potentials are \emph{truncated}, whereas in \citet{PetFehBur08} the authors offer speed-ups for models that are specifically grid-like.

The latter work is perhaps the most similar in spirit to ours, as it exploits the fact that certain factors can be \emph{sorted} in order to reduce the search space of a certain maximization problem. In practice, this leads to linear speed-ups over a $\Theta(N^4)$ algorithm.

Another closely related paper is that of \citet{jointree}. This work can be seen to compliment ours in the sense that it exploits essentially the same type of factorization that we study, though it applies to \emph{sum-product} versions of the algorithm, rather than the \emph{max-product} version that we shall study. \citet{nested} also exploits factorization within cliques of junction-trees, albeit a different type of factorization than that studied here.

In Section \ref{sec:optimizing}, we shall see that our algorithm is closely related to a well-studied problem known as `funny matrix multiplication' \citep{Kerr70}. The worst-case complexity of this problem has been studied in relation to the all-pairs shortest path problem \citep{allpsp,karger}.

\section{Background}
\label{sec:background}

The notation we shall use is briefly defined in Table \ref{tab:definitions}. We shall assume throughout that the \emph{max-product} semiring is being used, though our analysis is almost identical for any suitable choice.

\begin{table}[t]
  \caption{Notation}
  \label{tab:definitions}
\begin{center}
\ \\
\begin{tabular}{p{0.25\columnwidth} | p{0.58\columnwidth} }
 \hline
\textbf{Example} & \textbf{description}\\
\hline
\hline
$A; B$ & capital letters refer to sets of nodes (or similarly, cliques);\\
$A \cup B; A \cap B; A \setminus B$ & standard set operators are used ($A \setminus B$ denotes set difference);\\
$\text{dom}(A)$ & the domain of a set; this is just the Cartesian product of the domains of each element in the set;\\
$\mathbf{P}$ & bold capital letters refer to arrays;\\
$\mathbf{x}$ & bold lower-case letters refer to vectors;\\
$\mathbf{x}[a]$ & vectors are indexed using square brackets;\\
$\mathbf{P}[n]$ & similarly, square brackets are used to index a \emph{row} of a 2-d array,\\
$\mathbf{P}[\mathbf{n}]$ & or a row of an $(|\mathbf{n}| + 1)$-dimensional array;\\
$\mathbf{P}^X; \mathbf{v}^a$ & superscripts are just labels, i.e., $\mathbf{P}^X$ is an array, $\mathbf{v}^a$ is a vector;\\
$\mathbf{v}_a$ & \emph{constant} subscripts are also labels, i.e., if $a$ is a constant, then $\mathbf{v}_a$ is a constant vector;\\
$x_i; \mathbf{x}_A$ & \emph{variable} subscripts define variables; the subscript defines the domain of the variable;\\
$\mathbf{n}|_X$ & if $\mathbf{n}$ is a constant vector, then $\mathbf{n}|_X$ is the \emph{restriction} of that vector to those indices corresponding to variables in $X$ (assuming that $X$ is an ordered set);\\
$\Phi_A; \Phi_A(\mathbf{x}_A)$ & a function over the variables in a set $A$; the argument $\mathbf{x}_A$ will be suppressed if clear, given that `functions' are essentially arrays for our purposes;\\
$\Phi_{i,j}(x_i, x_j)$ & a function over a pair of variables $(x_i, x_j)$;\\
$\Phi_A(\mathbf{n}|_B; \mathbf{x}_{A \setminus B})$ & if one argument to a function is constant (here $\mathbf{n}|_B$), then it becomes a function over fewer variables (in this case, only $\mathbf{x}_{A \setminus B}$ is free);\\
 \hline
\end{tabular}
\end{center}

\end{table}

MAP-inference in a graphical model $\mathcal G$ consists of solving an optimization problem of the form
\begin{equation}
 \mathbf{\hat{x}} = \argmax_{\mathbf{x}} \prod_{C \in \mathcal C} \Phi_C(\mathbf{x}_C),
\label{eq:map}
\end{equation}
where $\mathcal C$ is the set of maximal cliques in $\mathcal G$. This problem is often solved via \emph{message-passing} algorithms such as the junction-tree algorithm, loopy belief-propagation, or inference in a factor graph \citep{thegdl,Weiss00,factorgraphs}.

Two of the fundamental steps encountered in message-passing algorithms are defined below. Firstly, the message from a clique $X$ to an intersecting clique $Y$ is defined by
\begin{equation}
 m_{X \rightarrow Y}(\mathbf{x}_{X \cap Y}) = \max_{\mathbf{x}_{X \setminus Y}} \left\lbrace \Phi_X(\mathbf{x}_X) \!\!\!\!\! \prod_{Z\in \Gamma(X) \setminus Y} \!\!\!\!\! m_{Z \rightarrow X}(\mathbf{x}_{X \cap Z}) \right\rbrace
\label{eq:message}
\end{equation}
(where $\Gamma(X)$ returns the neighbors of the clique $X$). If such messages are computed after $Y$ has received messages from all of its neighbors except $X$ (i.e., $\Gamma(X) \setminus Y$), then this defines precisely the update scheme used by the junction-tree algorithm. The same update scheme is used for loopy belief-propagation, though it is done iteratively in a randomized fashion.

Secondly, after all messages have been passed, the MAP-states for a subset of nodes $M$ (assumed to belong to a clique $X$) is computed using
\begin{equation}
 m_M(\mathbf{x}_M) = \max_{\mathbf{x}_{X \setminus M}} \left\lbrace \Phi_X(\mathbf{x}_X) \!\!\! \prod_{Z\in \Gamma(X)} \!\!\! m_{Z \rightarrow X}(\mathbf{x}_{X \cap Z}) \right\rbrace.
\label{eq:marginal}
\end{equation}

Often, the clique-potential $\Phi_X(\mathbf{x}_X)$ shall be decomposable into several smaller factors, i.e.,
\begin{equation}
 \Phi_X(\mathbf{x}_X) = \prod_{F \subset X} \Phi_{F}(\mathbf{x}_{F}).
\end{equation}
Some simple motivating examples are shown in Figure \ref{fig:tree3}: a model for pose estimation from \citet{sigal06}, a `skip-chain CRF' from \citet{galley06}, and a model for shape matching from \citet{lbpmatch}. In each case, the triangulated model has third-order cliques, but the potentials are only pairwise. Other examples have already been shown in Figure \ref{fig:examples_intro}; analogous cases are ubiquitous in many real applications.

\begin{figure}
 \begin{center}
\begin{tabular}{ccc}
  \parbox[c]{0.3\columnwidth}{\centering\includegraphics[scale=0.25]{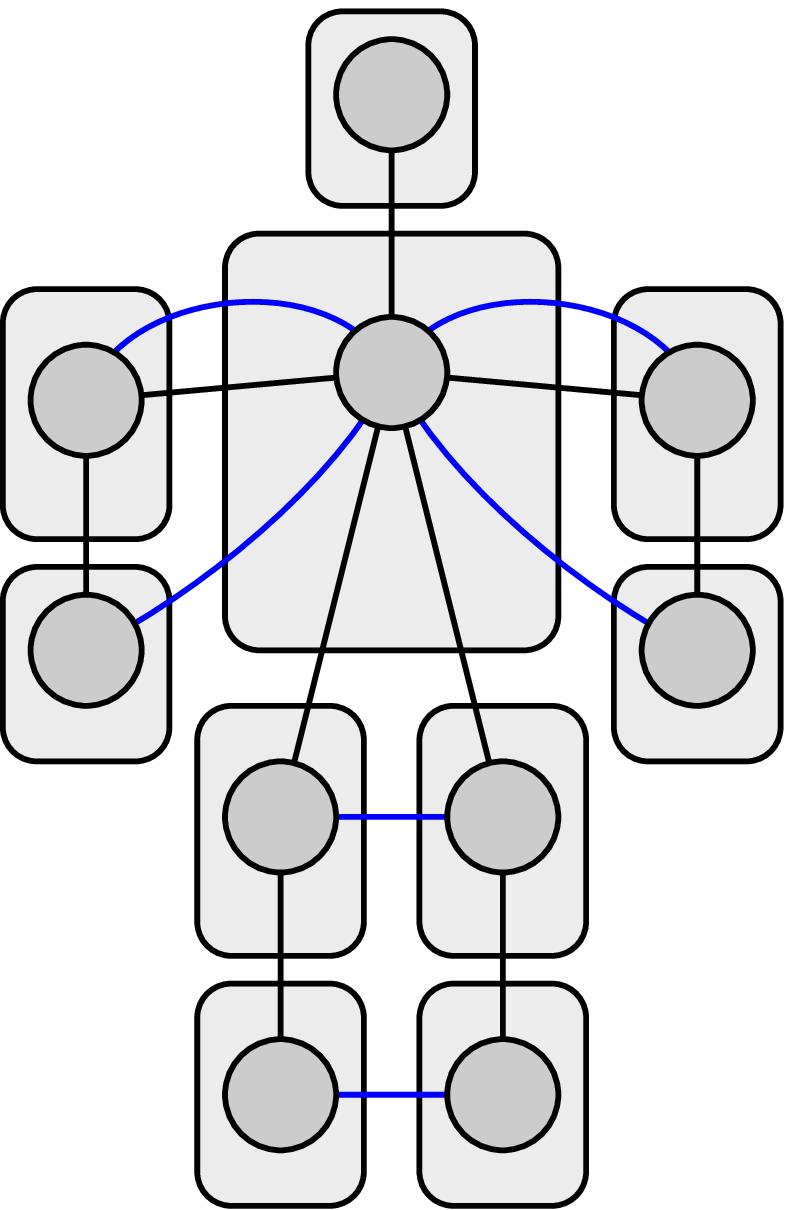}\\\vspace{1mm}} & \parbox[c]{0.3\columnwidth}{\centering\includegraphics[scale=0.25]{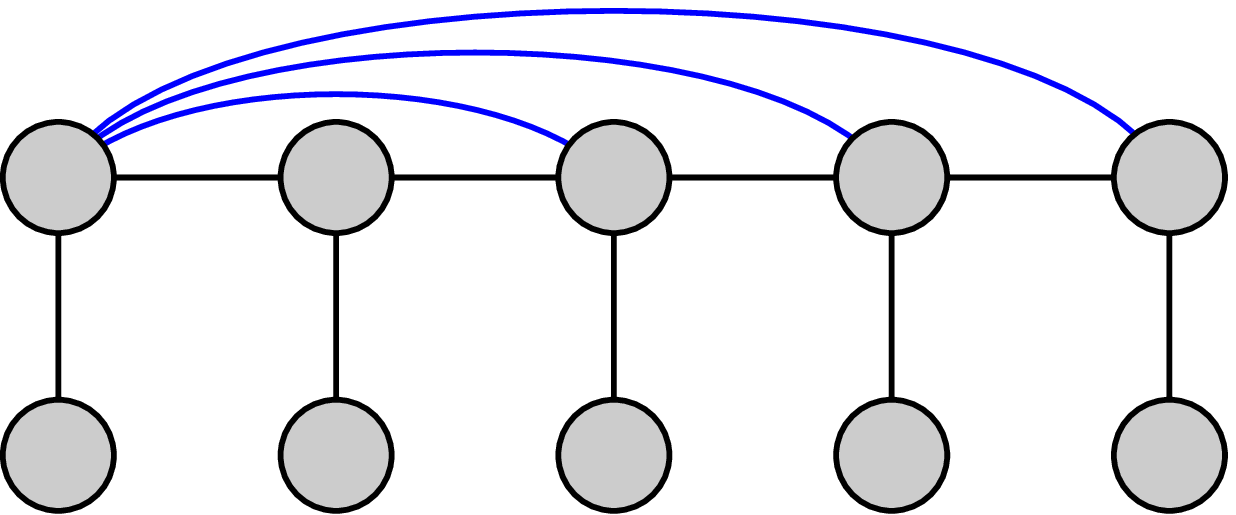}} & \parbox[c]{0.3\columnwidth}{\centering\includegraphics[scale=0.25]{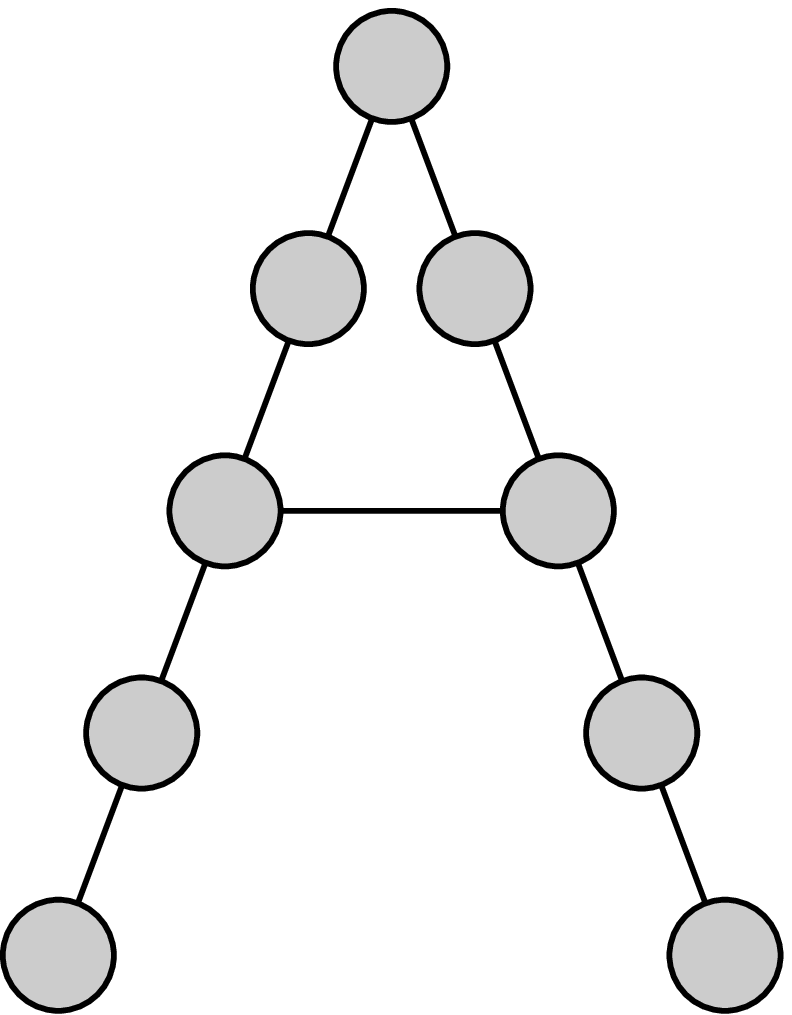}}\\
(a) & (b) & (c)
\end{tabular}
 \end{center}
\vspace{-3mm}
\caption{(a) A model for pose reconstruction from \citet{sigal06}; (b) A `skip-chain CRF' from \citet{galley06}; (c) A model for deformable matching from \citet{lbpmatch}. Although the (triangulated) models have cliques of size three, their potentials factorize into pairwise terms.}
\label{fig:tree3}
\end{figure}


The optimizations we suggest shall apply to general problems of the form
\begin{equation}
 m_M(\mathbf{x}_M) = \max_{\mathbf{x}_{X \setminus M}} \prod_{F \subset X} \Phi_{F}(\mathbf{x}_{F}),
\end{equation}
which subsumes both \eq{eq:message} and \eq{eq:marginal}, where we simply treat the messages as factors of the model. Algorithm \ref{alg:brute} gives the traditional solution to this problem, which does not exploit the factorization of $\Phi_X(\mathbf{x}_X)$. This algorithm runs in $\Theta(N^{|X|})$, where $N$ is the number of states per node, and $|X|$ is the size of the clique $X$ (we assume that for a given $\mathbf{x}_X$, computing $\prod_{F \subset X} \Phi_{F}(\mathbf{x}_{F})$ takes constant time, as our optimizations shall not modify this cost).

\begin{algorithm}
 \caption{Brute-force computation of max-marginals}
 \label{alg:brute}
\begin{algorithmic}[1]
 \REQUIRE a clique $X$ whose max-marginal $m_M(\mathbf{x}_M)$ (where $M \subset X$) we wish to compute; assume that each node in $X$ has domain $\left\lbrace 1 \ldots N\right\rbrace$
 \FOR{$\mathbf{m} \in \text{dom}(M)$ \COMMENT{i.e., $\left\lbrace 1 \ldots N\right\rbrace^{|M|}$}}
 \STATE $\mathit{max} := -\infty$
 \FOR{$\mathbf{y} \in \text{dom}(X \setminus M)$}
 \IF {$\prod_{F \subset X} \Phi_{F}(\mathbf{m}|_{F}; \mathbf{y}|_{F}) > max$}
   \STATE $max := \prod_{F \subset X} \Phi_{F}(\mathbf{m}|_{F}; \mathbf{y}|_{F})$
 \ENDIF
 \ENDFOR\ \COMMENT{this loop takes $\Theta(N^{|X \setminus M|})$}
 \STATE $m_M(\mathbf{m}) := \mathit{max}$
 \ENDFOR\ \COMMENT{this loop takes $\Theta(N^{|X|})$}
 \RETURN $m_M$
\end{algorithmic}
\end{algorithm}



\section{Optimizing Algorithm \ref{alg:brute}}
\label{sec:optimizing}

In order to specify a more efficient version of Algorithm \ref{alg:brute}, we begin by considering the simplest possible nontrivial factorization: a clique of size three containing pairwise factors. In such a case, our aim is to compute
\begin{equation}
m_{i,j}(x_i,x_j) = \max_{x_k} \Phi_{i,j,k}(x_i, x_j, x_k),
\label{eq:max3a}
\end{equation}
which we have assumed takes the form
\begin{equation}
m_{i,j}(x_i,x_j) =\max_{x_k} \Phi_{i,j}(x_i, x_j)\times\Phi_{i,k}(x_i, x_k)\times\Phi_{j,k}(x_j, x_k).
\end{equation}
For a particular value of $(x_i, x_j) = (a, b)$, we must solve
\begin{equation}
 m_{i,j}(a,b) = \Phi_{i,j}(a,b) \times \max_{x_k} \underbrace{\Phi_{i,k}(a,x_k)}_{\mathbf{v}_a} \times \underbrace{\Phi_{j,k}(b,x_k)}_{\mathbf{v}_b},
\label{eq:max3}
\end{equation}
which we note is in precisely the form shown in \eq{eq:hati}.

There is a close resemblance between \eq{eq:max3} and the problem of multiplying two matrices: if the `$\max$' in \eq{eq:max3} is replaced by summation, we essentially recover traditional matrix multiplication. While traditional matrix multiplication is well known to have a sub-cubic worst-case solution \citep[see][]{Strassen69}, the version in \eq{eq:max3} (often referred to as `funny matrix multiplication', or simply `max-product matrix multiplication') is known to be cubic in the worst case, assuming that only multiplication and comparison operations are used \citep{Kerr70}. The complexity of solving \eq{eq:max3} can also be shown to be equivalent to the all-pairs shortest path problem, which is studied in \citet{allpsp}. Not surprisingly, we shall not improve the worst-case complexity, but shall instead give far better \emph{expected-case} performance than existing solutions. Just as Strassen's algorithm can be used to solve \eq{eq:max3} when maximization is replaced by summation, there has been work studying the problem of \emph{sum-product} inference in graphical models, subject to the same type of factorization we discuss \citep{jointree}.

As we have previously suggested, it will be possible to solve \eq{eq:max3} efficiently if $\mathbf{v}_a$ and $\mathbf{v}_b$ are already sorted. We note that $\mathbf{v}_a$ will be reused for every value of $x_j$, and likewise $\mathbf{v}_b$ will be reused for every value of $x_i$. Sorting every row of $\Phi_{i,k}$ and $\Phi_{j,k}$ can be done in $\Theta(N^2\log N)$ (for $2N$ rows of length $N$).

The following elementary lemma is the key observation required in order to solve \eq{eq:max3} efficiently:

\begin{lemma}
If the $p^\text{th}$ largest element of $\mathbf{v}_a$ has the same index as the $q^\text{th}$ largest element of $\mathbf{v}_b$, then we only need to search through the $p$ largest values of $\mathbf{v}_a$, and the $q$ largest values of $\mathbf{v}_b$; any corresponding pair of smaller values could not possibly be the largest solution.
\label{main_lemma}
\end{lemma}

This observation is used to construct Algorithm \ref{alg1}. Here we iterate through the indices starting from the largest values of $\mathbf{v}_a$ and $\mathbf{v}_b$, stopping once both indices are `behind' the maximum value found so far (which we then know is the maximum). This algorithm is demonstrated pictorially in Figure \ref{fig:alg1}.

\begin{algorithm}
  \caption{Find $i$ such that $\mathbf{v}_a[i] \times \mathbf{v}_b[i]$ is maximized}
  \label{alg1}
\begin{algorithmic}[1]
\REQUIRE two vectors $\mathbf{v}_a$ and $\mathbf{v}_b$, and permutation functions $p_a$ and $p_b$ that sort them in decreasing order (so that $\mathbf{v}_a[p_a[1]]$ is the largest element in $\mathbf{v}_a$)
\STATE \textbf{Initialize:} $\mathit{start} := 1$,
$\mathit{end}_a := p_a^{-1}[p_b[1]]$, $\mathit{end}_b := p_b^{-1}[p_a[1]]$
\COMMENT{if $end_b=k$, then the largest element in $\mathbf{v}_a$ has the same index as the $k^{\text{th}}$ largest element in $\mathbf{v}_b$}\\
\STATE $\mathit{best} := p_a[1]$, $\mathit{max} := \mathbf{v}_a[\mathit{best}]\times\mathbf{v}_b[\mathit{best}]$
\IF {$\mathbf{v}_a[p_b[1]]\times\mathbf{v}_b[p_b[1]] > \mathit{max}$}
\STATE $\mathit{best} := p_b[1]$, $\mathit{max} := \mathbf{v}_a[\mathit{best}]\times\mathbf{v}_b[\mathit{best}]$
\ENDIF
\WHILE{$\mathit{start} < \mathit{end}_a$\ \COMMENT{in practice, we could also stop if $\mathit{start} < \mathit{end}_b$, but the version given here is the one used for analysis in Appendix \ref{sec:analysis}}}\label{line2}
\STATE $\mathit{start} := \mathit{start} + 1$
\IF {$\mathbf{v}_a[p_a[\mathit{start}]]\times\mathbf{v}_b[p_a[\mathit{start}]] > \mathit{max}$} \label{if1}
\STATE $\mathit{best} := p_a[\mathit{start}]$
\STATE $\mathit{max} := \mathbf{v}_a[\mathit{best}]\times\mathbf{v}_b[\mathit{best}]$
\ENDIF
\IF {$p_b^{-1}[p_a[\mathit{start}]] < \mathit{end}_b$}
\STATE $\mathit{end}_b := p_b^{-1}[p_a[\mathit{start}]]$
\ENDIF \label{endif1}
\STATE \COMMENT{repeat Lines \ref{if1}--\ref{endif1}, interchanging $a$ and $b$} \label{line:repeat}
\ENDWHILE\ \COMMENT{this takes \emph{expected time} $O(\sqrt{N})$}
\RETURN $\mathit{best}$
\end{algorithmic}
\end{algorithm}

\begin{figure}
\footnotesize
\begin{flushright}
$\text{Step 1:} \left\lbrace\text{\parbox{0.77\columnwidth}{\includegraphics[angle=-90,width=0.6\columnwidth]{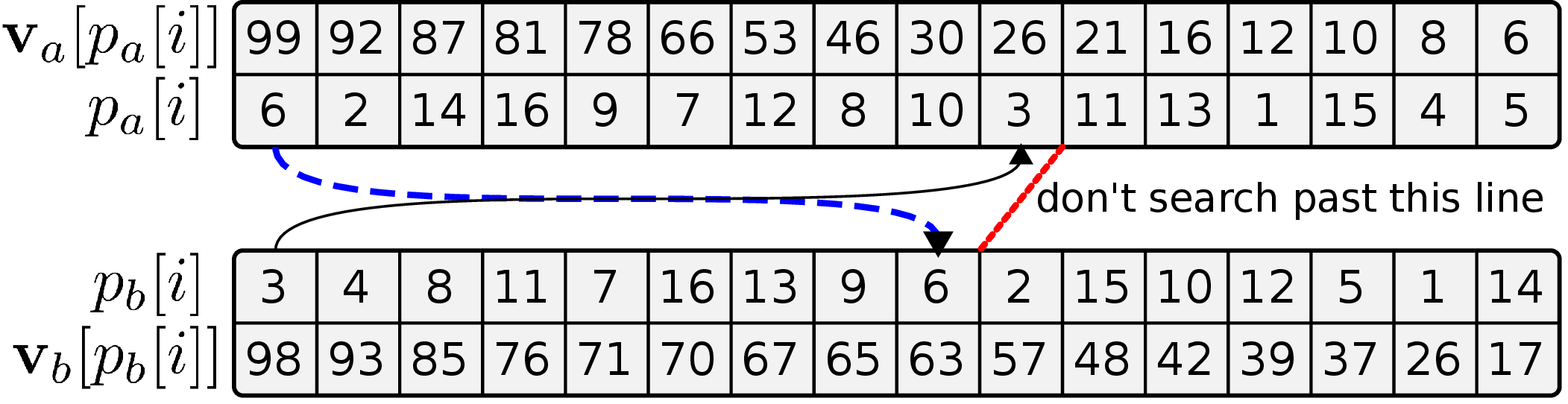}}}\right.$\parbox{0.01\columnwidth}{\ }\\

\ \\
\ \\
\ \\

$\text{Step 2:} \left\lbrace\text{\parbox{0.77\columnwidth}{\includegraphics[width=0.6\columnwidth]{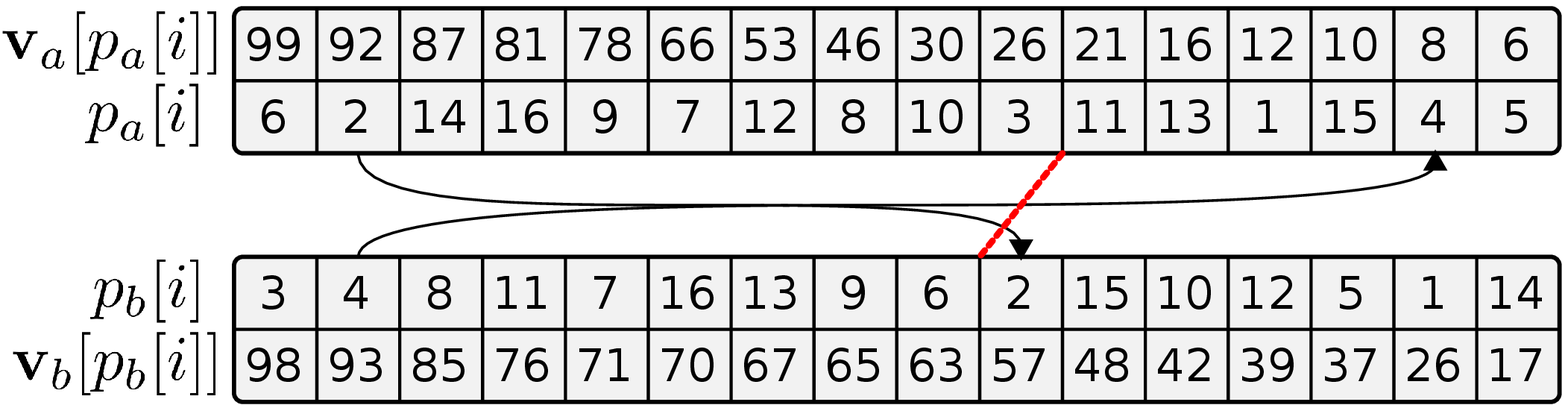}}}\right.$\parbox{0.01\columnwidth}{\ }\\

\ \\
\ \\
\ \\

$\text{Step 3:} \left\lbrace\text{\parbox{0.77\columnwidth}{\includegraphics[width=0.6\columnwidth]{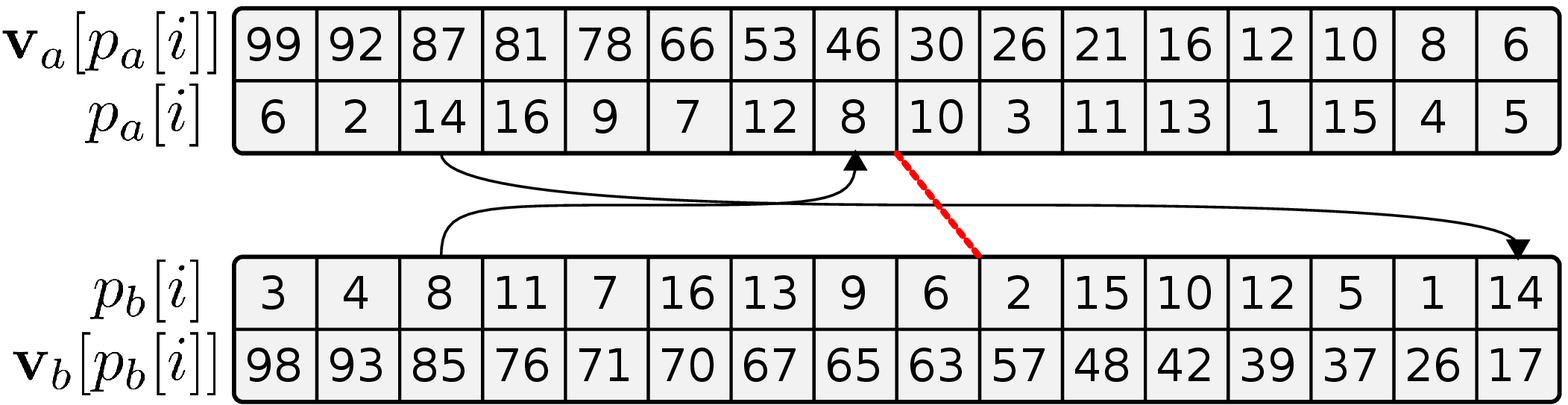}}}\right.$\parbox{0.01\columnwidth}{\ }\\

\ \\
\ \\
\ \\

$\text{Step 4:} \left\lbrace\text{\parbox{0.77\columnwidth}{\includegraphics[width=0.6\columnwidth]{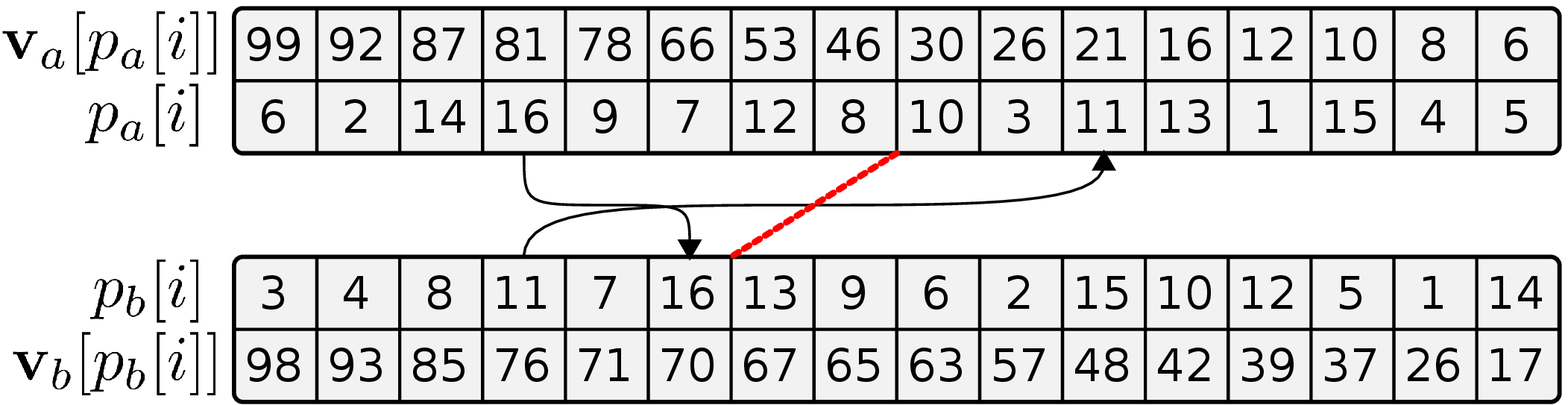}}}\right.$\parbox{0.01\columnwidth}{\ }\\

\ \\
\ \\
\ \\

$\text{Step 5:} \left\lbrace\text{\parbox{0.77\columnwidth}{\includegraphics[width=0.6\columnwidth]{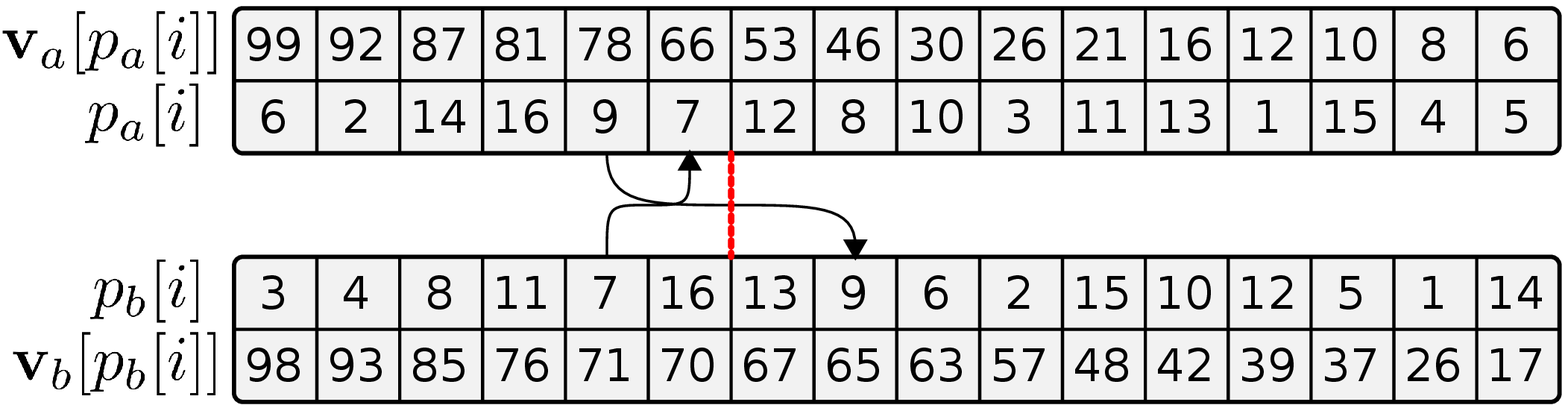}}}\right.$\parbox{0.01\columnwidth}{\ }\\

\end{flushright}
 \caption{Algorithm \ref{alg1}, explained pictorially. The arrows begin at $p_a[\mathit{start}]$ and $p_b[\mathit{start}]$; the red line connects $\mathit{end}_a$ and $\mathit{end}_b$, behind which we need not search; a dashed arrow is used when a new maximum is found. Note that in the event that $\mathbf{v}_a$ and $\mathbf{v}_b$ contain repeated elements, they can be sorted arbitrarily.}
\label{fig:alg1}
\end{figure}

A prescription of how Algorithm \ref{alg1} can be used to solve \eq{eq:max3a} is given in Algorithm \ref{alg:3clique}. Determining precisely the running time of Algorithm \ref{alg1} (and therefore Algorithm \ref{alg:3clique}) is not trivial, and will be explored in depth in Appendix \ref{sec:analysis}. We note that if the expected-case running time of Algorithm \ref{alg1} is $O(f(N))$, then the time taken to solve Algorithm \ref{alg:3clique} shall be $O(N^2(\log N + f(N)))$. At this stage we shall state an upper-bound on the true complexity in the following theorem:

\begin{theorem}
 The \emph{expected} running time of Algorithm \ref{alg1} is $O(\sqrt{N})$, yielding a speed-up of at least $\Omega(\sqrt{N})$ in cliques containing pairwise factors.
\label{the:alg1}
\end{theorem}

\begin{algorithm}
 \caption{Use Algorithm \ref{alg1} to compute the max-marginal of a 3-clique containing pairwise factors}
 \label{alg:3clique}
\begin{algorithmic}[1]
 \REQUIRE a potential $\Phi_{i,j,k}(a,b,c) = \Phi_{i,j}(a,b) \times \Phi_{i,k}(a,c) \times \Phi_{j,k}(b,c)$ whose max-marginal $m_{i,j}(x_i, x_j)$ we wish to compute
\FOR{$n \in \left\lbrace 1 \ldots N \right\rbrace$}
\STATE compute $\mathbf{P}^i[n]$ by sorting $\Phi_{i,k}(n, x_k)$ \COMMENT{takes $\Theta(N\log N)$} \label{line:sort1}
\STATE compute $\mathbf{P}^j[n]$ by sorting $\Phi_{j,k}(n, x_k)$ \COMMENT{$\mathbf{P}^i$ and $\mathbf{P}^j$ are $N\times N$ arrays, each row of which is a permutation; $\Phi_{i, k}(n,x_k)$ and $\Phi_{j, k}(n,x_k)$ are functions over $x_k$, since $n$ is constant in this expression} \label{line:sort2}
\ENDFOR\ \COMMENT{this loop takes $\Theta(N^2\log N)$}
\FOR{$(a,b) \in \left\lbrace 1\ldots N \right\rbrace^2$}
\STATE $\left(\mathbf{v}_a, \mathbf{v}_b\right) := \left(\Phi_{i,k}(a,x_k), \Phi_{j,k}(b,x_k) \right)$
\STATE $\left( p_a, p_b\right) := \left( \mathbf{P}^i[a], \mathbf{P}^j[b]\right)$
\STATE $\mathit{best} := \mathit{Algorithm\ref{alg1}}\left(\mathbf{v}_a, \mathbf{v}_b, p_a, p_b\right)$ \COMMENT{takes $O(\sqrt{N})$}
\STATE $m_{i,j}(a,b) := \Phi_{i,j}(a,b) \times \Phi_{i,k}(a,\mathit{best}) \times \Phi_{j,k}(b,\mathit{best})$
\ENDFOR\ \COMMENT{this loop takes $O(N^2\sqrt{N})$}\\ \COMMENT{the total running time is $O(N^2\log N + N^2\sqrt{N})$, which is dominated by $O(N^2\sqrt{N})$}
\RETURN $m_{i,j}$
\end{algorithmic}
\end{algorithm}

\subsection{An Extension to Higher-Order Cliques with Three Factors}
\label{sec:3factors}

The simplest extension that we can make to Algorithms \ref{alg1} and \ref{alg:3clique} is to note that they can be applied even when there are several overlapping terms in the factors. For instance, Algorithm \ref{alg:3clique} can be adapted to solve
\begin{equation}
m_{i,j}(x_i,x_j) =\max_{x_k, x_m} \Phi_{i,j}(x_i, x_j)\times
\Phi_{i,k,m}(x_i, x_k, x_m)\times\Phi_{j,k,m}(x_j, x_k, x_m),
\label{eq:shared}
\end{equation}
and similar variants containing three factors. Here both $x_k$ and $x_m$ are shared by $\Phi_{i,k,m}$ and $\Phi_{j,k,m}$. We can follow precisely the reasoning of the previous section, except that when we sort $\Phi_{i,k,m}$ (similarly $\Phi_{j,k,m}$) for a fixed value of $x_i$, we are now sorting an \emph{array} rather than a \emph{vector} (Algorithm \ref{alg:3clique}, Lines \ref{line:sort1} and \ref{line:sort2}); in this case, the permutation functions $p_a$ and $p_b$ in Algorithm \ref{alg1} simply return \emph{pairs} of indices. This is illustrated in Figure \ref{fig:alg1_multiple}.
Effectively, in this example we are sorting the variable $x_{k,m}:=x_k\otimes x_m$, which has state space of size $N^2$.
\begin{figure}
\footnotesize
\begin{flushright}
$\text{Step 1:} \left\lbrace\text{\parbox{0.77\columnwidth}{\includegraphics[width=0.6\columnwidth]{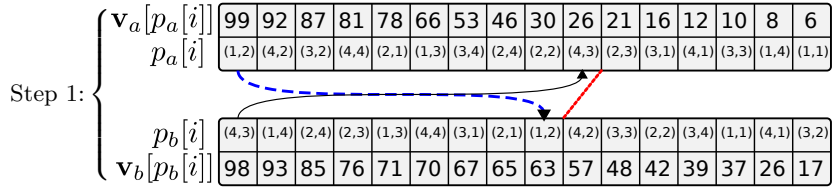}}}\right.$\parbox{0.01\columnwidth}{\ }\\
\end{flushright}
 \caption{The reasoning applied in Algorithm \ref{alg1} applies even when the elements of $p_a$ and $p_b$ are multidimensional indices.}
\label{fig:alg1_multiple}
\end{figure}

As the number of shared terms increases, so does the improvement to the running time. While \eq{eq:shared} would take $\Theta(N^4)$ to solve using Algorithm \ref{alg:brute}, it takes only $O(N^3)$ to solve using Algorithm \ref{alg:3clique} (more precisely, if Algorithm \ref{alg1} takes $O(f(N))$, then \eq{eq:shared} takes $O(N^2f(N^2))$, which we have mentioned is $O(N^2\sqrt{N^2}) = O(N^3)$). In general, if we have $S$ shared terms, then the running time is $O(N^2\sqrt{N^S})$, yielding a speed-up of $\Omega(\sqrt{N^S})$ over the na\"ive solution of Algorithm \ref{alg:brute}.

\subsection{An Extension to Higher-Order Cliques with Decompositions Into Three Groups}
\label{sec:morethanthree}

By similar reasoning, we can apply our algorithm to cases where there are more than three factors, in which the factors can be separated into three \emph{groups}. For example, consider the clique in Figure \ref{fig:split}(a), which we shall call $G$ (the entire graph is a clique, but for clarity we only  draw an edge when the corresponding nodes belong to a common factor). Each of the factors in this graph have been labeled using either differently colored edges (for factors of size larger than two) or dotted edges (for factors of size two), and the max-marginal we wish to compute has been labeled using colored nodes. We assume that it is possible to split this graph into three groups such that every factor is contained within a single group, along with the max-marginal we wish to compute (Figure \ref{fig:split}, (b)). If such a decomposition is not possible, we will have to resort to further extensions to be described in Section \ref{sec:ext2}.

\begin{figure}
\footnotesize
\begin{center}
\begin{tabular}{cc}
\parbox[c]{0.39\columnwidth}{\centering\includegraphics[scale=0.46,angle=-90]{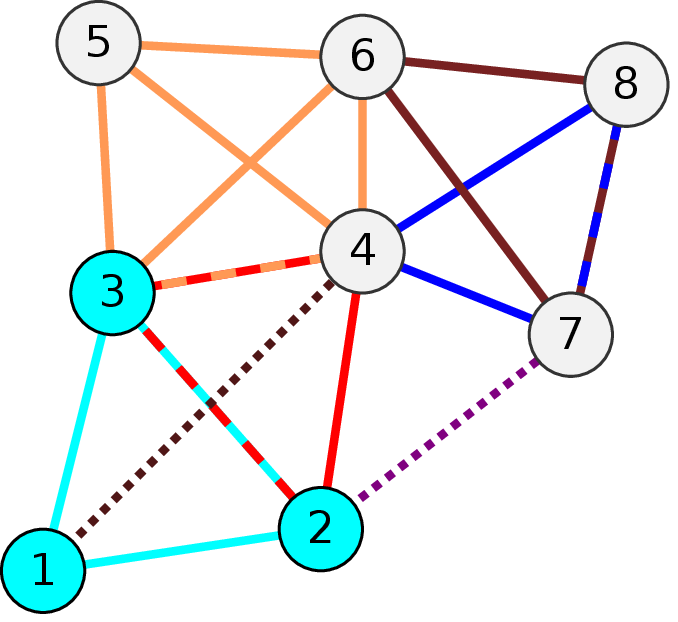}} & \parbox[c]{0.5\columnwidth}{\centering\includegraphics[scale=0.46,angle=-90]{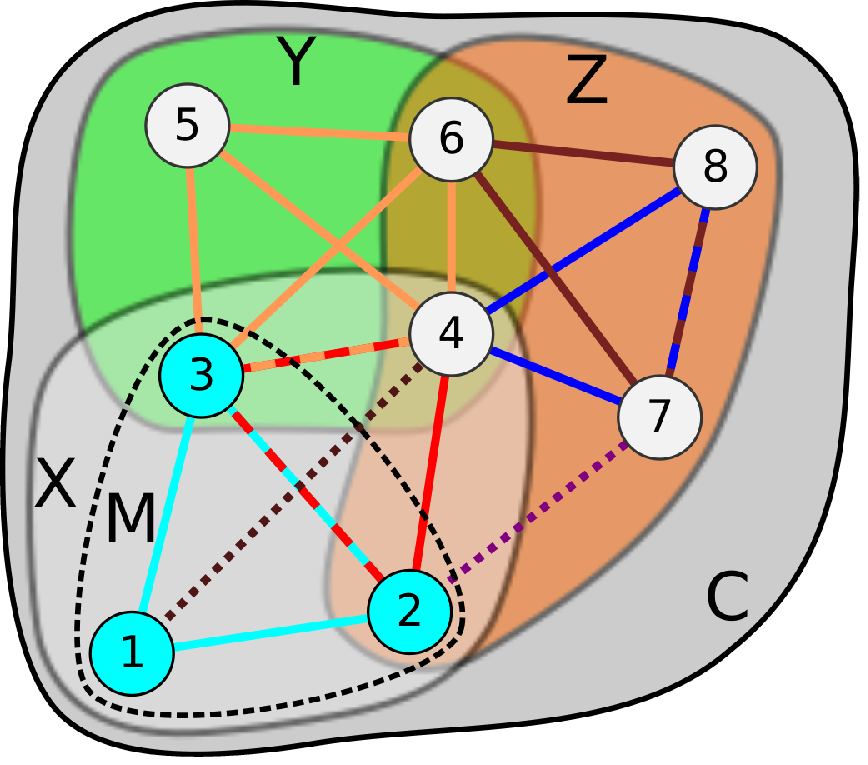}}\\
(a) & (b)\\
\end{tabular}
\end{center}

(a) We begin with a set of factors (indicated using colored lines), which are assumed to belong to some clique in our model; we wish to compute the max-marginal with respect to one of these factors (indicated using colored nodes); (b) The factors are split into three groups, such that every factor is entirely contained within one of them (Algorithm \ref{alg2}, line \ref{line:defs}).

\begin{center}
\begin{tabular}{ccc}
\parbox[c]{0.35\columnwidth}{\centering\includegraphics[scale=0.46]{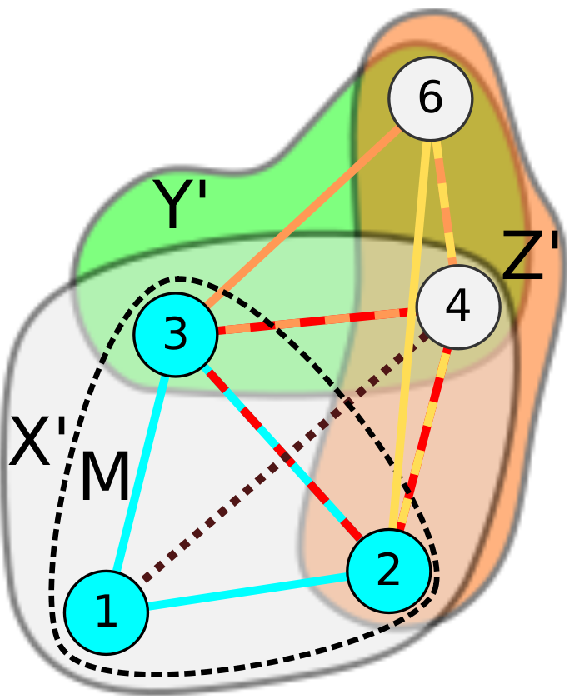}} & \parbox[c]{0.25\columnwidth}{\centering\includegraphics[scale=0.46]{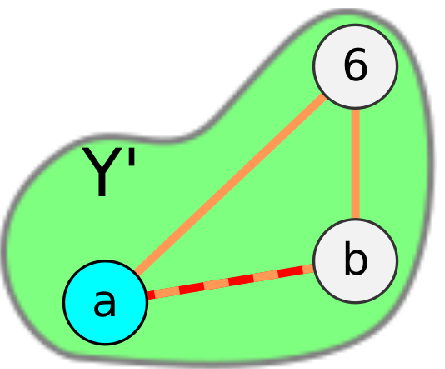}} & \parbox[c]{0.2\columnwidth}{\centering\includegraphics[scale=0.46]{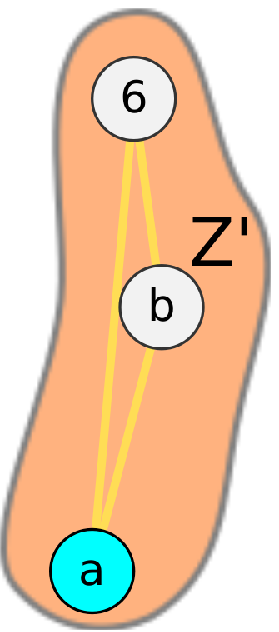}}\\
(c) & (d) & (e)\\
\end{tabular}
\end{center}

(c) Any nodes contained in only one of the groups are marginalized (Algorithm \ref{alg2}, lines \ref{line:marg1}, \ref{line:marg2}, and \ref{line:marg3}); the problem is now very similar to that described in Algorithm \ref{alg:3clique}, except that \emph{nodes} have been replaced by \emph{groups}; note that this essentially introduces maximal factors in $Y'$ and $Z'$; (d) For every value $(a,b) \in \text{dom}(x_3, x_4)$, $\Psi^Y(a, b, x_6)$ is sorted (Algorithm \ref{alg2}, lines \ref{line:for1start}--\ref{line:for1end}); (e) For every value $(a, b) \in \text{dom}(x_2, x_4)$, $\Psi^Z(a, b, x_6)$ is sorted (Algorithm \ref{alg2}, lines \ref{line:for2start}--\ref{line:for2end}).

\begin{center}
\begin{tabular}{cc}
\parbox[c]{0.4\columnwidth}{\centering\includegraphics[scale=0.46]{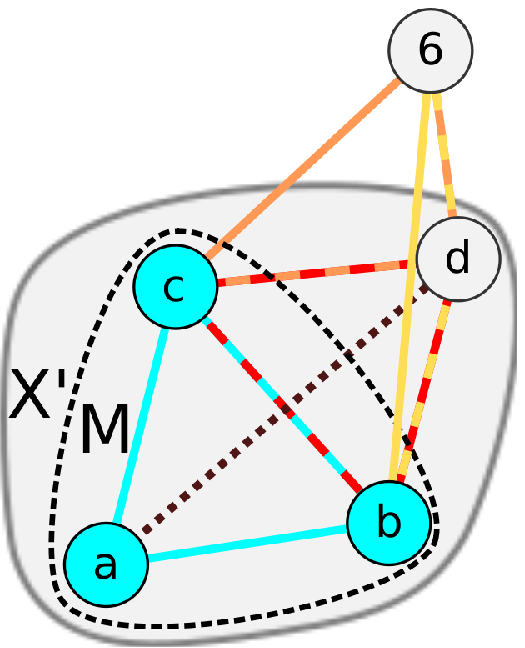}} & \parbox[c]{0.4\columnwidth}{\centering\includegraphics[scale=0.46]{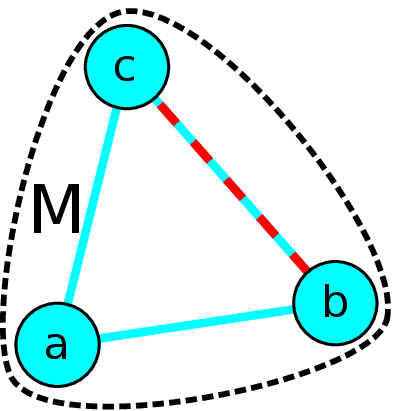}}\\
(f) & (g)\\
\end{tabular}
\end{center}

(f) For every $\mathbf{n} \in \text{dom}(X')$, we choose the best value of $x_6$ by Algorithm \ref{alg1} (Algorithm \ref{alg2}, lines \ref{line:for3start}--\ref{line:for3end}); (g) The result is marginalized with respect to $M$ (Algorithm \ref{alg2}, line \ref{line:marg}).

\caption{Algorithm \ref{alg2}, explained pictorially. In this case, the most computationally intensive step is the marginalization of $Z$ (in step (c)), which takes $\Theta(N^5)$. However, the algorithm can actually be applied \emph{recursively} to the group $Z$, resulting in an overall running time of $O(N^4\sqrt{N})$, for a max-marginal that would have taken $\Theta(N^8)$ to compute using the na\"ive solution of Algorithm \ref{alg:brute}.}
\label{fig:split}
\end{figure}

\begin{algorithm}
  \caption{Compute the max-marginal of $G$ with respect to $M$, where $G$ is split into three groups}
  \label{alg2}
\begin{algorithmic}[1]
\REQUIRE potentials $\Phi_G(\mathbf{x}) = \Phi_X(\mathbf{x}_X)\times\Phi_Y(\mathbf{x}_Y)\times\Phi_Z(\mathbf{x}_Z)$; each of the factors should be contained in exactly one of these terms, and we assume that $M\subseteq X$ (see Figure \ref{fig:split})

\STATE \textbf{Define:} $X' := ((Y \cup Z) \cap X) \cup M$; $Y' := (X \cup Z) \cap Y$; $Z' := (X \cup Y) \cap Z$ \COMMENT{$X'$ contains the variables in $X$ that are shared by at least one other group; alternately, the variables in $X \setminus X'$ appear only in $X$ (sim.~for $Y'$ and $Z'$)} \label{line:defs}

\STATE compute $\Psi^X(\mathbf{x}_{X'}) := \max_{X \setminus X'} \Phi_X(\mathbf{x}_X)$\ \COMMENT{we are marginalizing over those variables in $X$ that do not appear in any of the other groups (or in $M$); this takes $\Theta(N^S)$ if done by brute force (Algorithm \ref{alg:brute}), but may also be done by a recursive call to Algorithm \ref{alg2}}\label{line:marg1}

\STATE compute $\Psi^Y(\mathbf{x}_{Y'}) := \max_{Y \setminus Y'} \Phi_Y(\mathbf{x}_Y)$\label{line:marg2}

\STATE compute $\Psi^Z(\mathbf{x}_{Z'}) := \max_{Z \setminus Z'} \Phi_Z(\mathbf{x}_Z)$\label{line:marg3}

\FOR{$\mathbf{n} \in \text{dom}(X \cap Y)$}\label{line:for1start}

  \STATE compute $\mathbf{P}^Y[\mathbf{n}]$ by sorting $\Psi^Y(\mathbf{n}; \mathbf{x}_{Y' \setminus X})$\ \COMMENT{takes $\Theta(S_\setminus N^{S_\setminus}\log N)$; $\Psi^Y(\mathbf{n}; \mathbf{x}_{Y' \setminus X})$ is free over $\mathbf{x}_{Y' \setminus X}$, and is treated as an array by `flattening' it; $\mathbf{P}^Y[\mathbf{n}]$ contains the $|Y' \setminus X| = |(Y \cap Z) \setminus X|$-dimensional indices that sort it}

\ENDFOR\ \COMMENT{this loop takes $\Theta(S_\setminus N^S\log N)$}\label{line:for1end}

\FOR{$\mathbf{n} \in \text{dom}(X \cap Z)$}\label{line:for2start}

  \STATE compute $\mathbf{P}^Z[\mathbf{n}]$ by sorting $\Psi^Z(\mathbf{n}; \mathbf{x}_{Z' \setminus X})$

\ENDFOR\ \COMMENT{this loop takes $\Theta(S_\setminus N^S\log N)$}\label{line:for2end}

\FOR{$\mathbf n \in \text{dom}(X')$}\label{line:for3start}

  \STATE $\left(\mathbf{v}_a, \mathbf{v}_b\right) := \left(\Psi^Y(\mathbf{n}|_{Y'}; \mathbf{x}_{Y'\setminus X'}), \Psi^Z(\mathbf{n}|_{Z'}; \mathbf{x}_{Z'\setminus X'})\right)$\ \COMMENT{$\mathbf{n}|_{Y'}$ is the `restriction' of the vector $\mathbf{n}$ to those indices in $Y'$ (meaning that $\mathbf{n}|_{Y'} \in \text{dom}(X'\cap Y')$); hence $\Psi^Y(\mathbf{n}|_{Y'}; \mathbf{x}_{Y'\setminus X'})$ is free in $\mathbf{x}_{Y'\setminus X'}$, while $\mathbf{n}|_{Y'}$ is fixed}

  \STATE $\left(p_a, p_b\right) := \left(\mathbf{P}^Y[{\mathbf{n}|_{Y'}}], \mathbf{P}^Z[{\mathbf{n}|_{Z'}}]\right)$

  \STATE $\mathit{best} := \mathit{Algorithm\ref{alg1}}\left(\mathbf{v}_a, \mathbf{v}_b, p_a, p_b\right)$ \COMMENT{takes $O(\sqrt{S_\setminus})$}

  \STATE $m_X(\mathbf{n}) := \Psi^X(\mathbf{n}) \times \Psi^Y(best; \mathbf{n}|_{Y'}) \times \Psi^Z(\mathit{best}; \mathbf{n}|_{Z'})$

\ENDFOR\label{line:for3end}

\STATE $m_M(\mathbf{x}_M) := \mathit{Algorithm\ref{alg:brute}}(m_X, M)$ \COMMENT{i.e., we are using Algorithm \ref{alg:brute} to marginalize $m_X(\mathbf{x}_X)$ with respect to $M$; this takes $\Theta(N^S)$}\label{line:marg}

\end{algorithmic}
\end{algorithm}

Ideally, we would like these groups to have size $\simeq |G|/3$, though in the worst case they will have size no larger than $|G|-1$. We call these groups $X$, $Y$, $Z$, where $X$ is the group containing the max-marginal $M$ that we wish to compute. In order to simplify the analysis of this algorithm, we shall express the running time in terms of the size of the largest group, $S = \max(|X|, |Y|, |Z|)$, and the largest difference, $S_{\setminus} = \max(|Y \setminus X|, |Z \setminus X|)$. The max-marginal can be computed using Algorithm \ref{alg2}.

The running times shown in Algorithm \ref{alg2} are loose upper-bounds, given for the sake of expressing the running time in simple terms. More precise running times are given in Table \ref{tab:runtime}; any of the terms shown in Table \ref{tab:runtime} may be dominant. Some example graphs, and their resulting running times are shown in Figure \ref{fig:examples}.

\begin{table}[t]
  \caption{Detailed running time analysis of Algorithm \ref{alg2}; any of these terms may be asymptotically dominant}
  \label{tab:runtime}
\begin{center}
\begin{tabular}{l|l|l}
 \hline
 \textbf{Description} & \textbf{lines} & \textbf{time}\\
 \hline
 \hline
 Marginalization of $\Phi_X$, without recursion & \ref{line:marg1} & $\Theta(N^{|X|})$\\
 Marginalization of $\Phi_Y$ & \ref{line:marg2} & $\Theta(N^{|Y|})$\\
 Marginalization of $\Phi_Z$ & \ref{line:marg3} & $\Theta(N^{|Z|})$\\
 Sorting $\Phi_Y$ & \ref{line:for1start}--\ref{line:for1end} & $\Theta(|Y'\!\setminus\! X|N^{|Y'|}\log N)$\\
 Sorting $\Phi_Z$ & \ref{line:for2start}--\ref{line:for2end} & $\Theta(|Z'\!\setminus\! X|N^{|Z'|}\log N)$\\
 Running Algorithm \ref{alg1} on the sorted values & \ref{line:for3start}--\ref{line:for3end} & $O(N^{|X'|}\sqrt{N^{|(Y' \cap Z') \setminus X'|}})$\\
 \hline
\end{tabular}
\end{center}
\end{table}

\begin{figure*}
\begin{tabular}{p{0.14\textwidth}ccccc}
Graph: & \parbox[c]{0.115\textwidth}{\centering\includegraphics[scale=0.28]{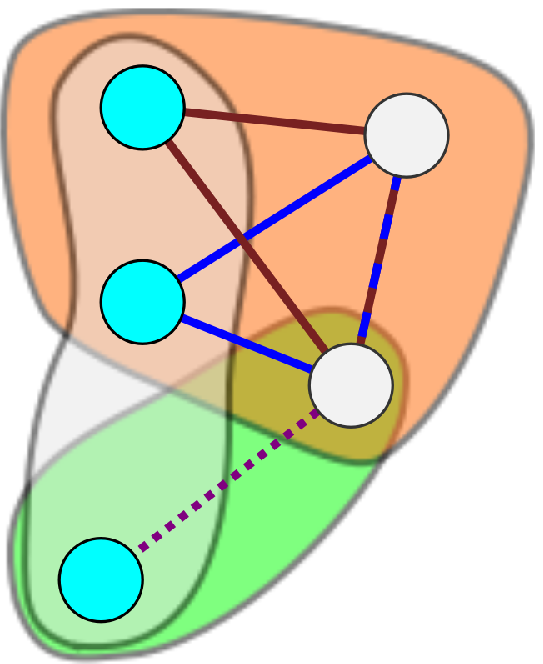}} &
         \parbox[c]{0.115\textwidth}{\centering\includegraphics[scale=0.28]{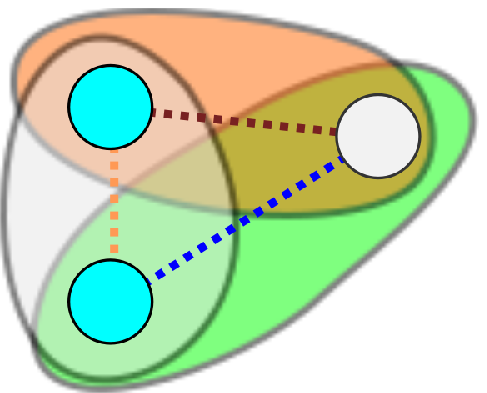}} &
         \parbox[c]{0.19\textwidth}{\hspace*{-0.025\textwidth}\centering\includegraphics[scale=0.28]{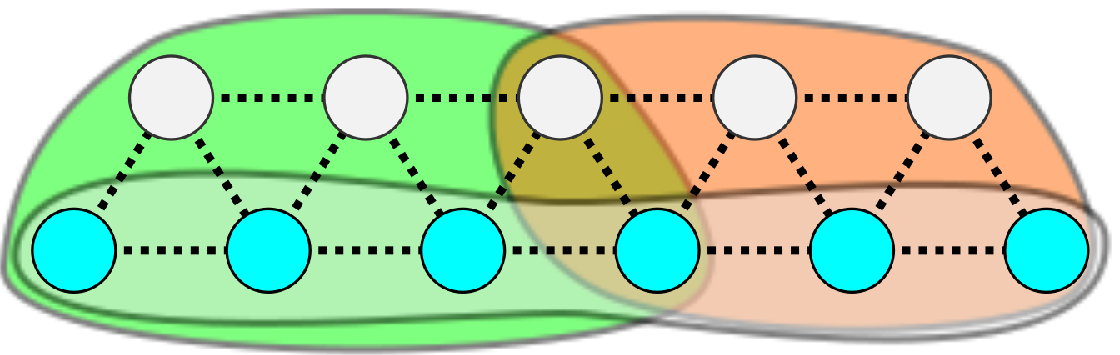}} &
         \parbox[c]{0.115\textwidth}{\centering\includegraphics[scale=0.28]{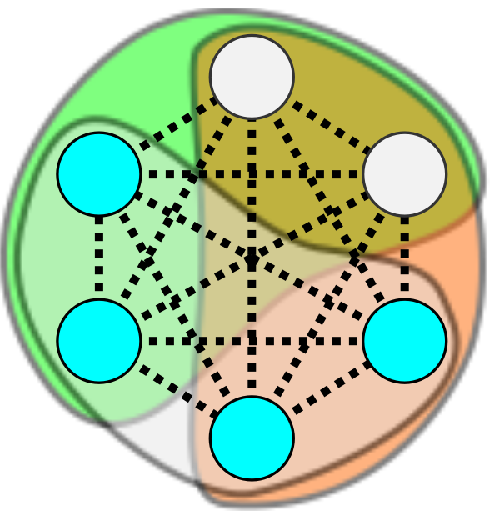}} &
         \parbox[c]{0.15\textwidth}{\centering$\lbrace$A complete graph $K_M$, with pairwise terms$\rbrace$}\\
       & (a) & (b) & (c) & (d) & (e)\\
Algorithm \ref{alg:brute}: & $\Theta(N^5)$              & $\Theta(N^3)$         & $\Theta(N^{11})$      & $\Theta(N^6)$ & $\Theta(N^M)$\\
Algorithm \ref{alg2}:      & $O(N^3\sqrt{N})$        & $O(N^2\sqrt{N})$ & $O(N^6\sqrt{N})$ & $O(N^5)$ & $O(N^{5M/6})$\\
Speed-up:                   & $\Omega(N\sqrt{N})$ & $\Omega(\sqrt{N})$    & $\Omega(N^4\sqrt{N})$ & $\Omega(N)$   & $\Omega(N^{M/6})$
\end{tabular}
 \caption{Some example graphs whose max-marginals are to be computed with respect to the colored nodes, using the three regions shown. Factors are indicated using differently colored edges, while dotted edges always indicate pairwise factors.
(a) is the region $Z$ from Figure \ref{fig:split} (recursion is applied \emph{again} to achieve this result); (b) is the graph used to motivate Algorithm \ref{alg:3clique}; (c) shows a query in a graph with regular structure; (d) shows a complete graph with six nodes; (e) generalizes this to a clique with $M$ nodes.}

\label{fig:examples}
\end{figure*}

\subsubsection{Applying Algorithm \ref{alg2} Recursively}

The marginalization steps of Algorithm \ref{alg2} (Lines \ref{line:marg1}, \ref{line:marg2}, and \ref{line:marg3}) may further decompose into smaller groups, in which case Algorithm \ref{alg2} can be applied recursively. For instance, the graph in Figure \ref{fig:examples}(a) represents the marginalization step that is to be performed in Figure \ref{fig:split}(c) (Algorithm \ref{alg2}, Line \ref{line:marg3}). Since this marginalization step is the asymptotically dominant step in the algorithm, applying Algorithm \ref{alg2} recursively lowers the asymptotic complexity.

Another straightforward example of applying recursion in Algorithm \ref{alg2} is shown in Figure \ref{fig:recursion}, in which a ring-structured model is marginalized with respect to two of its nodes. Doing so takes $O(MN^2\sqrt{N})$; in contrast, solving the same problem using the junction-tree algorithm (by triangulating the graph) would take $\Theta(MN^3)$. Loopy belief-propagation takes $\Theta(MN^2)$ per iteration, meaning that our algorithm will be faster if the number of iterations is $\Omega(\sqrt{N})$. Naturally, Algorithm \ref{alg:3clique} could be applied directly to the triangulated graph, which would again take $O(MN^2\sqrt{N})$.

\begin{figure*}
\begin{center}
\begin{tabular}{cc}
\parbox[c]{0.5\textwidth}{\centering\includegraphics[scale=0.5]{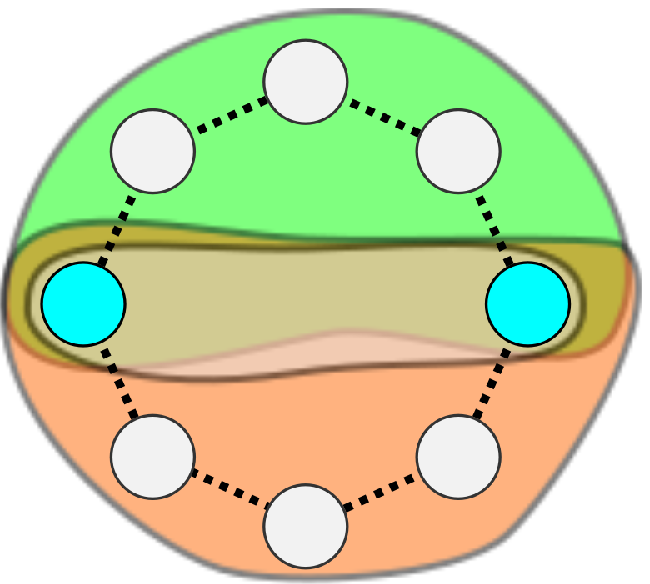}} & $O(N^2)$ \\
 & $+$\\
\parbox[c]{0.5\textwidth}{\centering\parbox[c]{0.225\columnwidth}{\centering\includegraphics[scale=0.5]{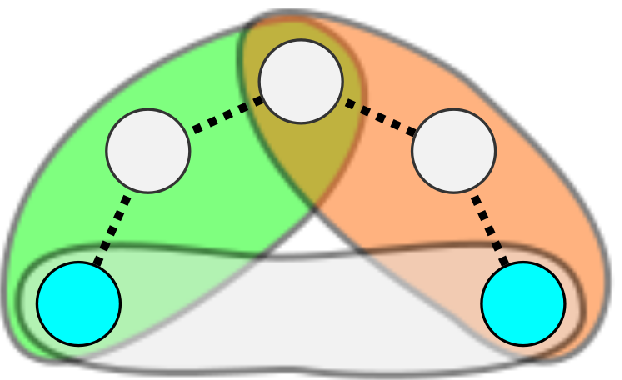}}\parbox[c]{0.225\columnwidth}{\centering\includegraphics[scale=0.5,angle=180]{figures_fast/recurs/recurs2}}} & $O(2N^2\sqrt{N})$\\
 & $+$\\
\parbox[c]{0.5\textwidth}{\centering\parbox[c]{0.125\columnwidth}{\centering\includegraphics[scale=0.5]{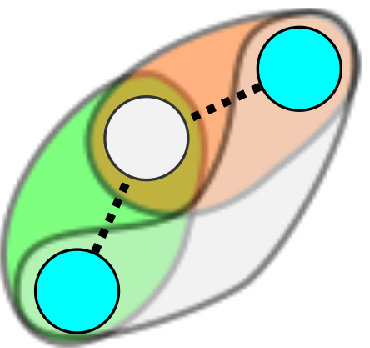}}\parbox[c]{0.125\columnwidth}{\centering\includegraphics[scale=0.5,angle=90]{figures_fast/recurs/recurs3}}\parbox[c]{0.125\columnwidth}{\centering\includegraphics[scale=0.5,angle=180]{figures_fast/recurs/recurs3}}\parbox[c]{0.125\columnwidth}{\centering\includegraphics[scale=0.5,angle=270]{figures_fast/recurs/recurs3}}} & $O(4N^2\sqrt{N})$ (by Algorithm \ref{alg:3clique})\\
\end{tabular}
 \end{center}
\caption{In the above example, lines \ref{line:marg1}--\ref{line:marg3} of Algorithm \ref{alg2} are applied recursively, achieving a total running time of $O(MN^2\sqrt{N})$ for a loop with $M$ nodes (our algorithm achieves the same running time in the triangulated graph).}
\label{fig:recursion}
\end{figure*}

\subsection{A General Extension to Higher-Order Cliques}
\label{sec:ext2}

Naturally, there are cases for which a decomposition into three terms is not possible, such as
\begin{multline}
m_{i,j,k}(x_i,x_j,x_k) =\max_{x_m} \Phi_{i,j,k}(x_i, x_j, x_k)\times
\Phi_{i,j,m}(x_i, x_j, x_m)\times\\
\Phi_{i,k,m}(x_i, x_k, x_m)\times\Phi_{j,k,m}(x_j, x_k, x_m)
\label{eq:decomp4}
\end{multline}
(i.e., a clique of size four with third-order factors). However, if the model contains factors of size $K$, it must always be possible to split it into $K+1$ groups (e.g.~four in the case of \eq{eq:decomp4}).

Our optimizations can easily be applied in these cases simply by adapting Algorithm \ref{alg1} to solve problems of the form
\begin{equation}
 \hat{i} = \argmax_{i\in\lbrace 1 \ldots N \rbrace} \left\lbrace \mathbf{v}_1[i] \times \mathbf{v}_2[i] \times \cdots \times \mathbf{v}_K[i] \right\rbrace.
\label{eq:hatK}
\end{equation}
Pseudocode for this extension is presented in Algorithm \ref{alg:ext}. Note carefully the use of the variable $\mathit{read}$: we are storing which indices have been read to avoid re-reading them; this guarantees that our Algorithm is never asymptotically worse than the na\"ive solution. Figure \ref{fig:alg3} demonstrates how such an algorithm behaves in practice. Again, we shall discuss the running time of this extension in Appendix \ref{sec:analysis}. For the moment, we state the following theorem:

\begin{theorem}
 Algorithm \ref{alg:ext} generalizes Algorithm \ref{alg1} to $K$ lists with an expected running time of $O(KN^\frac{K-1}{K})$, yielding a speed-up of at least $\Omega(\frac{1}{K}N^\frac{1}{K})$ in cliques containing $K$-ary factors. It is never worse than the na\"ive solution, meaning that it takes $O(\min(N, KN^\frac{K-1}{K}))$.
\label{the:algext}
\end{theorem}


\begin{algorithm}
  \caption{Find $i$ such that $\prod_{k=1}^K \mathbf{v}_k[i]$ is maximized}
  \label{alg:ext}
\begin{algorithmic}[1]
\REQUIRE $K$ vectors $\mathbf{v}_1 \ldots \mathbf{v}_K$; permutation functions $p_1\ldots p_K$ that sort them in decreasing order; a vector $\mathit{read}$ indicating which indices have been read, and a unique value $T \notin \mathit{read}$ \COMMENT{$\mathit{read}$ is essentially a boolean array indicating which indices have been read; since \emph{creating} this array is an $O(N)$ operation, we create it externally, and reuse it $O(N)$ times; setting $\mathit{read}[i] = T$ indicates that a particular index has been read; we use a different value of $T$ for each call to this function so that $\mathit{read}$ can be reused without having to be reinitialized}
\STATE \textbf{Initialize:} $\mathit{start} := 1$,\\
$\mathit{max} := \max_{p \in \lbrace p_1 \ldots p_K \rbrace} \prod_{k=1}^K \mathbf{v}_k[p[1]]$,\\
$\mathit{best} := \argmax_{p \in \lbrace p_1 \ldots p_K \rbrace} \prod_{k=1}^K \mathbf{v}_k[p[1]]$
\FOR{$k \in 1 \ldots K$}
\STATE $\mathit{end}_k := \max_{q \in \lbrace p_1 \ldots p_K \rbrace} p_k^{-1}[q[1]]$
\ENDFOR
\STATE $\mathit{read}[p[1]] = T$
\WHILE{$\mathit{start} < \max \lbrace \mathit{end}_1 \ldots \mathit{end}_K \rbrace$}
\STATE $\mathit{start} := \mathit{start} + 1$
\IF{$\mathit{read}[p[start]] := T$}
\STATE \textbf{continue}
\ENDIF
\STATE $\mathit{read}[p[start]] := T$
\STATE $\mathit{m} := \max_{p \in \lbrace p_1 \ldots p_K \rbrace} \prod_{k=1}^K \mathbf{v}_k[p[\mathit{start}]]$
\STATE $\mathit{b} := \argmax_{p \in \lbrace p_1 \ldots p_K \rbrace} \prod_{k=1}^K \mathbf{v}_k[p[\mathit{start}]]$

\IF {$\mathit{m} > \mathit{max}$}
\STATE $\mathit{best} := \mathit{b}$
\STATE $\mathit{max} := \mathit{m}$
\ENDIF

\FOR{$k \in \lbrace 1 \ldots K \rbrace$}
\STATE $\mathit{e}_k := \max_{q \in \lbrace p_1 \ldots p_K \rbrace} p_k^{-1}[q[\mathit{start}]]$
\ENDFOR
\FOR{$k \in \lbrace 1 \ldots K \rbrace$}
\STATE $\mathit{end}_k := \min(\mathit{e}_k, \mathit{end}_k)$
\ENDFOR
\ENDWHILE\ \COMMENT{see Appendix \ref{sec:analysis} for running times}
\RETURN $\mathit{best}$
\end{algorithmic}
\end{algorithm}

\begin{figure}
 \footnotesize
\begin{flushright}
$\text{Step 1:} \left\lbrace\text{\parbox{0.77\columnwidth}{\includegraphics[angle=-90,width=0.6\columnwidth]{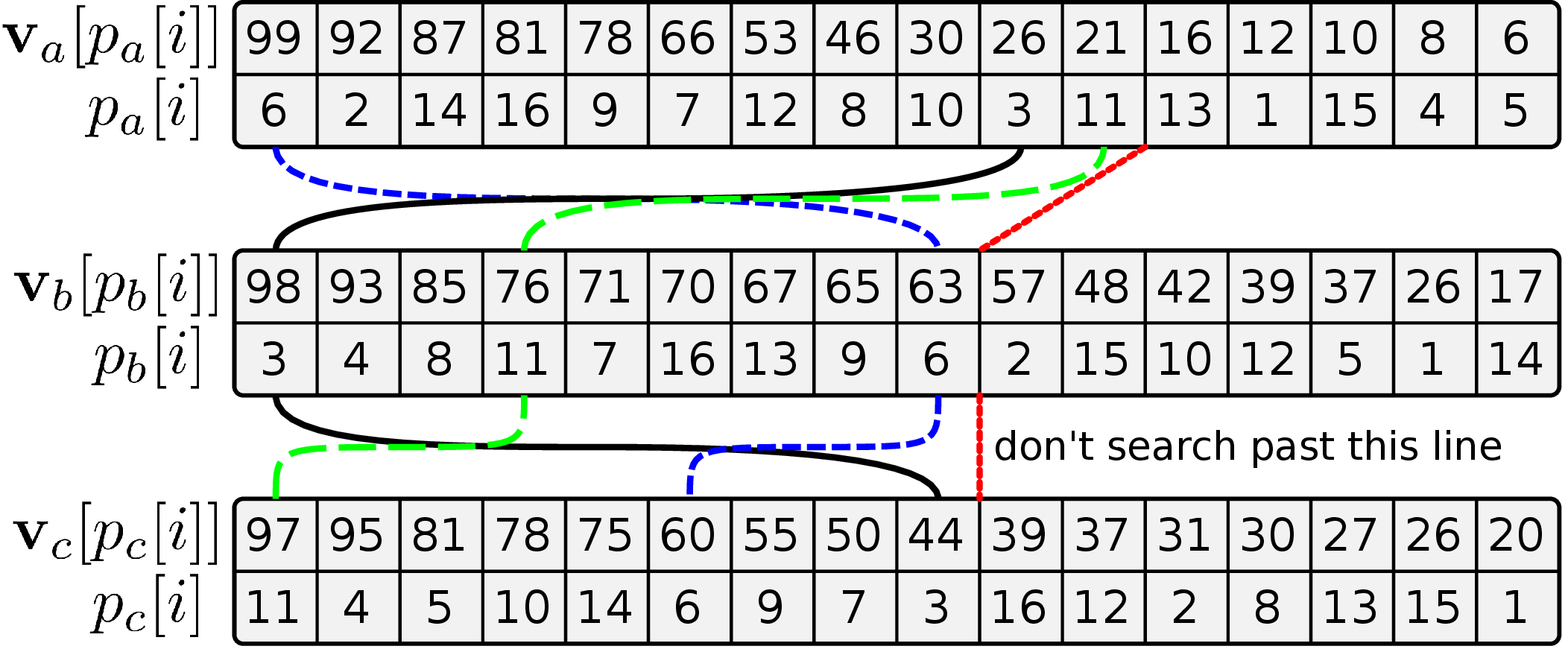}}}\right.$\parbox{0.01\columnwidth}{\ }\\
\end{flushright}
\caption{Algorithm \ref{alg1} can easily be extended to cases including more than two sequences.}
\label{fig:alg3}
\end{figure}

Using Algorithm \ref{alg:ext}, we can similarly extend Algorithm \ref{alg2} to allow for any number of \emph{groups} (pseudocode is not shown; all statements about the groups $Y$ and $Z$ simply become statements about $K$ groups $\left\lbrace G_1 \ldots G_K \right\rbrace$, and calls to Algorithm \ref{alg1} become calls to Algorithm \ref{alg:ext}). The one remaining case that has not been considered is when the sequences $\mathbf{v}_1 \cdots \mathbf{v}_K$ are functions of different (but overlapping) variables; na\"ively, we can create a new variable whose domain is the product space of all of the overlapping terms, and still achieve the performance improvement guaranteed by Theorem \ref{the:algext}; in some cases, better results can again be obtained by applying recursion, as in Figure \ref{fig:examples}.

As a final comment we note that we have not provided an algorithm for choosing \emph{how} to split the variables of a model into $(K+1)$-groups. We note even if we split the groups in a na\"ive way, we are guaranteed to get \emph{at least} the performance improvement guaranteed by Theorem \ref{the:algext}, though more `intelligent' splits may further improve the performance.
Furthermore, in all of the applications we have studied, $K$ is sufficiently small that it is inexpensive to consider all possible splits by brute-force.

\section{Exploiting `Data Independence' in Latent Factors}
\label{sec:latent}

While \eq{eq:map} gave the general form of MAP-inference in a graphical model, it will often be more convenient to express our objective function as being conditioned upon some \emph{observation}, $\mathbf{y}$. Thus inference consists of solving an optimization problem of the form
\begin{equation}
 \mathbf{\hat{x}}(\mathbf{y}) = \argmax_{\mathbf{x}} \prod_{C \in \mathcal C} \Phi_C(\mathbf{x}_C | \mathbf{y}_C).
\label{eq:map_obs}
\end{equation}
When our objective function is written in this way, further factorization is often possible, yielding an expression of the form
\begin{equation}
 \mathbf{\hat{x}}(\mathbf{y}) = \argmax_{\mathbf{x}} \underbrace{\prod_{F \in \mathcal F} \Phi_F(\mathbf{x}_F | \mathbf{y}_F)}_{\text{data dependent}} \times \!\!\underbrace{\prod_{C \in \mathcal C} \Phi_C(\mathbf{x}_C)}_{\text{data independent}}\!\!,
\label{eq:alternate}
\end{equation}
where each $F \in \mathcal F$ is a subset of some $C \in \mathcal C$. We shall say that those factors that do not depend on the observation are `data independent'. 

By far the most common instance of this type of model has `data dependent' factors consisting of a single latent variable, and conditioned upon a single observation, and `data independent' factors consisting of a pair of latent variables. This was precisely the class of models depicted at the beginning of our paper in Figure \ref{fig:examps}, whose objective function takes the form
\begin{equation}
 \mathbf{\hat{x}}(\mathbf{y}) = \argmax_{\mathbf{x}} \underbrace{\prod_{i \in \mathcal N} \! \Phi_i(x_i | y_i)}_{\text{node potential}} \times \underbrace{\prod_{(i,j) \in \mathcal E} \Phi_{i,j}(x_i, x_j)}_{\text{edge potential}}
\label{eq:pairwise}
\end{equation}
(where $\mathcal N$ and $\mathcal E$ are the set of nodes and edges in our graphical model). As in the Section \ref{sec:optimizing}, we shall concern ourselves with this version of the model, and explain only briefly how it can be applied with larger factors, as in Section \ref{sec:morethanthree}.

Note that in \eq{eq:pairwise} we are no longer concerned solely with exact inference via the junction-tree algorithm. In many models, such as grids and rings, \eq{eq:pairwise} shall be solved approximately by means of either loopy belief-propagation, or inference in a factor graph.

Given the decomposition of \eq{eq:pairwise}, message-passing now takes the form
\begin{equation}
 m_{A\rightarrow B}(q) = \Phi_i(q) \times \max_{y_j} \underbrace{\Phi_j(y_j)}_{\mathbf{v}_a} \times \underbrace{\Phi_{i,j}(q,y_j)}_{\mathbf{v}_b}
\label{eq:pairwisemessage}
\end{equation}
(where $A = (i,j)$ and $B = (i,k)$). Just as we made the comparison between \eq{eq:max3} and matrix multiplication, we can see \eq{eq:pairwisemessage} as being related to the multiplication of a matrix ($\Phi_{i,j}$) with a vector ($\Phi_j$), again with summation replaced by maximization. Given the results we have already shown, it is trivial to solve \eq{eq:pairwisemessage} in $O(N\sqrt{N})$ \emph{if we know the permutations that sort $\Phi_j$, and the rows of $\Phi_{i,j}$}. The algorithm for doing so is shown in Algorithm \ref{alg:message}. The difficultly we face in this instance is that sorting the rows of $\Phi_{i,j}$ takes $\Theta(N^2\log N)$, i.e., longer than Algorithm \ref{alg:message} itself.

This problem is circumvented due to the following simple observation: since $\Phi_{i,j}(x_i, x_j)$ consists only of latent variables (and not upon the observation), this `sorting' step can take place \emph{offline}, i.e., before the `data' has been observed.

Two further observations mean that even this offline cost can often be avoided. Firstly, many models have a `homogeneous' prior, i.e., the same prior is shared amongst every edge (or clique) of the model. In such cases, only a single `copy' of the prior needs to be sorted, meaning that in any model containing $\Omega(\log N)$ edges, speed improvements can be gained over the naive implementation. Secondly, where an iterative algorithm (such as loopy belief-propagation) is to be used, the sorting step need only take place prior to the \emph{first} iteration; if $\Omega(\log N)$ iterations of belief propagation are to be performed (or indeed, if the number of edges multiplied by the number of iterations is $\Omega(\log N)$), we shall again gain speed improvements even when the sorting step is done online.

In fact, the second of these conditions obviates the need for data independence altogether. In other words, in \emph{any} pairwise model in which $\Omega(\log N)$ iterations of belief propagation are to be performed, \emph{the pairwise terms need to be sorted only during the first iteration}. Thus these improvements apply to those models in Figure \ref{fig:examples_intro}, so long as the number of iterations is $\Omega(\log N)$.

\begin{algorithm}[htb]
 \caption{Solve \eq{eq:pairwisemessage} using Algorithm \ref{alg1}}
 \label{alg:message}
\begin{algorithmic}[1]
 \REQUIRE a potential $\Phi_{i,j}(a,b) \times \Phi_{i}(a | y_i) \times \Phi_{j}(b | y_j)$ whose max-marginal $m_{i}(x_i)$ we wish to compute, and a set of permutation functions $\mathbf{P}$ such that $\mathbf{P}[i]$ sorts the $i^{\textit{th}}$ row of $\Phi_{i,j}$ (in decreasing order).
\STATE compute the permutation function $p_a$ by sorting $\Psi_j$ \COMMENT{takes $\Theta(N\log N)$}
\FOR{$q \in \lbrace 1\ldots N \rbrace$}
\STATE $\left(\mathbf{v}_a, \mathbf{v}_b\right) := \left(\Psi_j, \Phi_{i,j}(q,x_j|y_i,y_j) \right)$
\STATE $\mathit{best} := \mathit{Algorithm\ref{alg1}}\left(\mathbf{v}_a, \mathbf{v}_b, p_a, \mathbf{P}[q] \right)$ \COMMENT{$O(\sqrt{N})$}
\STATE $m_{A\rightarrow B}(q) := \Phi_i(q) \times \Phi_j(\mathit{best}) \times \Phi_{i,j}(q, \mathit{best} | y_i, y_j)$
\ENDFOR\ \COMMENT{\emph{expected-case} $O(N\sqrt{N})$}
\RETURN $m_{A\rightarrow B}$
\end{algorithmic}
\end{algorithm}

\subsection{Extension to Higher-Order Cliques}

Just as in Section \ref{sec:morethanthree}, we can extend Algorithm \ref{alg:message} to factors of any size, so long as the purely latent cliques contain more latent variables than those cliques that depend upon the observation. The analysis for this type of model is almost exactly the same as that presented in Section \ref{sec:morethanthree}, except that any terms consisting of purely latent variables are processed offline.

As we mentioned in \ref{sec:morethanthree}, if a model contains (non-maximal) factors of size $K$, we will gain a speed-up of $\Omega(\frac{1}{K}N^{\frac{1}{K}})$. If in addition there is a factor (either maximal or non-maximal) consisting of purely latent variables, we can still obtain a speed-up of $\Omega(\frac{1}{K+1}N^{\frac{1}{K+1}})$, since this factor merely contributes an additional term to \eq{eq:hatK}. Thus when our `data-dependent' terms contain only a single latent variable (i.e., $K=1$), we gain a speed-up of $\Omega(\sqrt{N})$, as in Algorithm \ref{alg:message}.

\section{Performance Improvements in Existing Applications}
\label{sec:improvements}

Our results are immediately compatible with several applications that rely on inference in graphical models. As we have mentioned, our results apply to \emph{any model whose cliques decompose into lower-order terms}.


Often, potentials are defined only on \emph{nodes} and \emph{edges} of a model. A $D^{\text{th}}$-order Markov model has a tree-width of $D$, despite often containing only pairwise relationships. Similarly `skip-chain CRFs' \citep{skipchain,galley06}, and junction-trees used in SLAM applications \citep{Paskin2003} often contain only pairwise terms, and may have low tree width under reasonable conditions. In each case, if the tree-width is $D$, Algorithm \ref{alg2} takes $O(MN^{D}\sqrt{N})$ (for a model with $M$ nodes and $N$ states per node), yielding a speed-up of $\Omega(\sqrt{N})$.



Models for shape matching and pose reconstruction often exhibit similar properties \citep{Tres09,Donner07sparsemrf,sigal06}. In each case, third-order cliques factorize into second order terms; hence we can apply Algorithm \ref{alg:3clique} to achieve a speed-up of $\Omega(\sqrt{N})$.

Another similar model for shape matching is that of \citet{pedro_shapes}; this model again contains third-order cliques, though it includes a `geometric' term constraining all three variables. Here, the third-order term is \emph{independent of the input data}, meaning that each of its rows can be sorted \emph{offline}, as described in Section \ref{sec:latent}. In this case, those factors that depend upon the observation are pairwise, meaning that we achieve a speed-up of $\Omega(N^{\frac{1}{3}})$. Further applications of this type shall be explored in Section \ref{sec:latent_exp}.

In \citet{lbpmatch}, deformable shape-matching is solved approximately using loopy belief-propagation. Their model has only second-order cliques, meaning that inference takes $\Theta(MN^2)$ \emph{per iteration}. Although we cannot improve upon this result, we note that we can typically do \emph{exact} inference in a single iteration in $O(MN^2\sqrt{N})$; thus our model has the same running time as $O(\sqrt{N})$ iterations of the original version. This result applies to all second-order models containing a single loop \citep{Weiss00}.

In \citet{McACaeBar08}, a model is presented for graph-matching using loopy belief-propagation; the maximal cliques for $D$-dimensional matching have size $(D+1)$, meaning that inference takes $\Theta(MN^{D+1})$ \emph{per iteration} (it is shown to converge to the correct solution); we improve this to $O(MN^D\sqrt{N})$.

\emph{Interval graphs} can be used to model resource allocation problems \citep{interval}; each node encodes a request, and overlapping requests form edges.  Maximal cliques grow with the number of overlapping requests, though the constraints are only pairwise, meaning that we again achieve an $\Omega(\sqrt{N})$ improvement.

Belief-propagation can be used to solve \emph{LP-relaxations} in pairwise graphical models. In \citet{sontag_lp}, LP-relaxations are computed for pairwise models by constructing several third-order `clusters', which compute pairwise messages for each of their edges. Again, an $\Omega(\sqrt{N})$ improvement is achieved.

Finally, in Section \ref{sec:latent_exp} we shall explore a variety of applications in which we have pairwise models of the form shown in \eq{eq:pairwise}. In all of these cases, we see an (expected) reduction of a $\Theta(MN^2)$ message-passing algorithm to $O(MN\sqrt{N})$.



Table \ref{tab:improvements} summarizes these results. Reported running times reflect the \emph{expected case}. Note that we are assuming that \emph{max-product belief-propagation is being used in a discrete model}; some of the referenced articles may use different variants of the algorithm (e.g.~Gaussian models, or approximate inference schemes).
We believe that our improvements may revive the exact, discrete version as a tractable option in these cases.

\begin{table*}
 \caption{Some existing work to which our results can be immediately applied ($M$ nodes, $N$ states per node, cliques of size $|C|$. `iter.' denotes that the algorithm is iterative).}
 \begin{center}
\footnotesize
  \begin{tabular}{|l|l|l|l|}
 \hline
   \normalsize{\textbf{Reference}} & \normalsize{\textbf{description}} & \normalsize{\textbf{running time}} & \normalsize{\textbf{our method}}\\
 \hline
 \hline
  \hspace{-4pt} \citet{McACaeBar08} & $D$-d graph-matching & $\Theta(MN^{D+1})$ (iter.) & $O(MN^D\sqrt{N})$ (iter.)\\
  \hspace{-4pt} \citet{skipchain} & Width-$D$ skip-chain & $O(MN^{D})$ & $O(MN^{D-1}\sqrt{N})$\\
  \hspace{-4pt} \citet{galley06} & Width-3 skip-chain & $\Theta(MN^3)$ & $O(MN^2\sqrt{N})$\\
  \hspace{-4pt} \citet{Paskin2003} (discrete case) & SLAM, width $D$ & $O(MN^D)$ & $O(MN^{D-1}\sqrt{N})$\\
  \hspace{-4pt} \citet{Tres09} & Deformable matching & $\Theta(MN^3)$ & $O(MN^2\sqrt{N})$\\
  \hspace{-4pt} \citet{lbpmatch} & Deformable matching & $\Theta(MN^2)$ (iter.) & $O(MN^2\sqrt{N})$\\
  \hspace{-4pt} \citet{sigal06} & Pose reconstruction & $\Theta(MN^3)$ & $O(MN^2\sqrt{N})$\\
  \hspace{-4pt} \citet{pedro_shapes} & Deformable matching & $\Theta(MN^3)$ & $\Theta(MN^{\frac{8}{3}})$ (online)\\
  \hspace{-4pt} \citet{interval} & Width-$D$ interval graph & $O(MN^{D+1})$ & $O(MN^{D}\sqrt{N})$\\
  \hspace{-4pt} \citet{sontag_lp} & LP with $M$ clusters & $\Theta(MN^3)$ & $O(MN^2\sqrt{N})$\\
\hline
  \end{tabular}
  \normalsize
 \end{center}
\label{tab:improvements}
\end{table*}

\section{Experiments}
\label{sec:experiments}

We present experimental results for two types of models: those whose cliques factorize into smaller terms, as discussed in Section \ref{sec:optimizing}, and those whose factors \emph{that depend upon the observation} contain fewer latent variables than their maximal cliques, as discussed in Section \ref{sec:latent}.

\subsection{Experiments with Within-Clique Factorization}
\label{sec:within}

In this section we present experiments in models whose cliques factorize into smaller terms, as discussed in Section \ref{sec:optimizing}. We also use this section to demonstrate Theorems \ref{the:alg1} and \ref{the:algext} experimentally.

\subsubsection{Comparison Between Asymptotic Performance and Upper-Bounds}

For our first experiment, we compare the performance of Algorithms \ref{alg1} and \ref{alg:ext} to the na\"ive solution of Algorithm \ref{alg:brute}. These are core subroutines of each of the other algorithms, meaning that determining their performance shall give us an accurate indication of the improvements we expect to obtain in real graphical models.

For each experiment, we generate $N$ i.i.d.~samples from $[0,1)$ to obtain the lists $v_1 \ldots v_{K}$. $N$ is the domain size; this may refer to a single node, or a \emph{group} of nodes as in Algorithm \ref{alg:ext}; thus large values of $N$ may appear even for binary-valued models. $K$ is the number of lists in \eq{eq:hatK}; we can observe this number of lists only if we are working in cliques of size $K+1$, and then only if the factors are of size $K$ (e.g. we will only see $K=5$ if we have cliques of size 6 with factors of size 5); therefore smaller values of $K$ are probably more realistic in practice (indeed, all of the applications in Section \ref{sec:improvements} have $K=2$).

The performance of our algorithm is shown in Figure \ref{fig:exp1}, for $K=2$ to $4$ (i.e., for 2 to 4 lists). When $K=2$, we execute Algorithm \ref{alg1}, while Algorithm \ref{alg:ext} is executed for $K\geq 3$.
The performance reported is simply the number of elements read from the lists (which is at most $K\times\mathit{start}$).
This is compared to $N$ itself, which is the number of elements read by the na\"ive algorithm. The upper-bounds we obtained in \eq{eq:bound} are also reported, while the true expected performance (i.e., \eq{eq:runtimek1}) is reported for $K=2$. Note that the variable $\mathit{read}$ was introduced into Algorithm \ref{alg:ext} in order to guarantee that it can never be asymptotically slower than the na\"ive algorithm. If this variable is ignored, the performance of our algorithm deteriorates to the point that it closely approaches the upper-bounds shown in Figure \ref{fig:exp1}. Unfortunately, this optimization proved overly complicated to include in our analysis, meaning that our upper-bounds remain highly conservative for large $K$.

\begin{figure*}
 \begin{center}
  \includegraphics[width=0.33\textwidth]{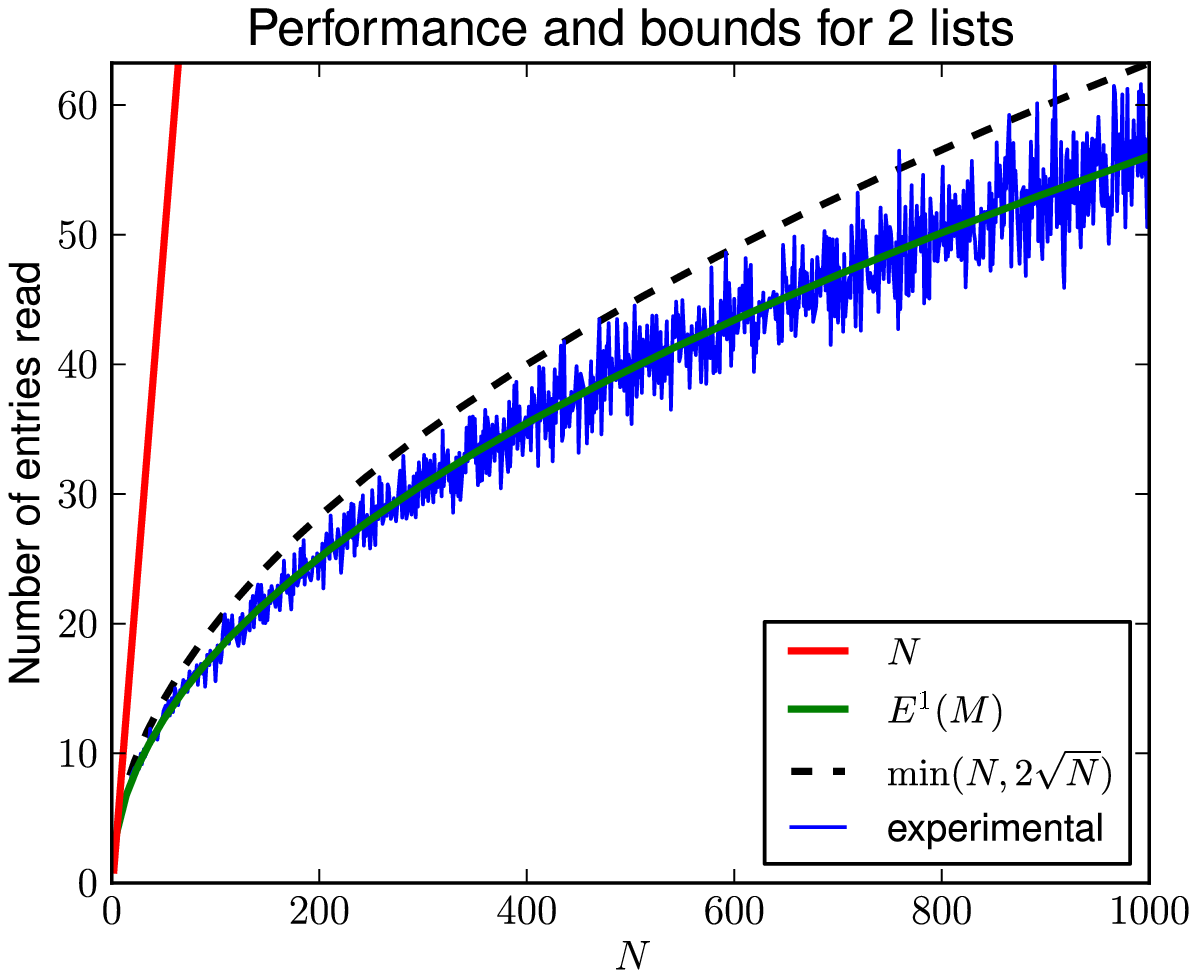}\includegraphics[width=0.33\textwidth]{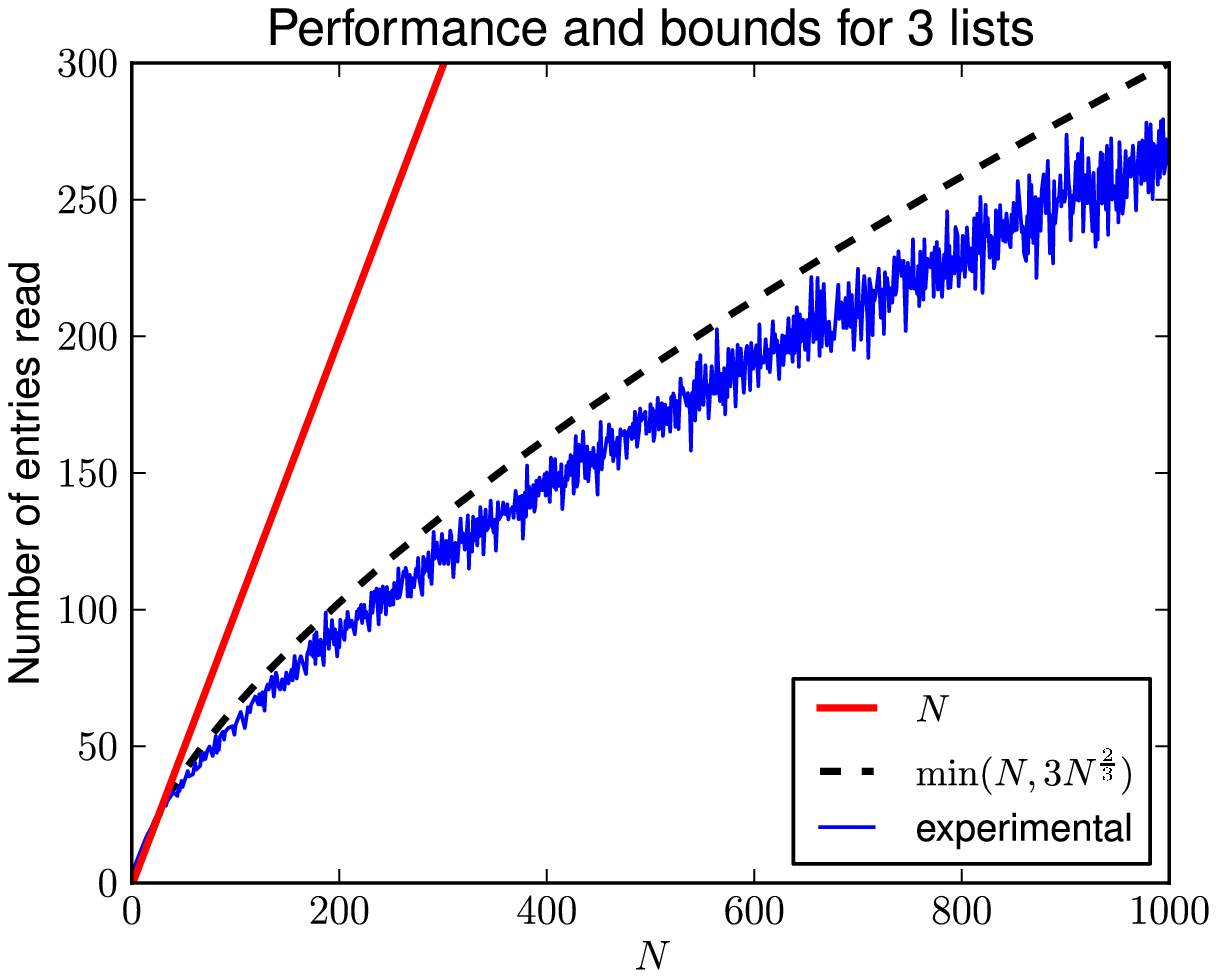}\includegraphics[width=0.33\textwidth]{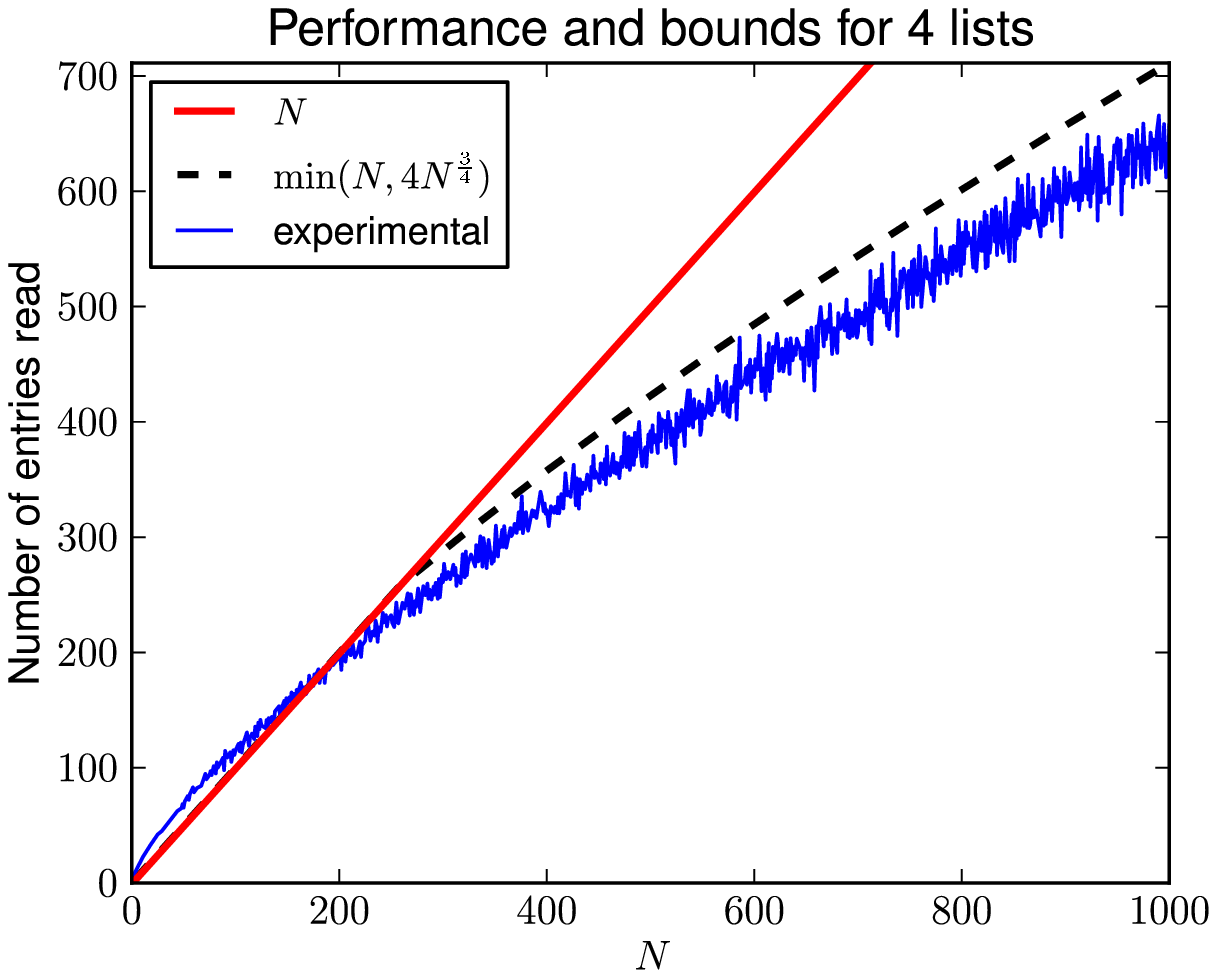}
 \end{center}
\caption{Performance of our algorithm and bounds. For $K=2$, the exact expectation is shown, which appears to precisely match the average performance (over 100 trials). The dotted lines show the bound of \eq{eq:bound}. While the bound is close to the true performance for $K=2$, it becomes increasingly loose for larger $K$.}
\label{fig:exp1}
\end{figure*}


\subsubsection{Performance Improvement for Dependent Variables}
\label{sec:exp_dependent}

The expected-case running time of our algorithm was derived under the assumption that each list has independent order statistics, as was the case for our previous experiment. We suggested that we will obtain worse performance in the case of negatively correlated variables, and better performance in the case of positively correlated variables; we shall assess these claims in this experiment.

Figure \ref{fig:permutations} shows how the order-statistics of $\mathbf{v}_a$ and $\mathbf{v}_b$ can affect the performance of our algorithm. Essentially, the running time of Algorithm \ref{alg1} is determined by the level of `diagonalness' of the permutation matrices in Figure \ref{fig:permutations}; highly diagonal matrices result in better performance than the expected case, while highly off-diagonal matrices result in worse performance. The expected case was simply obtained under the assumption that every permutation is equally likely.

\begin{figure*}
\begin{center}
\small
 \begin{tabular}{rccccc}
               & \hspace{-14mm} $\leftarrow$ best case \hspace{-20mm} \ \\
 permutation:  & \perm{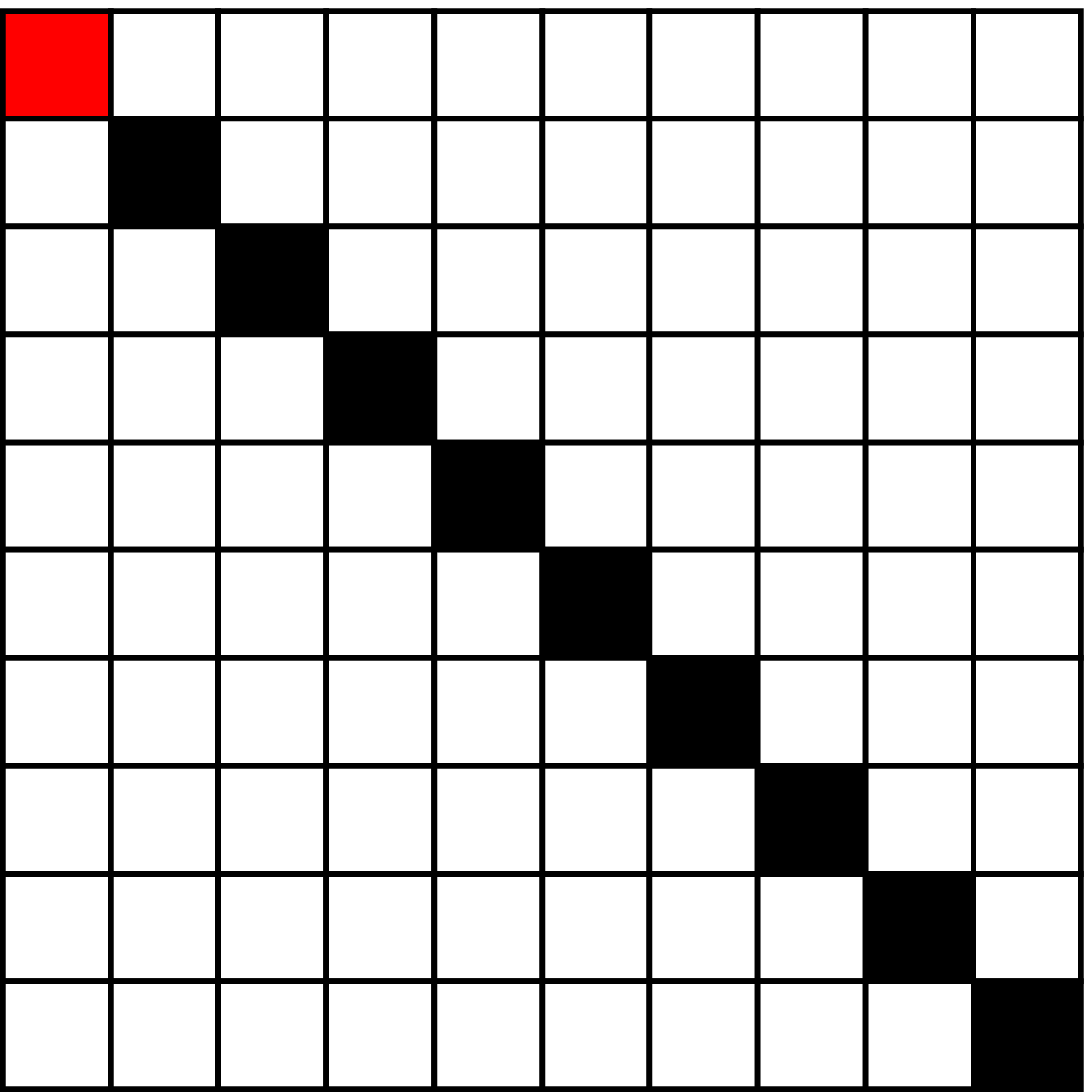} & \perm{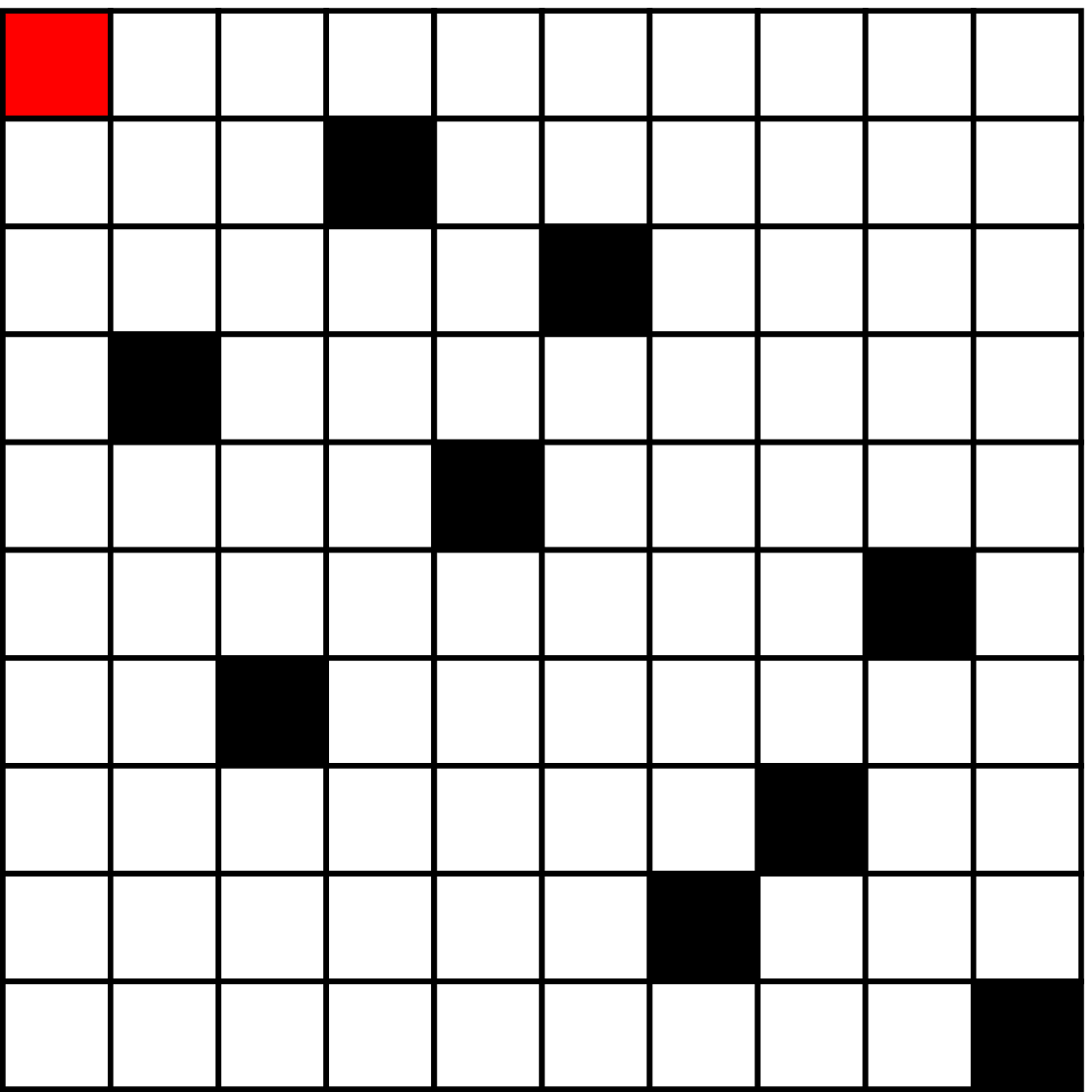} & \perm{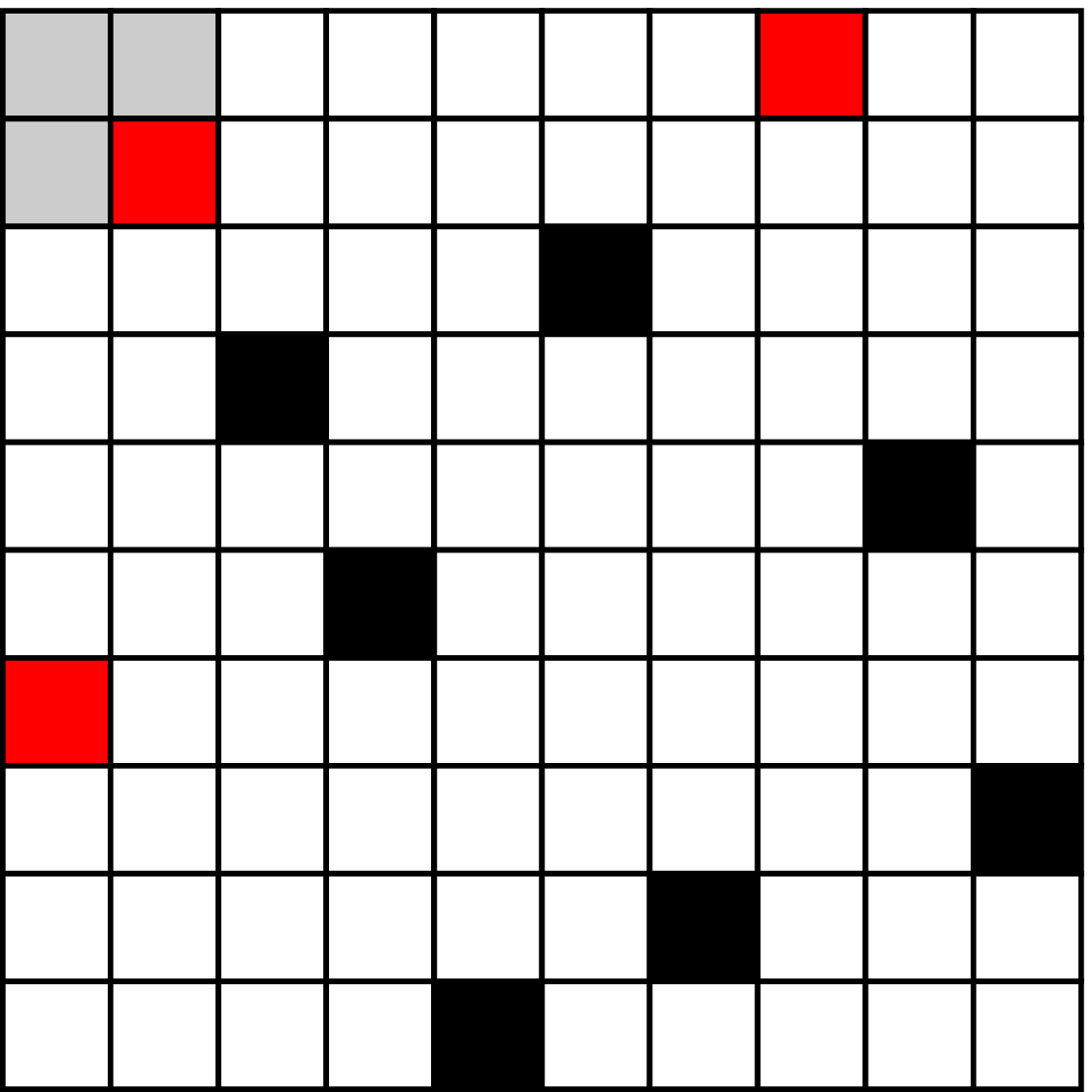} & \perm{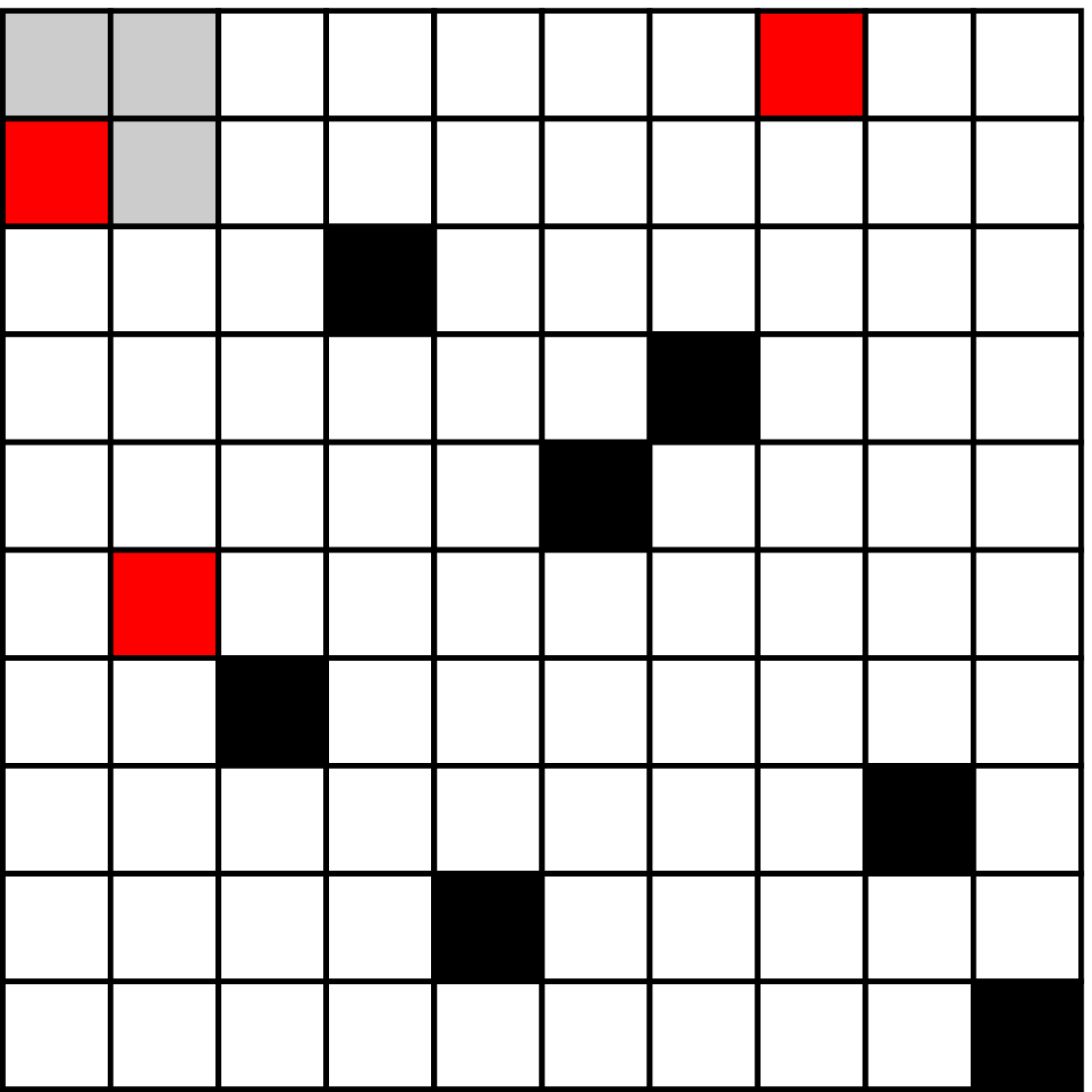} &  \perm{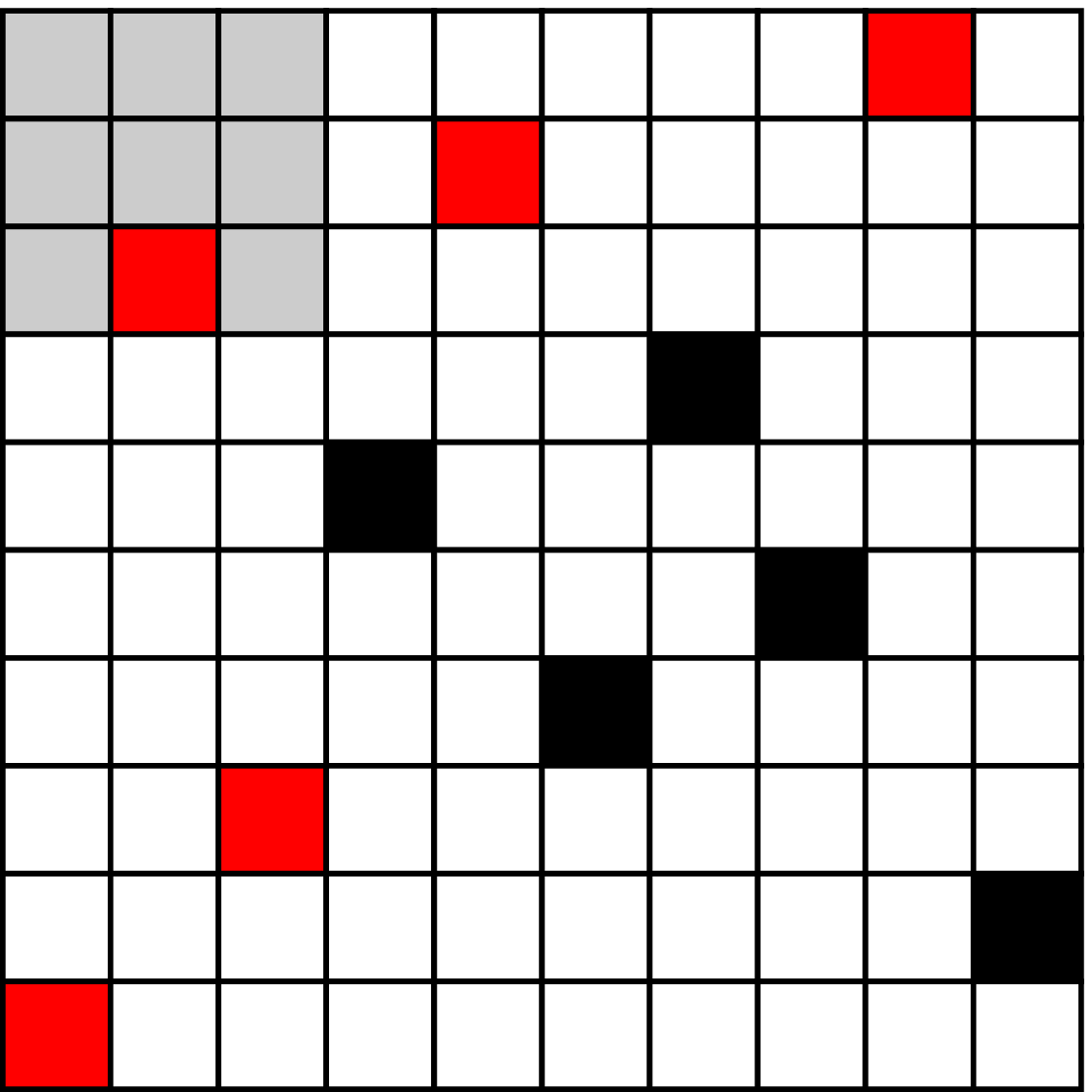} \\[25pt]
 operations: & 1 & 1 & 3 & 3 & 5 \\[14pt]

 &&&&& \hspace{-20mm} worst case $\rightarrow$ \hspace{-11mm}\ \\
 permutation   & \perm{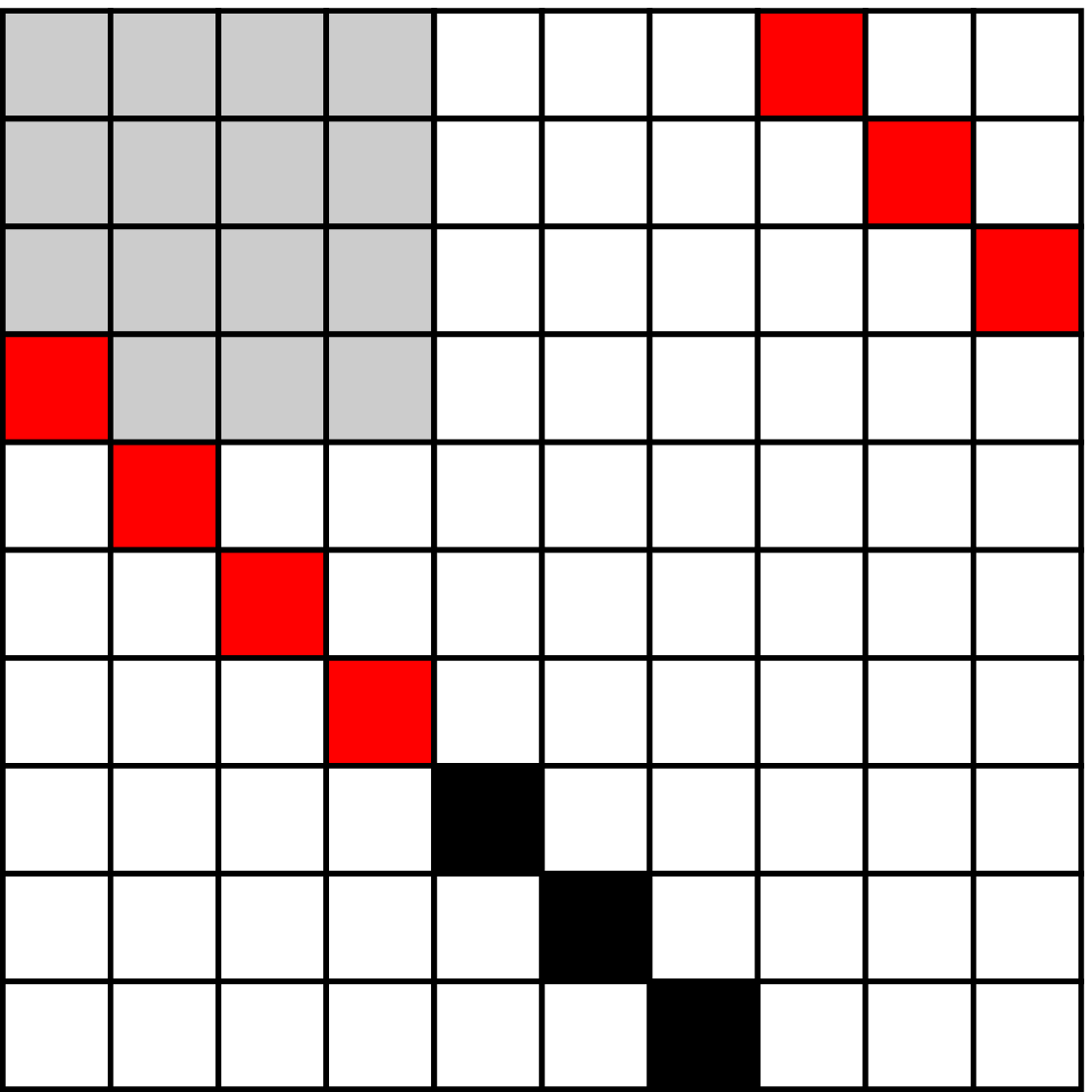} & \perm{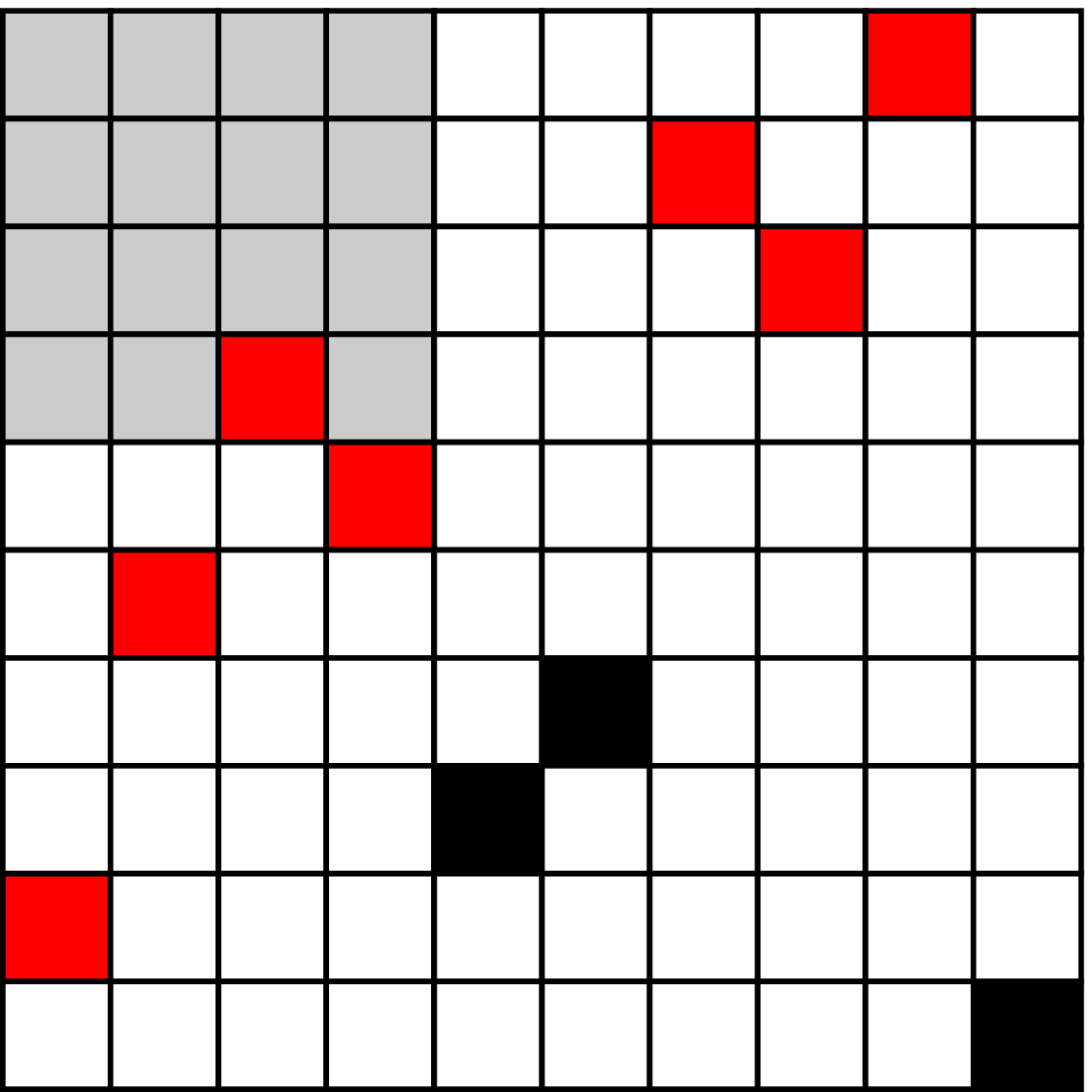} & \perm{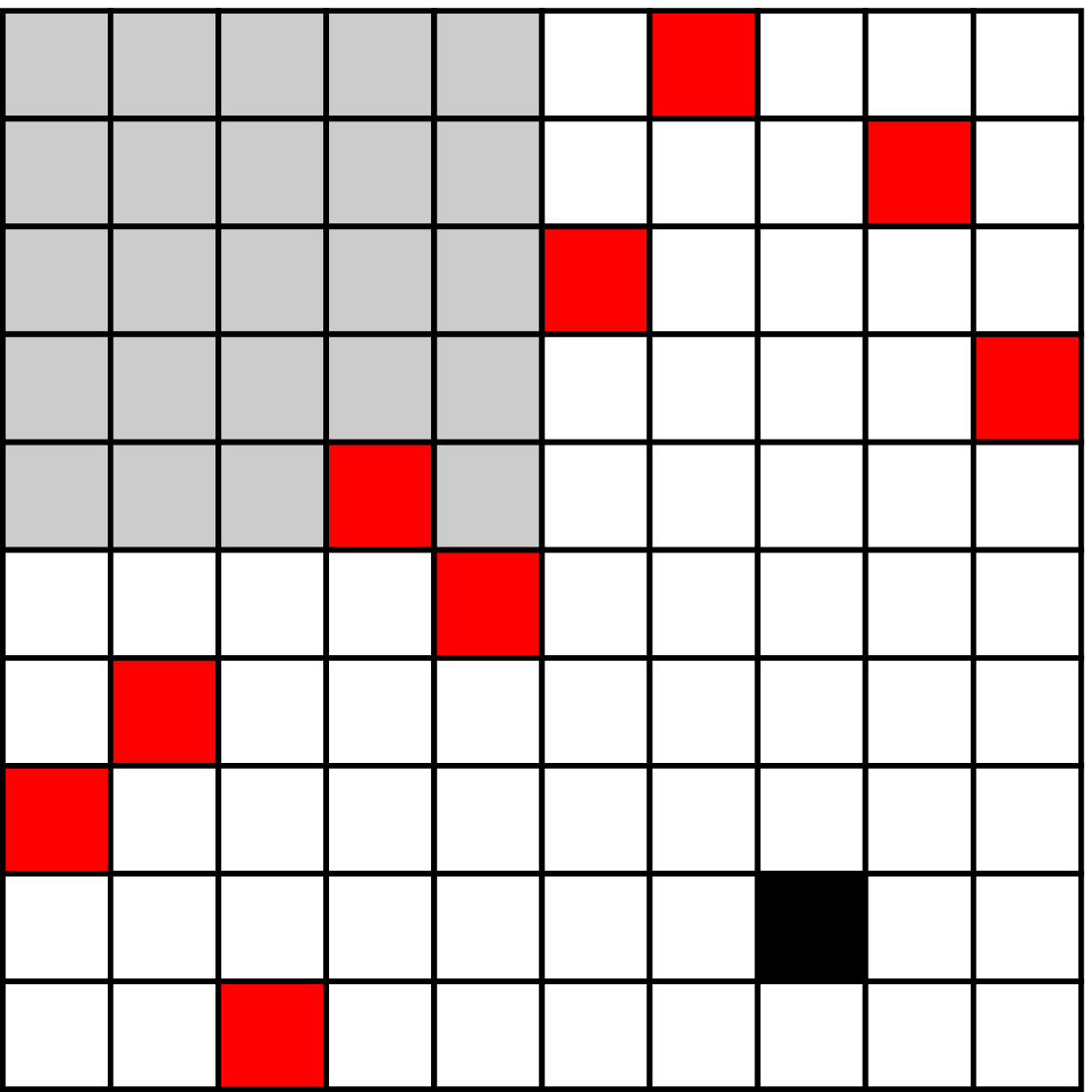} & \perm{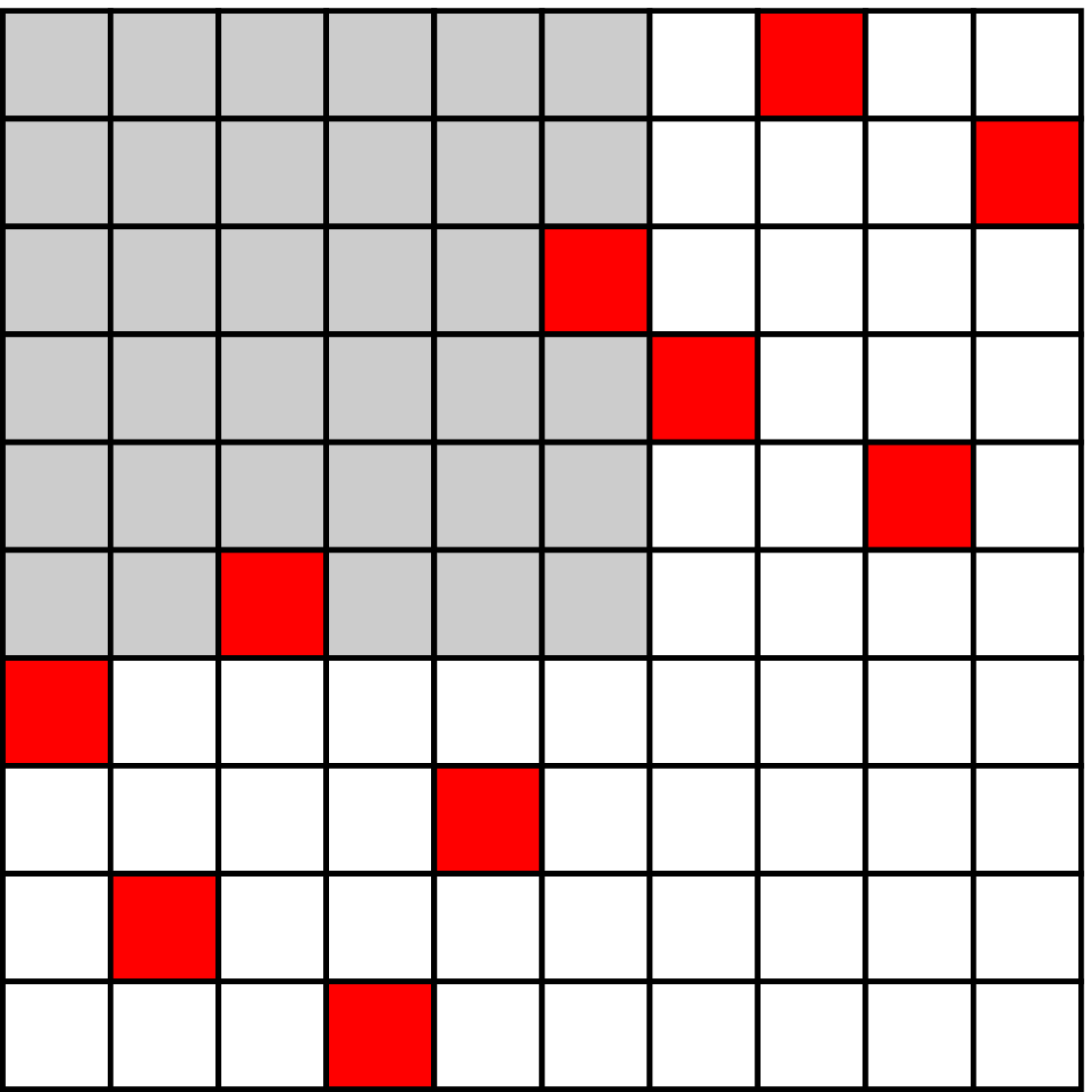} & \perm{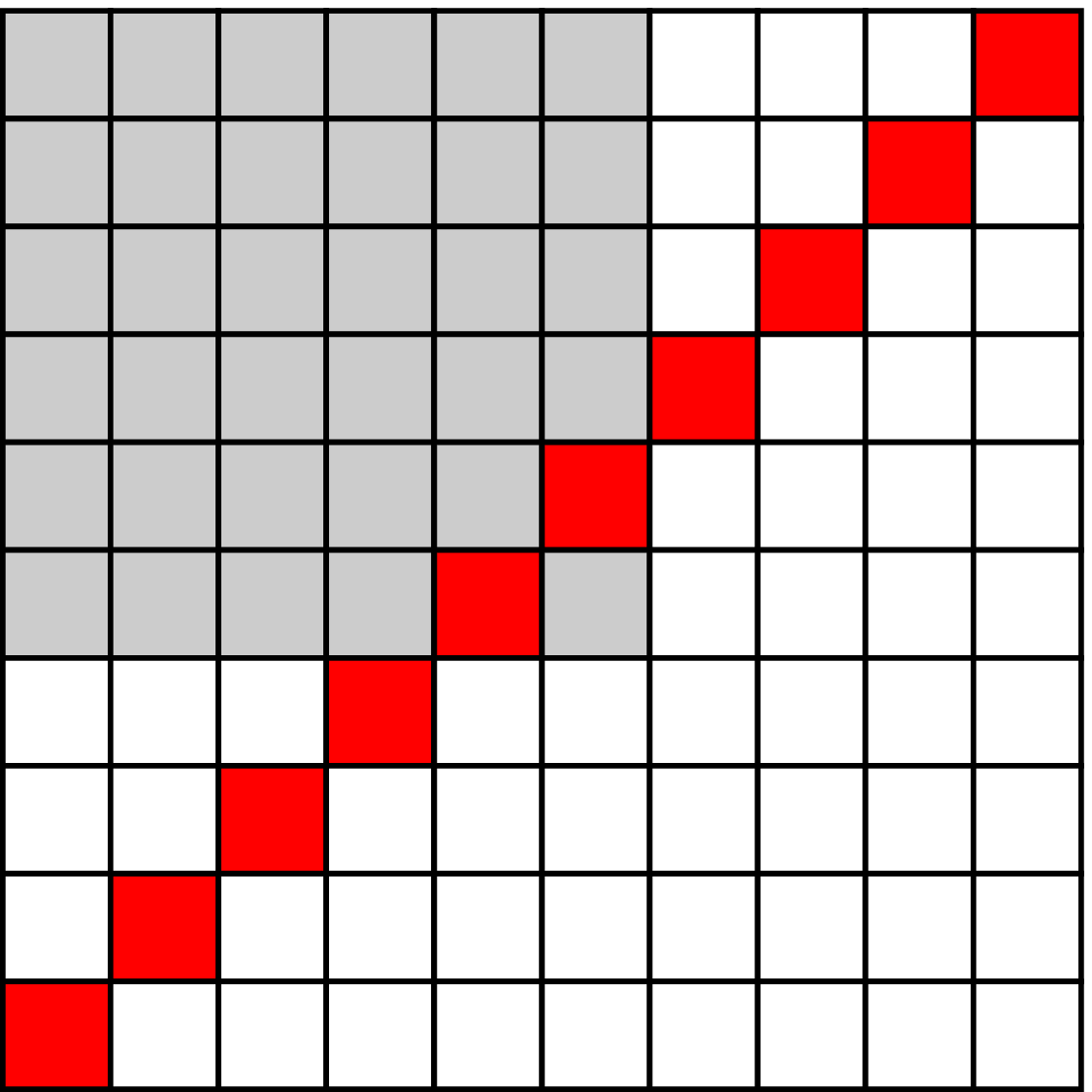} \\[25pt]
 operations: & 7 & 7 & 9 & 10 & 10
 \end{tabular}
 \normalsize
\end{center}
\caption{Different permutation matrices and their resulting cost (in terms of entries read/multiplications performed). Each permutation matrix transforms the \emph{sorted} values of one list into the sorted values of the other, i.e., it transforms $\mathbf{v}_a$ as sorted by $p_a$ into $\mathbf{v}_b$ as sorted by $p_b$.
The red squares show the entries that must be read before the algorithm terminates (each corresponding to one multiplication). See Figure \ref{fig:alg1_aistats} for further explanation.
}
\label{fig:permutations}
\end{figure*}

We report the performance for two lists (i.e., for Algorithm \ref{alg1}), whose values are sampled from a 2-dimensional Gaussian, with covariance matrix
\begin{equation}
 \Sigma = \left[ \begin{array}{cc} 1 & c \\ c & 1\end{array} \right],
\end{equation}
meaning that the two lists are correlated with correlation coefficient $c$. In the case of Gaussian random variables, the correlation coefficient precisely captures the `diagonalness' of the matrices in Figure \ref{fig:permutations}. Performance is shown in Figure \ref{fig:correlated} for different values of $c$ ($c=0$, is not shown, as this is the case observed in the previous experiment).

\begin{figure*}
 \begin{center}
  \includegraphics[width=0.33\textwidth]{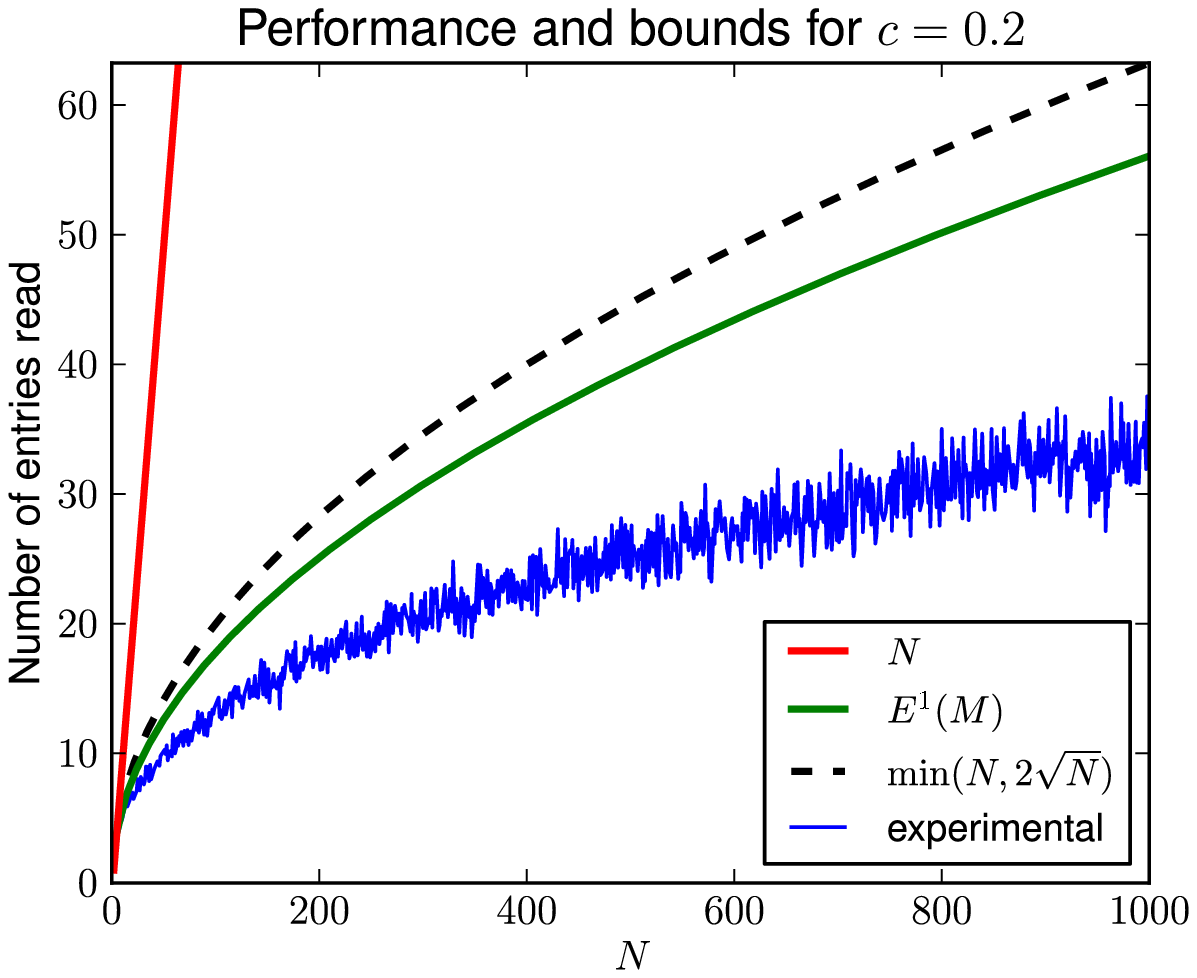}\includegraphics[width=0.33\textwidth]{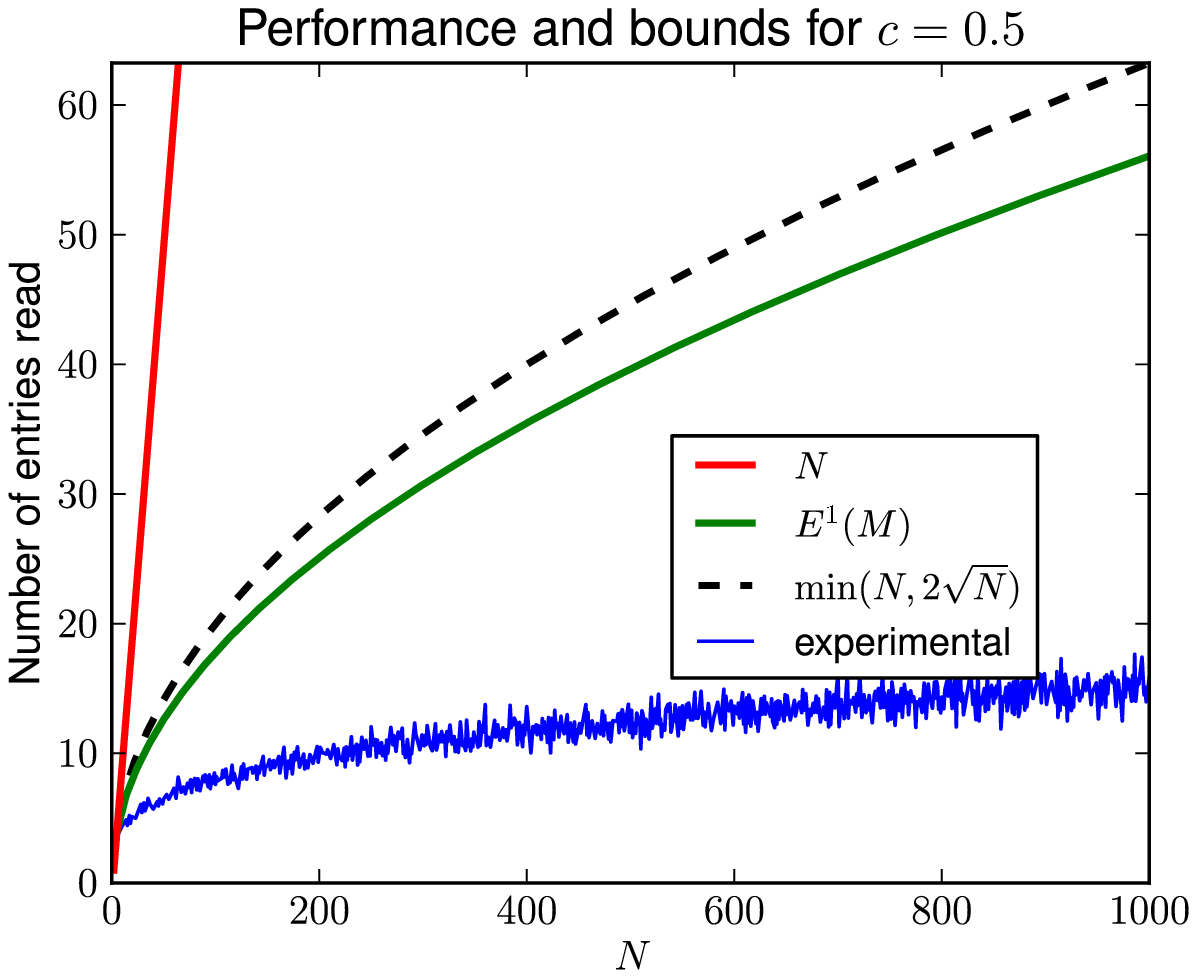}\includegraphics[width=0.33\textwidth]{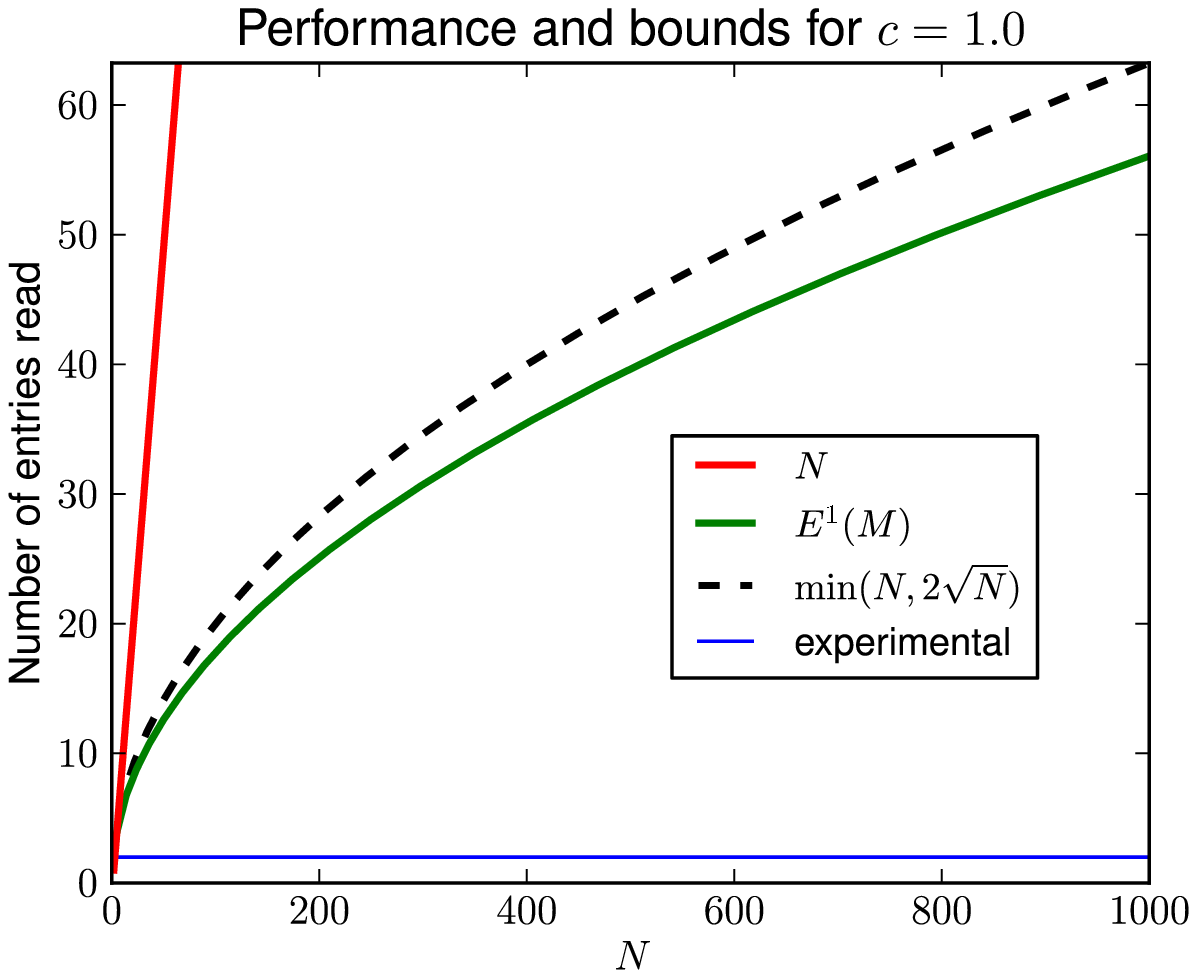}
  \includegraphics[width=0.33\textwidth]{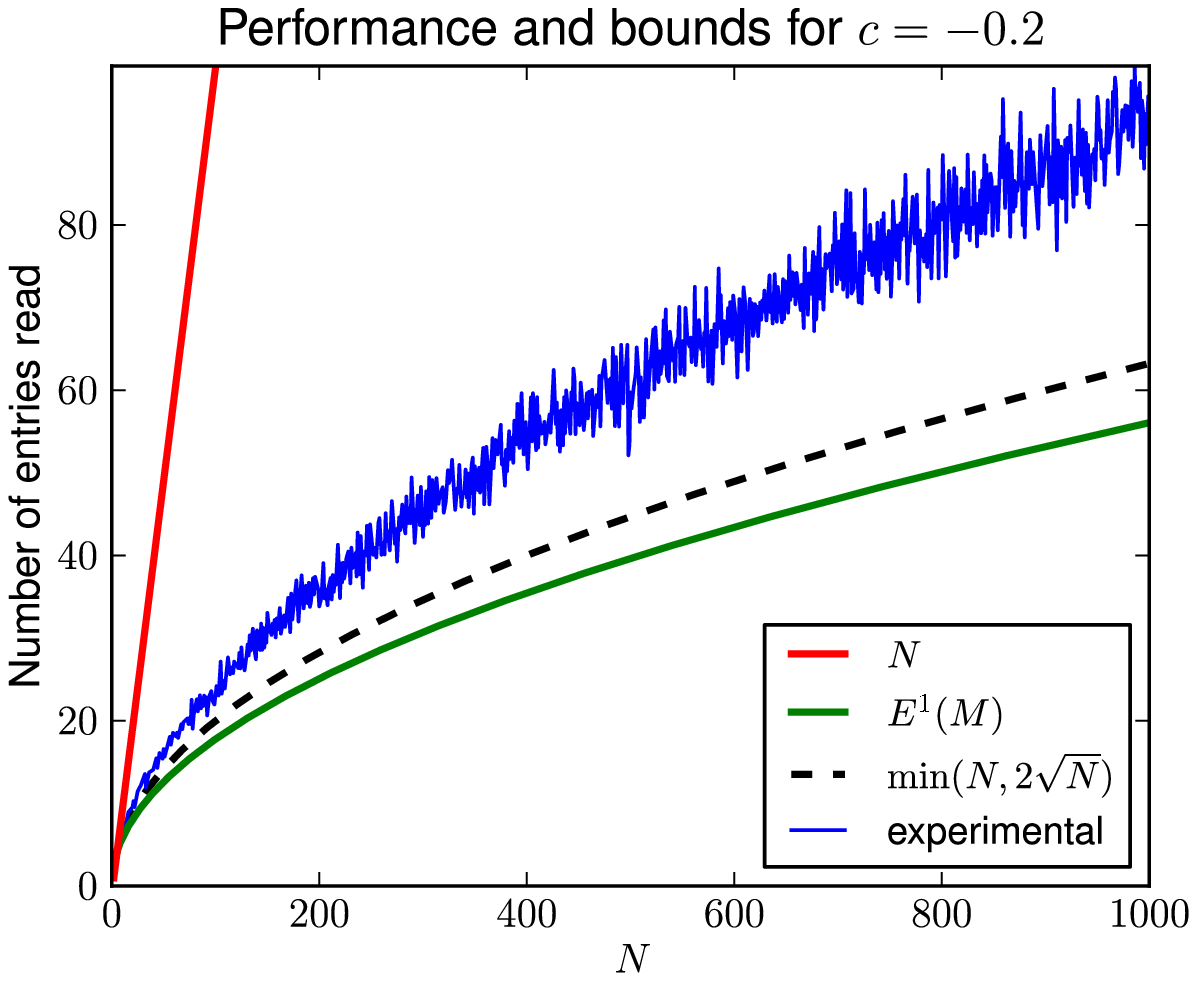}\includegraphics[width=0.33\textwidth]{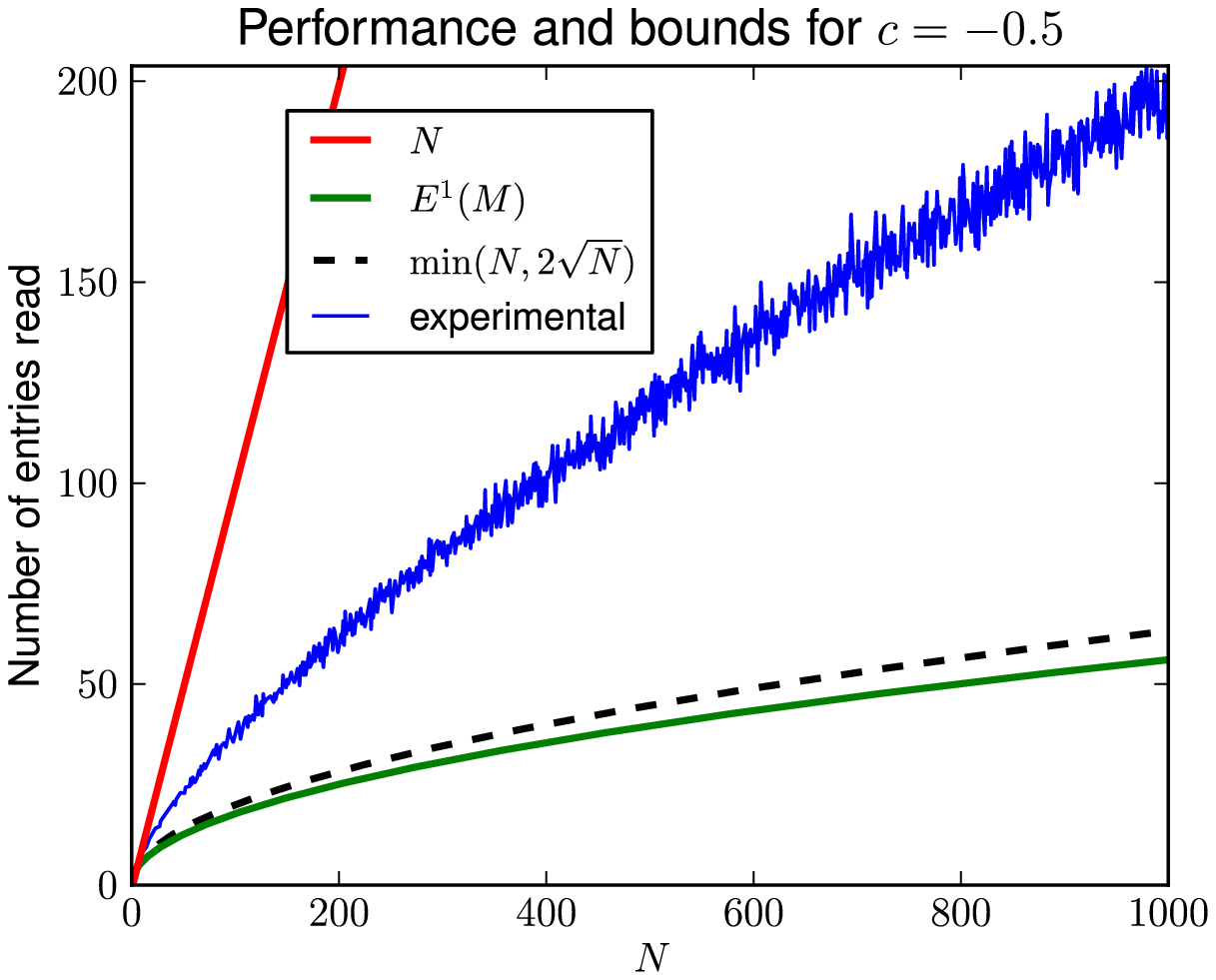}\includegraphics[width=0.33\textwidth]{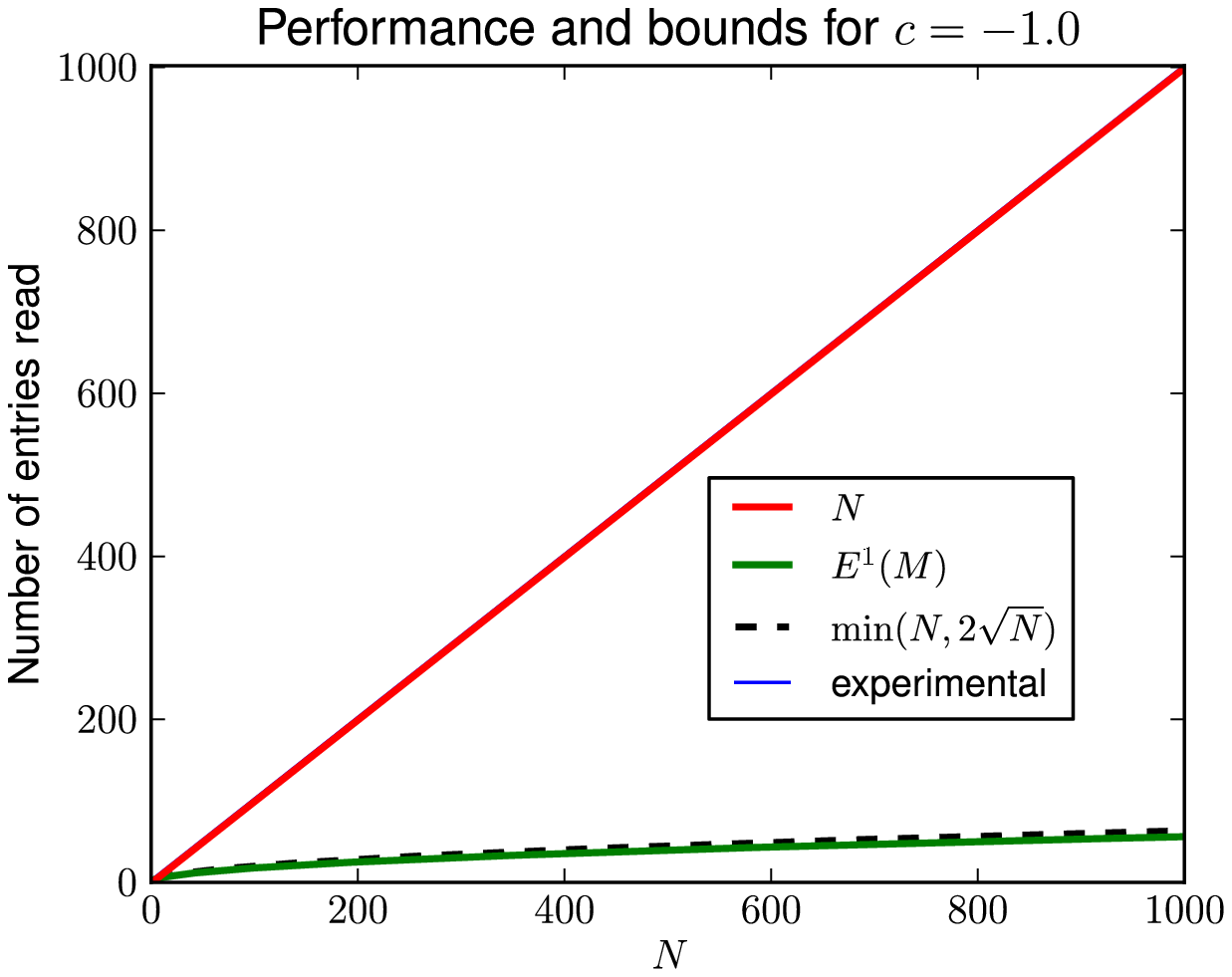}
 \end{center}
\caption{Performance of our algorithm for different correlation coefficients. The top three plots show positive correlation, the bottom three show negative correlation. Correlation coefficients of $c = 1.0$ and $c = -1.0$ capture precisely the best and worst-case performance of our algorithm, resulting in $O(1)$ and $\Theta(N)$ performance (respectively).}
\label{fig:correlated}
\end{figure*}

\subsubsection{2-Dimensional Graph Matching}

Naturally, Algorithm \ref{alg2} has additional overhead compared to the na\"ive solution, meaning that it will not be beneficial for small $N$. In this experiment, we aim to assess the extent to which our approach is useful in real applications. We reproduce the model from \citet{McACaeBar08}, which performs 2-dimensional graph matching, using a loopy graph with cliques of size three, containing only second order potentials (as described in Section \ref{sec:improvements}); the $\Theta(NM^3)$ performance of \citet{McACaeBar08} is reportedly state-of-the-art. We also show the performance on a graphical model with \emph{random} potentials, in order to assess how the results of the previous experiments are reflected in terms of actual running time.

We perform matching between a \emph{template} graph with $M$ nodes, and a \emph{target} graph with $N$ nodes, which requires a graphical model with $M$ nodes and $N$ states per node (see \citet{McACaeBar08} for details). We fix $M=10$ and vary $N$. Performance is shown in Figure \ref{fig:graphmatch}. Fitted curves are shown together with the actual running time of our algorithm, confirming its $O(MN^2\sqrt{N})$ performance. The coefficients of the fitted curves demonstrate that our algorithm is useful even for modest values of $N$.

We also report results for graph matching using graphs from the MPEG-7 dataset \citep{mpeg7}, which consists of 1,400 silhouette images. Again we fix $M=10$ (i.e., 10 points are extracted in each template graph) and vary $N$ (the number of points in the target graph). This experiment confirms that even when matching real-world graphs, the assumption of independent order-statistics appears to be reasonable.


\begin{figure}
 \begin{center}
    \includegraphics[width=0.5\textwidth]{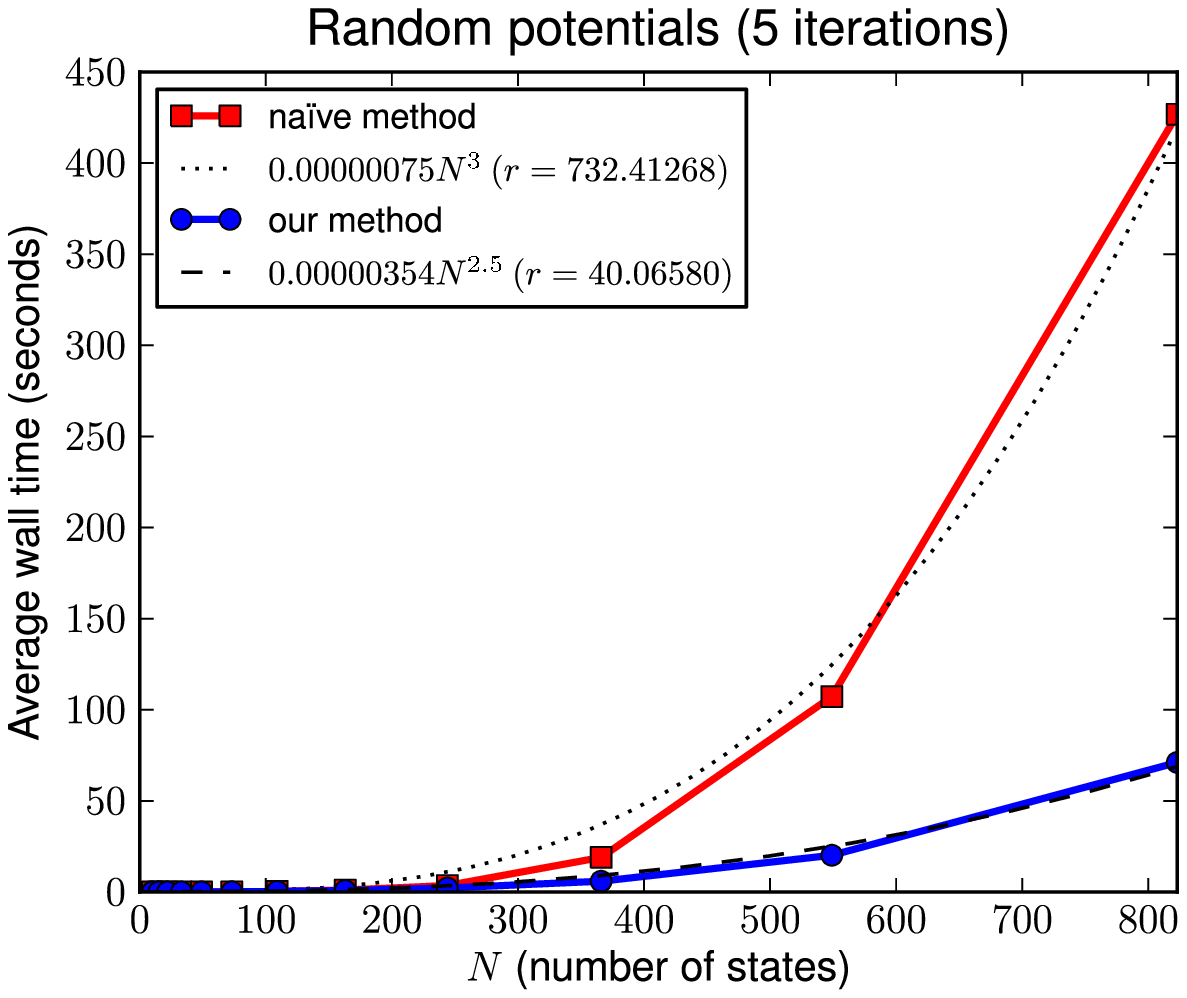}\includegraphics[width=0.5\textwidth]{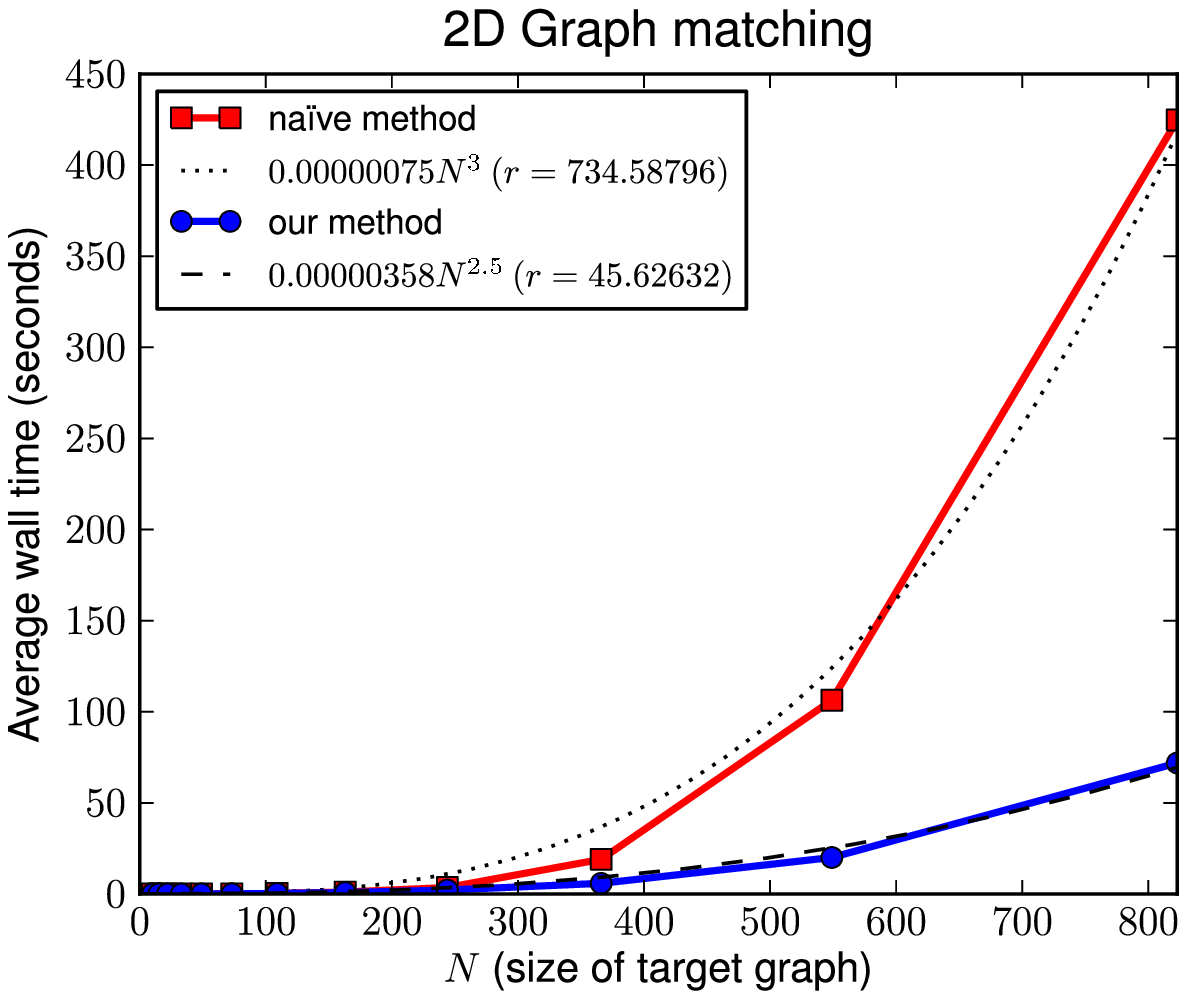}
 \end{center}
\caption{The running time of our method on randomly generated potentials, and on a graph matching experiment (both graphs have the same topology). Fitted curves are also obtained by performing least-squares regression; the residual error $r$ indicates the `goodness' of the fitted curve.}
\label{fig:graphmatch}
\end{figure}

\begin{figure}
 \begin{center}
    \includegraphics[width=0.5\textwidth]{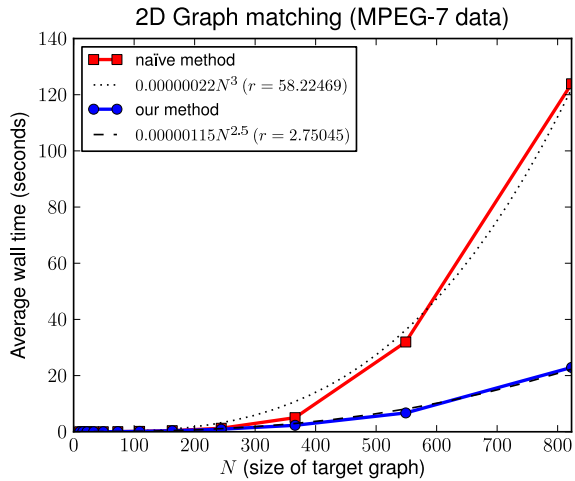}
 \end{center}
\caption{The running time of method our on graphs from the MPEG-7 dataset.}
\label{fig:fpeg}
\end{figure}

\subsubsection{Higher-Order Markov Models}
\label{sec:text_denoise}

In this experiment, we construct a simple Markov model for text-denoising. Random noise is applied to a text segment, which we try to correct using a prior extracted from a text corpus. For instance

\begin{center}
\texttt{wondrous sight of th4 ivory Pequod} \ \ is corrected to \ \ \texttt{wondrous sight of the ivory Pequod}.
\end{center}

In such a model, we would like to exploit higher-order relationships between characters, though the amount of data required to construct an accurate prior grows exponentially with the size of the maximal cliques. Instead, our prior consists entirely of pairwise relationships between characters (or `bigrams'); higher-order relationships are encoded by including bigrams of non-adjacent characters. Specifically, our model takes the form
\begin{equation}
 \Phi_X(\mathbf{x}_X) = \prod_{i=1}^{|X| - 1} \Phi_{i,i+1}(x_i, x_{i+1}) \times \prod_{i=1}^{|X| - 2} \Phi_{i,i+2}(x_i, x_{i+2})
\end{equation}
where
\begin{equation}
 \Phi_{i,j}(x_i, x_j) = \psi_{i,j}(x_i, x_j)p(x_i | o_i)p(x_j | o_j).
\end{equation}
Here $\psi$ is our \emph{prior} (extracted from text statistics), and $p$ is our `noise model' (given the observation $\mathbf{o}$). The computational complexity of inference in this model is similar to that of the skip-chain CRF shown in Figure \ref{fig:tree3}(b), as well as models for part-of-speech tagging and named-entity recognition, as in Figure \ref{fig:denoise}. Text denoising is useful for the purpose of demonstrating our algorithm, as there are several different corpora available in different languages, allowing us to explore the effect that the domain size (i.e., the size of the language's alphabet) has on running time.

\begin{figure*}
 \begin{center}
  \parbox[c]{0.58\textwidth}{\centering\includegraphics[angle=-90,scale=0.24]{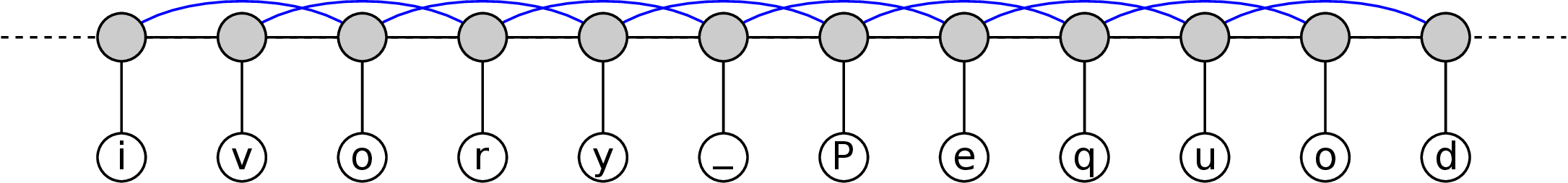}} \parbox[c]{0.4\textwidth}{\centering\includegraphics[angle=-90,scale=0.24]{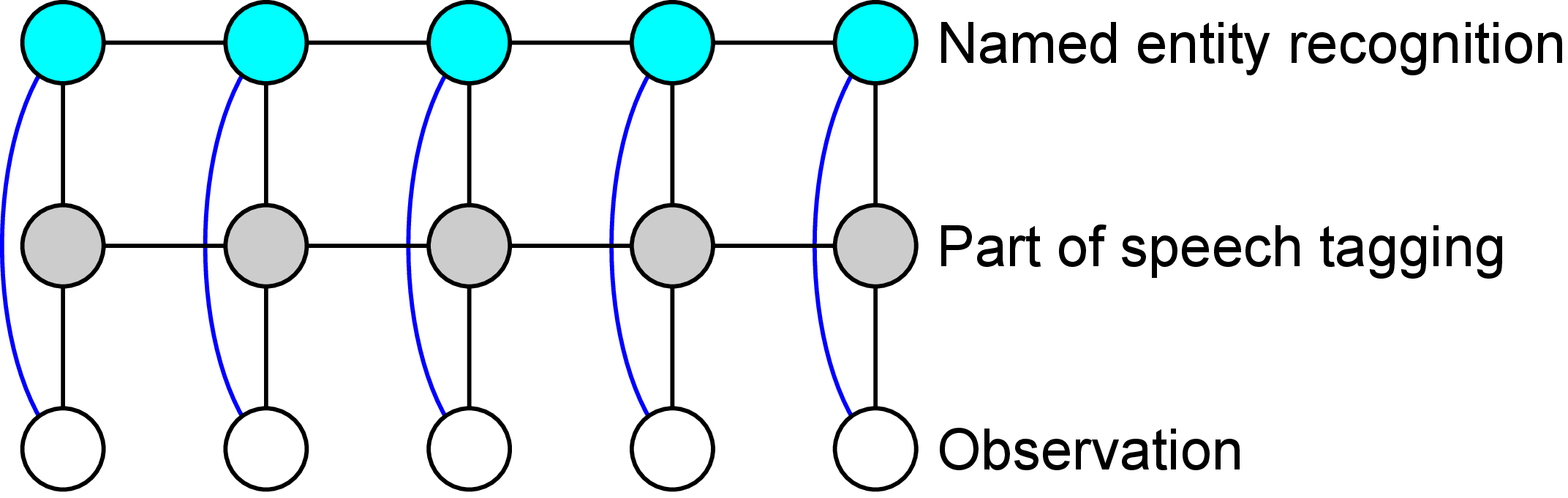}}
 \end{center}
\caption{Left: Our model for denoising. Its computational complexity is similar to that of a skip-chain CRF, and models for named-entity recognition (right).}
\label{fig:denoise}
\end{figure*}

We extracted pairwise statistics based on 10,000 characters of text, and used this to correct a series of 25 character sequences, with 1\% random noise introduced to the text. The domain was simply the set of characters observed in each corpus. The Japanese dataset was not included, as the $\Theta(MN^2)$ memory requirements of the algorithm made it infeasible with $N \simeq 2000$; this is addressed in Section \ref{sec:chain}.

The running time of our method, compared to the na\"ive solution, is shown in Figure \ref{fig:denoise_res}. One might expect that texts from different languages would exhibit different dependence structures in their order statistics, and therefore deviate from expected case in some instances. However, the running times appear to follow the fitted curve closely, i.e., we are achieving approximately the expected-case performance in all cases.

Since the prior $\psi_{i,i+1}(x_i, x_{i+1})$ is \emph{data-independent}, we shall further discuss this type of model in reference to Algorithm \ref{alg:message} in Section \ref{sec:latent_exp}.

\begin{figure}
 \begin{center}
  \includegraphics[width=0.5\textwidth]{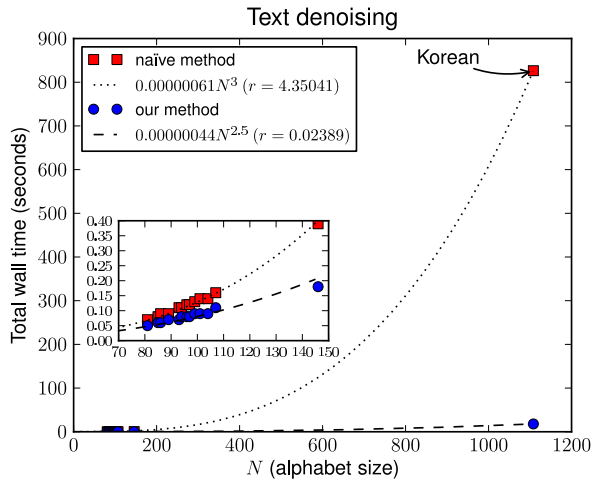}
 \end{center}
\caption{The running time of our method compared to the na\"ive solution. A fitted curve is also shown, whose coefficient estimates the computational overhead of our model.}
\label{fig:denoise_res}
\end{figure}

\subsubsection{Protein Design}
\label{sec:protein}

In \citet{sontag_lp}, a method is given for exact MAP-inference in graphical models using LP-relaxations. Where exact solutions cannot be obtained by considering only pairwise factors, `clusters' of pairwise terms are introduced in order to refine the solution. Message-passing in these clusters turns out to take exactly the form that we consider, as third-order (or larger) clusters are formed from pairwise terms. Although a number of applications are presented in \citet{sontag_lp}, we focus on protein design, as this is the application in which we typically observe the largest domain sizes. Other applications with larger domains may yield further benefits.

Without going into detail, we simply copy the two equations in \citet{sontag_lp} to which our algorithm applies. The first of these is concerned with passing messages between clusters, while the second is concerned with choosing new clusters to add. Below are the two equations, reproduced verbatim from \citet{sontag_lp}:
\begin{equation}
 \lambda_{c\rightarrow e}(x_e) ~\leftarrow~ -\frac{2}{3}\bigl( \lambda_{e\rightarrow e}(x_e) + \!\!\!\sum_{c' \neq c, e \in c'}\!\!\! \lambda_{c' \rightarrow e}(x_e)\bigr) + \frac{1}{3} \max_{x_{c \setminus e}}\Bigl[ \sum_{e' \in c \setminus e} \bigl( \lambda_{e' \rightarrow e'}(x_{e'}) + \!\!\!\sum_{c' \neq c, e' \in c'}\!\!\! \lambda_{c' \rightarrow e'}(x_{e'})\bigr) \Bigr]
\label{eq:sontag1}
\end{equation}
\citep[see][Figure 1, bottom]{sontag_lp}, which consists of marginalizing a cluster ($c$) that decomposes into edges ($e$), and
\begin{equation}
 d(c) = \sum_{e\in c} \max_{x_e} b_e(x_e) - \max_{x_c} \left[ \sum_{e\in c} b_e(x_e) \right],
 \label{eq:sontag2}
\end{equation}
\citep[see][(eq.~4)]{sontag_lp}, which consists of finding the MAP state in a ring-structured model.

As the code from \citet{sontag_lp} was publicly available, we simply replaced the appropriate functions with our own (in order to provide a fair comparison, we also replaced their implementation of the na\"ive algorithm, as ours proved to be faster than the highly generic matrix library used in their code).

In order to improve the running time of our algorithm, we made the following two modifications to Algorithm \ref{alg1}:
\begin{itemize}
 \item We used an \emph{adaptive sorting algorithm} (i.e., a sorting algorithm that runs faster on nearly-sorted data). While quicksort was used during the first iteration of message-passing, subsequent iterations used insertion sort, as the optimal ordering did not change significantly between iterations.
 \item We added an additional stopping criterion to the algorithm. Namely, we terminate the algorithm if $\mathbf{v}_a[p_a[\mathit{start}]]\times\mathbf{v}_b[p_b[\mathit{start}]] < \mathit{max}$. In other words, we check how large the maximum \emph{could be} given the best possible permutation of the next elements (i.e., if they have the same index); if this value could not result in a new maximum, the algorithm terminates. This check costs us an additional multiplication, but it means that the algorithm will terminate faster in cases where a large maximum is found early on.
\end{itemize}

Results for these two problems are shown in Figure \ref{fig:protein}. Although our algorithm consistently improves upon the running time of \citet{sontag_lp}, the domain size of the variables in question is not typically large enough to see a marked improvement. Interestingly, neither method follows the expected running time closely in this experiment. This is partly due to the fact that there is significant variation in the variable size (note that $N$ only shows the \emph{average} variable size), but it may also suggest that there is a complicated structure in the potentials which violates our assumption of independent order statistics.

\begin{figure}
 \begin{center}
  \includegraphics[width=0.5\textwidth]{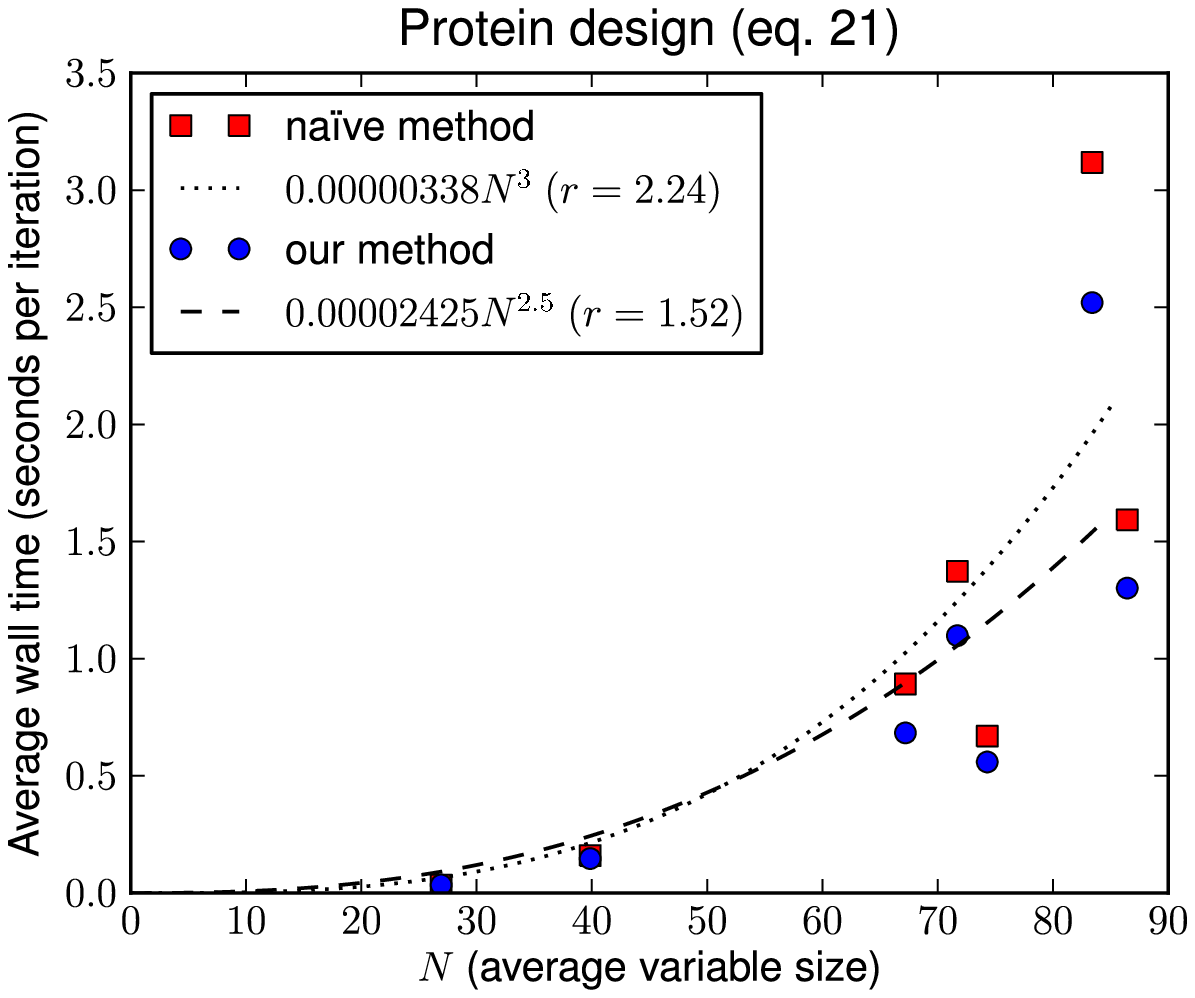}\includegraphics[width=0.5\textwidth]{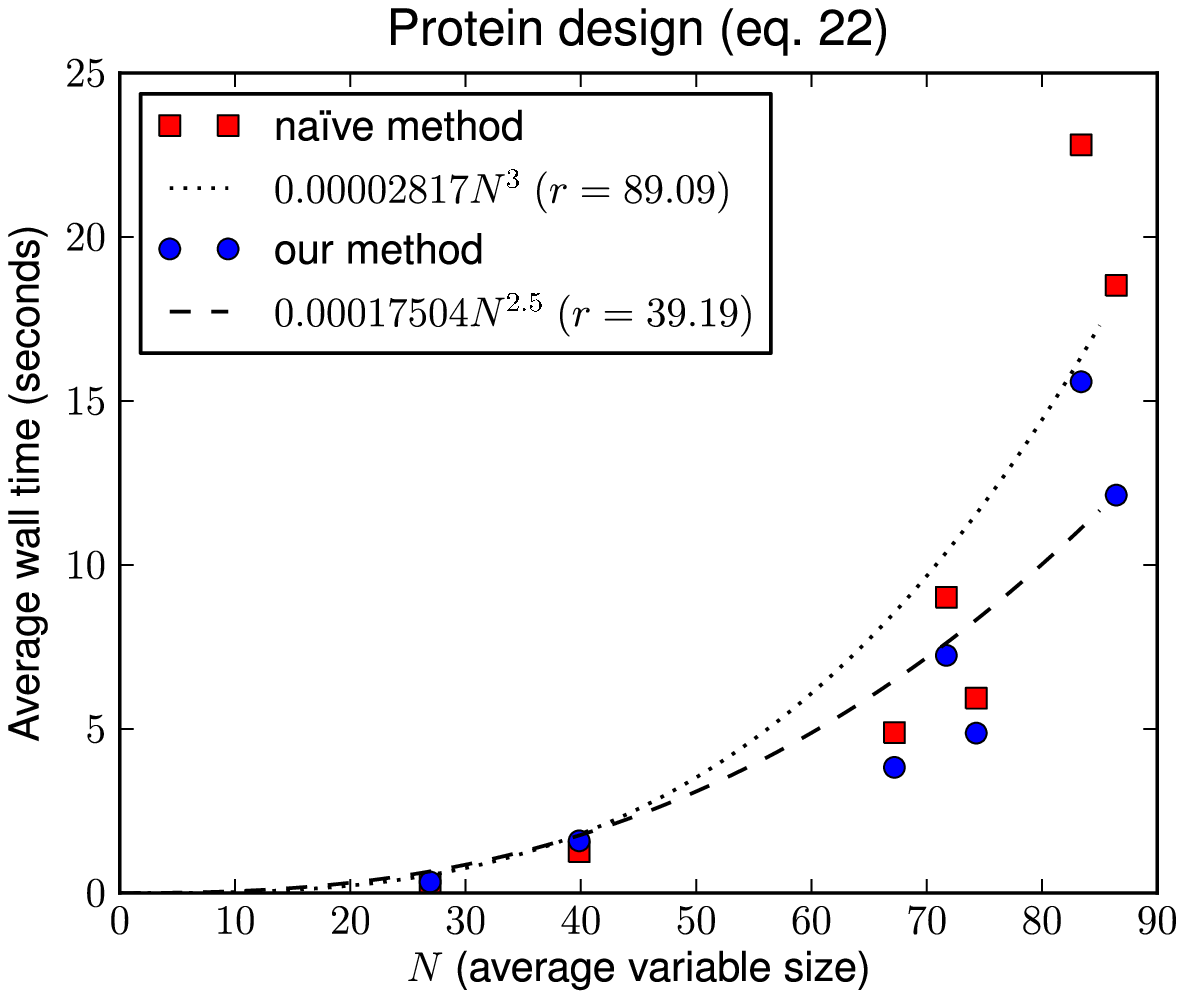}
 \end{center}
\caption{The running time of our method on protein design problems from \citet{sontag_lp}.}
\label{fig:protein}
\end{figure}


\subsection{Experiments with Data-Independent Factors}
\label{sec:latent_exp}

In each of the following experiments we perform belief-propagation in models of the form given in \eq{eq:pairwise}. Thus each model is completely specified by defining the node potentials $\Phi_i(x_i|y_i)$, the edge potentials $\Phi_{i,j}(x_i,x_j)$, and the topology $(\mathcal N, \mathcal E)$ of the graph.

Furthermore we assume that the edge potentials are \emph{homogeneous}, i.e., that the potential for each edge is the same, or rather that they have the same order statistics (for example, they may differ by a multiplicative constant). This means that sorting can be done \emph{online} without affecting the asymptotic complexity. When subject to heterogeneous potentials we need merely sort them \emph{offline}; the online cost shall be similar to what we report here.


\subsubsection{Chain-Structured Models}
\label{sec:chain}

In this section, we consider \emph{chain-structured} graphs. Here we have nodes $\mathcal N = \lbrace 1 \ldots Q \rbrace$, and edges $\mathcal E = \lbrace (1,2), (2,3) \ldots (Q-1,Q) \rbrace$. The max-product algorithm is known to compute the maximum-likelihood solution exactly for tree-structured models.


Figure \ref{fig:chain} (left) shows the performance of our method on a model with \emph{random} potentials, i.e., $\Phi_i(x_i|y_i) = U[0,1)$, $\Phi_{i,i+1}(x_i,x_{i+1}) = U[0,1)$, where $U[0,1)$ is the uniform distribution. Fitted curves are superimposed onto the running time, confirming that the performance of the standard solution grows quadratically with the number of states, while ours grows at a rate of $N\sqrt{N}$. The residual error $r$ shows how closely the fitted curve approximates the running time; in the case of random potentials, both curves have almost the same constant.

Figure \ref{fig:chain} (right) shows the performance of our method on the text-denoising experiment. This experiment is essentially identical to that shown in Section \ref{sec:text_denoise}, except that the model is a chain (i.e., there is no $\Phi_{i,i+2}$), and we exploit the notion of data-independence (i.e., the fact that $\Phi_{i,i+1}$ does not depend on the observation). Since the same $\Phi_{i,i+1}$ is used for every adjacent pair of nodes, there is no need to perform the `sorting' step offline -- only a single copy of $\Phi_{i,i+1}$ needs to be sorted, and this is included in the total running time shown in Figure \ref{fig:chain}.



\begin{figure}[t]
\begin{center}
 \includegraphics[width=0.5\textwidth]{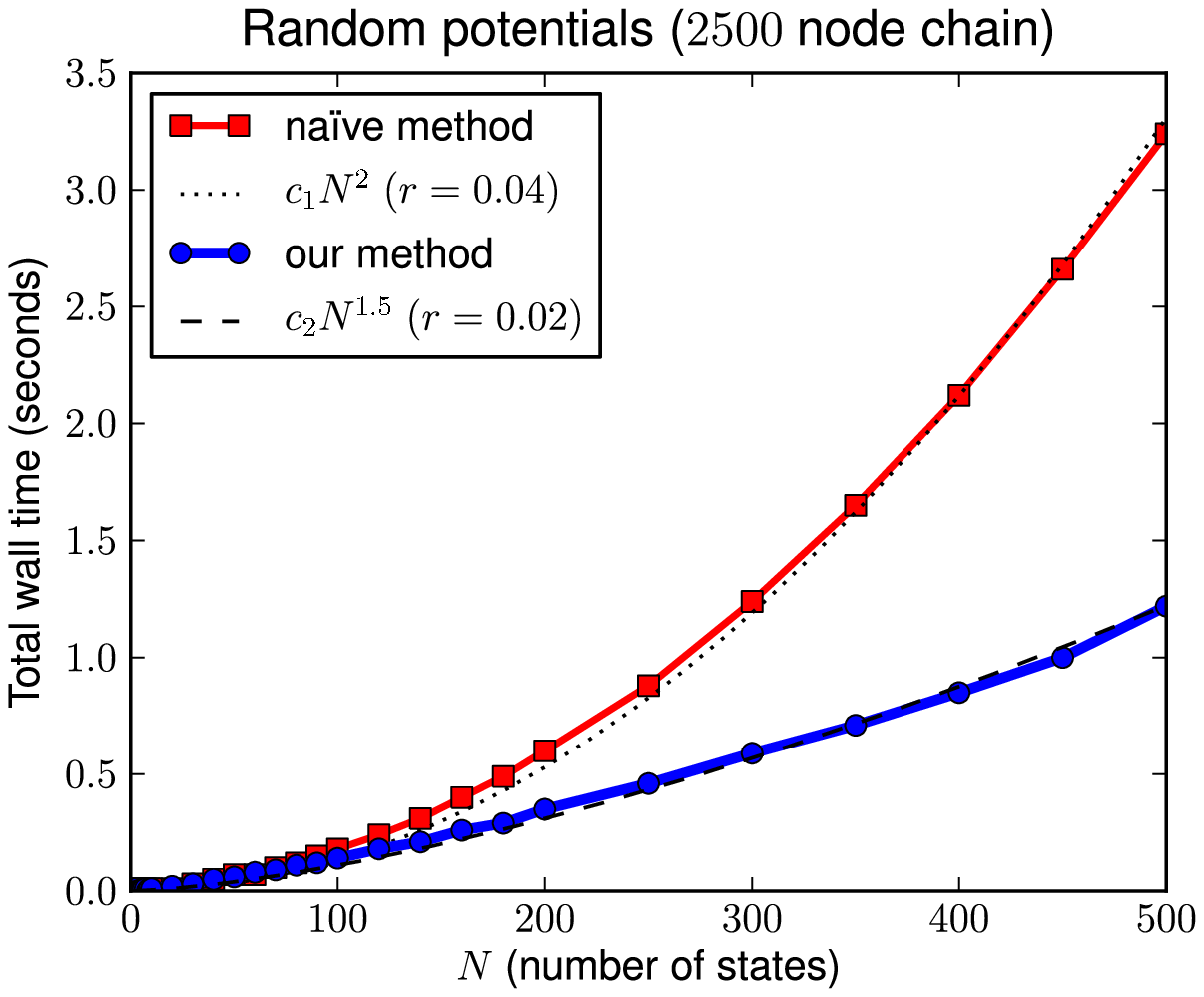}\includegraphics[width=0.5\textwidth]{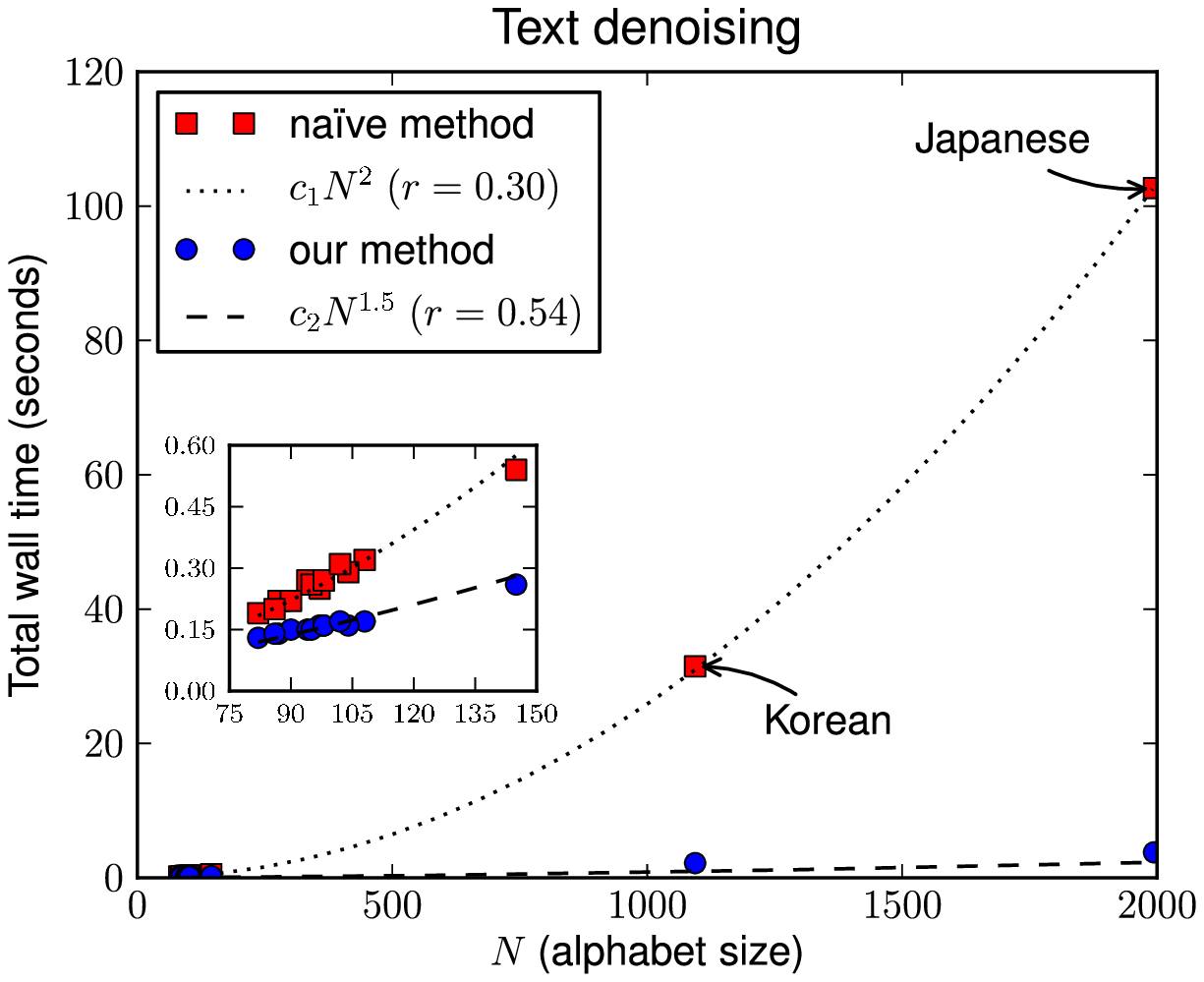}
\end{center}
\caption{Running time of inference in chain-structured models: random potentials (left), and text denoising (right). Fitted curves confirm that the exponent of our method is indeed $1.5$ ($r$ denotes the sum of residuals, i.e., the `goodness' of the fitted curve).}
\label{fig:chain}
\end{figure}

\subsubsection{Grid-Structured Models}
\label{sec:grid}

Similarly, we can apply our method to \emph{grid-structured} models. Here we resort to loopy belief-propagation to approximate the MAP solution, though indeed the same analysis applies in the case of factor graphs \citep{factorgraphs}.
We construct a $50\times 50$ grid model and perform loopy belief-propagation using a random message-passing schedule for five iterations. In these experiments our nodes are $\mathcal N = \lbrace 1 \ldots 50 \rbrace^2$, and our edges connect the 4-neighbors, i.e., the node $(i,j)$ is connected to both $(i+1,j)$ and $(i,j+1)$ (similar to the grid shown in Figure \ref{fig:examps}(a)).

Figure \ref{fig:grid} (left) shows the performance of our method on a grid with random potentials (similar to the experiment in Section \ref{sec:chain}). Figure \ref{fig:grid} (right) shows the performance of our method on an optical flow task \citep{Lucas81}. Here the states encode \emph{flow vectors}: for a node with $N$ states, the flow vector is assumed to take integer coordinates in the square $[-\sqrt{N}/2, \sqrt{N}/2)^2$ (so that there are $N$ possible flow vectors). For the unary potential we have
\begin{equation}
 \Phi_{(i,j)}(x | y) = \bigl\| \mathit{Im}_1[i,j] - \mathit{Im}_2[(i,j) + f(x)] \bigr\|,
\end{equation}
where $\mathit{Im}_1[a,b]$ and $\mathit{Im}_2[a,b]$ return the gray-level of the pixel at $(a,b)$ in the first and second images (respectively), and $f(x)$ returns the flow vector encoded by $x$. The pairwise potentials
simply encode the Euclidean distance between two flow vectors. 
Note that a variety of low-level computer vision tasks (including optical flow) are studied in \citet{pedro_bp}, where the highly structured nature of the potentials in question often allows for efficient solutions.

Our fitted curves in Figure \ref{fig:grid} show $O(N\sqrt{N})$ performance for both random data and for optical flow.

\begin{figure}[t]
\begin{center}
 \includegraphics[width=0.5\textwidth]{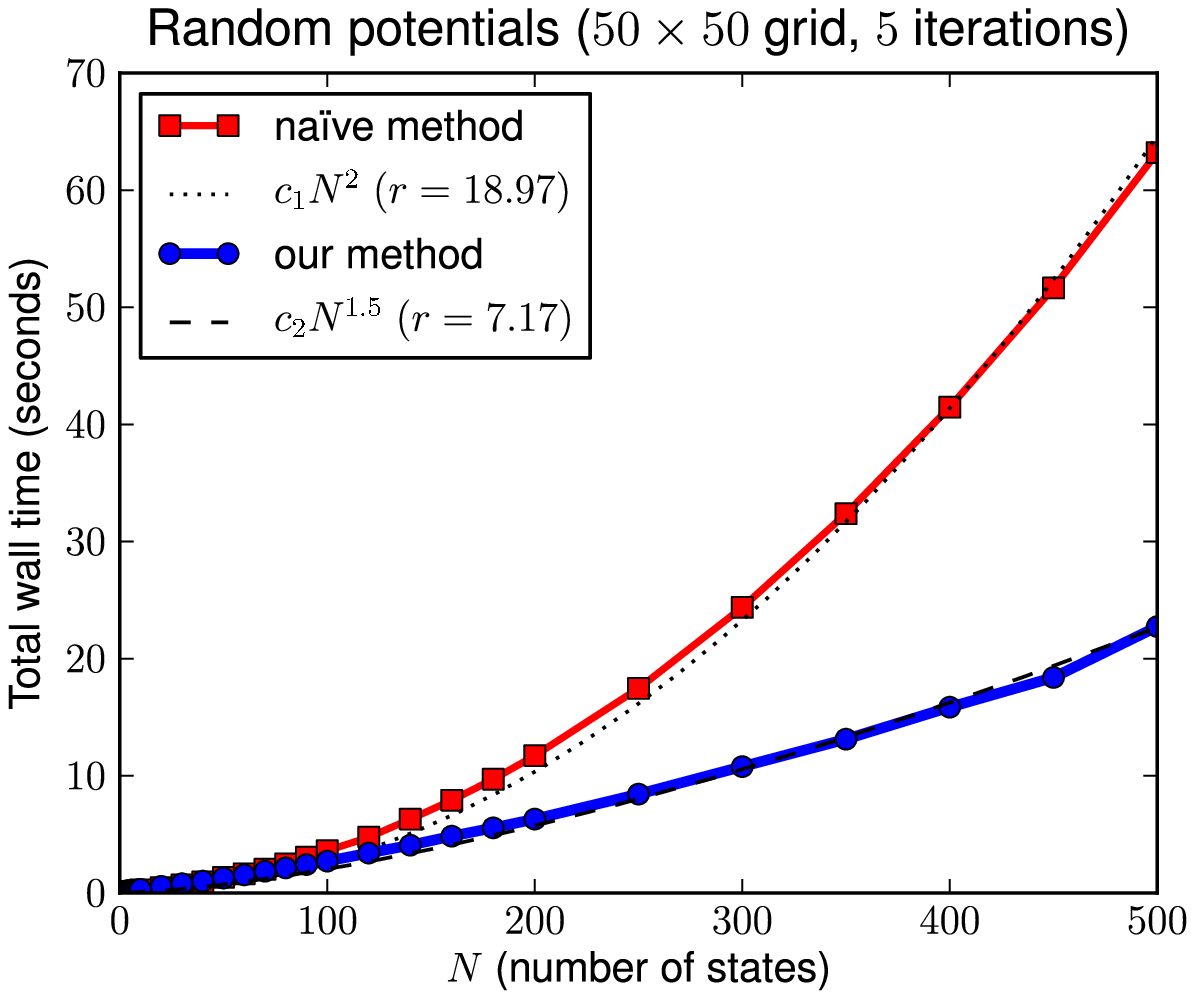}\includegraphics[width=0.5\textwidth]{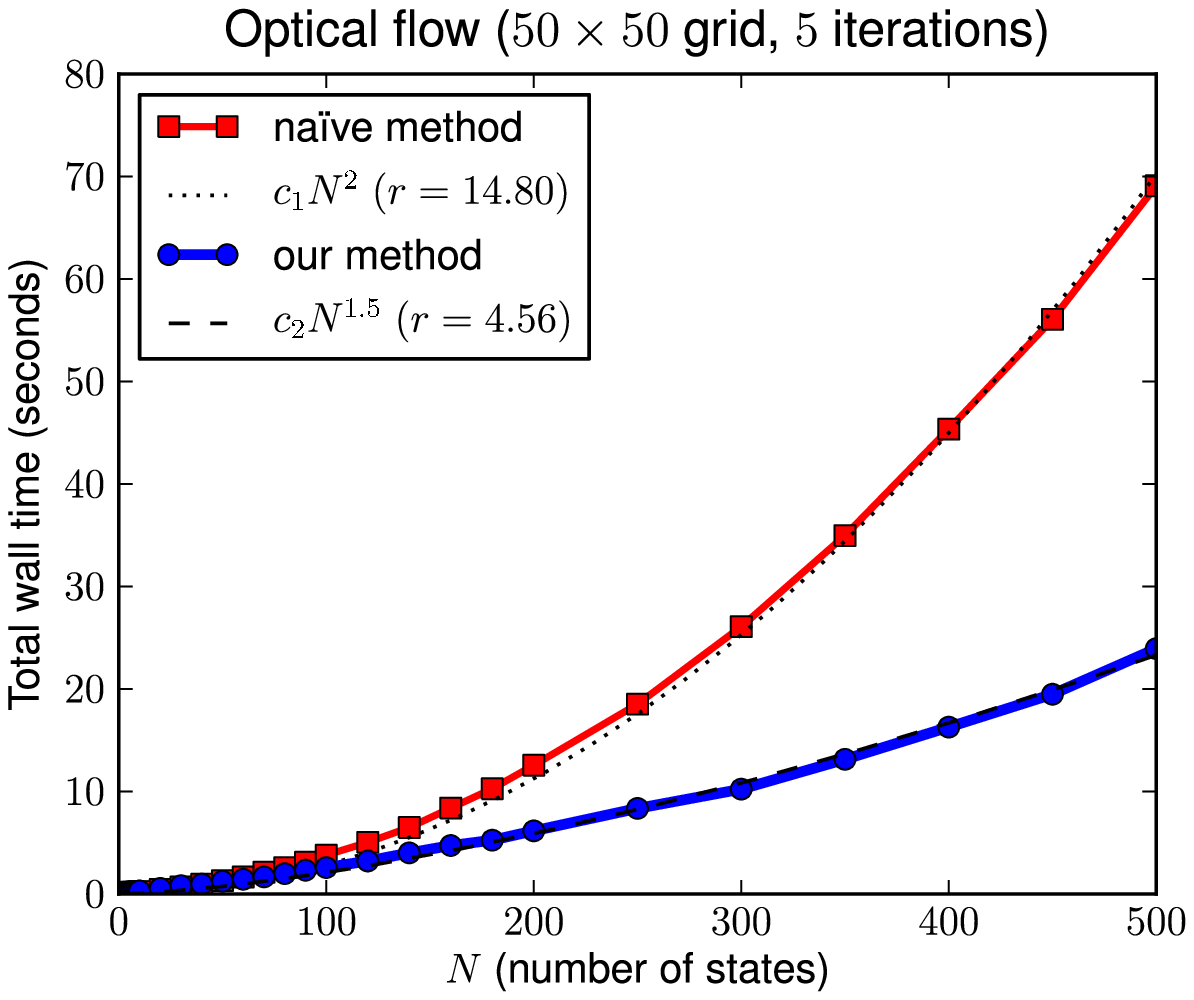}
\end{center}
\caption{Running time of inference in grid-structured models: random potentials (left), and optical flow (right).}
\label{fig:grid}
\end{figure}

\subsubsection{Failure Cases}
\label{sec:fail}

In our previous experiments on graph-matching, text denoising, and optical flow we observed running times similar to those for random potentials, indicating that there is no prevalent dependence structure between the order statistics of the messages and the potentials.

In certain applications the order statistics of these terms are highly dependent. The most straightforward example is that of \emph{concave} potentials (or \emph{convex} potentials in a min-sum formulation). For instance, in a stereo disparity experiment, the unary potentials encode the fact that the output should be `close to' a certain value; the pairwise potentials encode the fact that neighboring nodes should take similar values \citep{Scharstein01ataxonomy,stereo03}.

Whenever both $\mathbf{v}_a$ and $\mathbf{v}_b$ are concave in \eq{eq:hati}, the permutation matrix that transforms the sorted values of $\mathbf{v}_a$ to the sorted values of $\mathbf{v}_b$ is block-off-diagonal (see the sixth permutation in Figure \ref{fig:permutations}).
In such cases, our algorithm only decreases the number of multiplication operations by a multiplicative constant, and may in fact be slower due to its computational overhead. This is precisely the behavior shown in Figure \ref{fig:failure} (left), in the case of stereo disparity.

It should be noted that there exist algorithms specifically designed for this class of potential functions \citep{KolmConvex,pedro_bp}, which are preferable in such instances.

We similarly perform an experiment on image denoising, where the unary potentials are again convex functions of the input \citet[see][]{Geman84,lanroth06}. Instead of using a pairwise potential that merely encodes smoothness, we extract the pairwise statistics from image data (similar to our experiment on text denoising); thus the potentials are no longer concave. We see in Figure \ref{fig:failure} (right) that even if a small number of entries exhibit some `randomness' in their order statistics, we begin to gain a modest speed improvement over the na\"ive solution (though indeed, the improvements are negligible compared to those shown in previous experiments).

\begin{figure}[t]
\begin{center}
 \includegraphics[width=0.5\columnwidth]{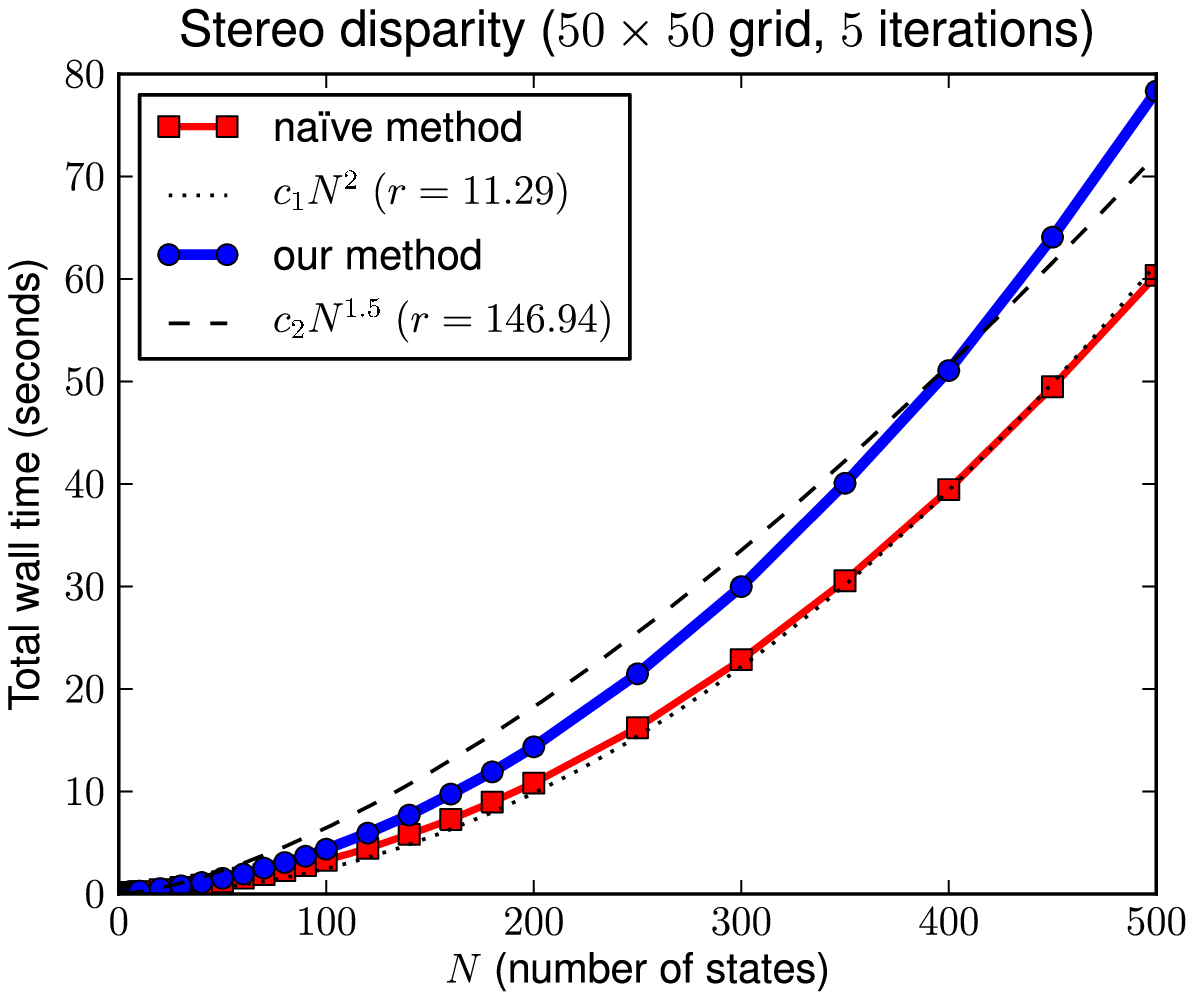}\includegraphics[width=0.5\columnwidth]{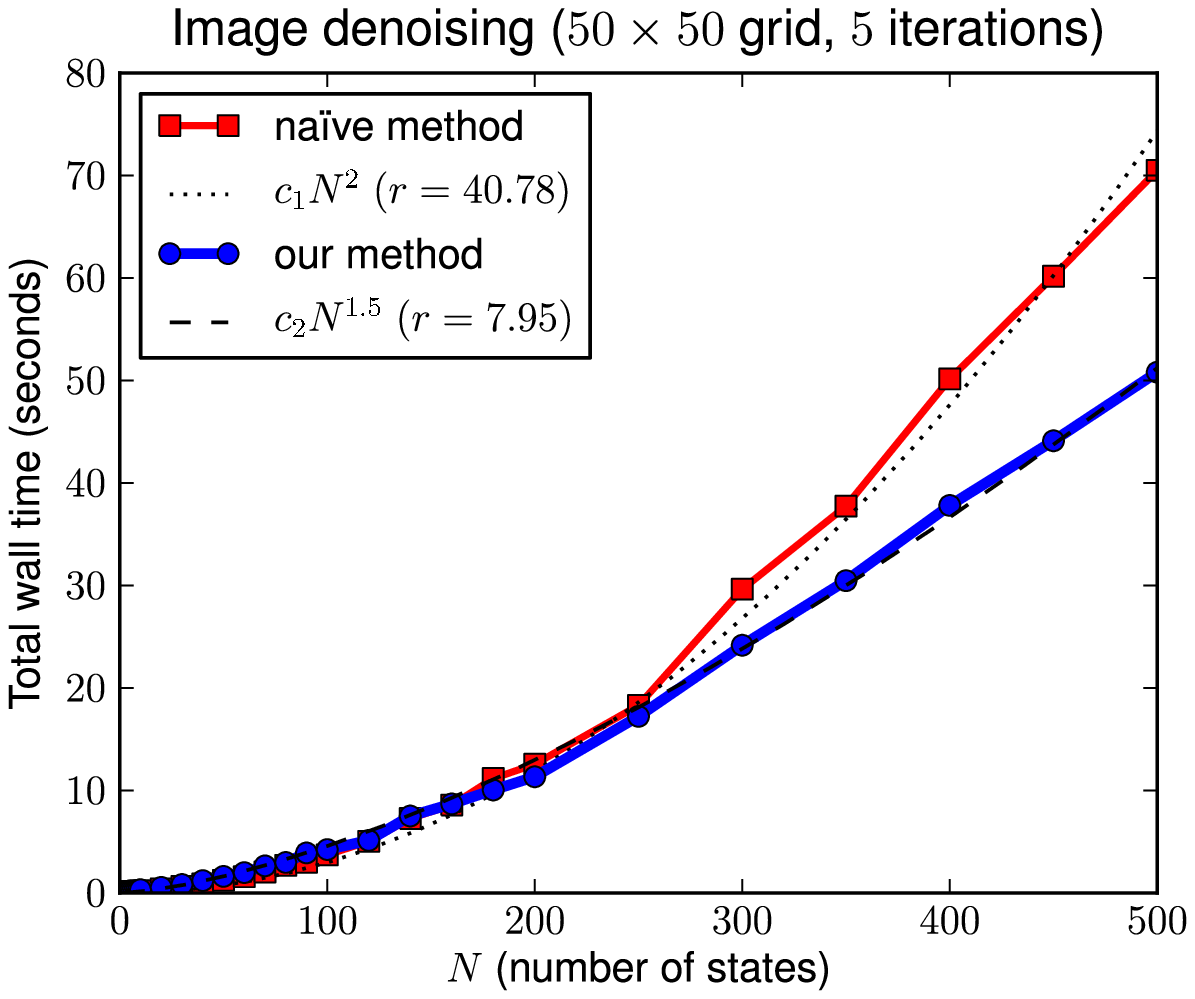}
\end{center}
\caption{Two experiments whose potentials and messages have highly dependent order statistics: stereo disparity (left), and image denoising (right).}
\label{fig:failure}
\end{figure}

\section{Discussion and Future Work}

As we touched upon briefly in Section \ref{sec:background}, there are a variety of applications of our algorithm beyond graphical models -- what we have in fact presented is a solution to the problem of funny matrix multiplication, which generalizes to matrices of arbitrary dimension. For instance, in \citet{aho83} a transformation is given between funny matrix multiplication and all-pairs shortest path, meaning that our algorithm results in a sub-cubic solution to this problem. While the fastest known solution \citep[due to][]{karger} has running time $O(N^2\log N)$ (subject to certain assumptions on the input graph), its implementation requires a Fibonacci heap, meaning that our algorithm proves to be faster for reasonable values of $N$.

It is interesting to consider the fact that our algorithm's running time is purely a function of the input data's \emph{order statistics}, and in fact does not depend on the \emph{data itself}. While it is pleasing that our assumption of independent order statistics appears to be a weak one, and is satisfied in a wide variety of applications, it ignores the fact that stronger assumptions may be reasonable in many cases. In factors with a high dynamic range, or when different factors have different scales, it may be possible to identify the maximum value very quickly, as we attempted to do in Section \ref{sec:protein}. Deriving faster algorithms that make stronger assumptions about the input data remains a promising avenue for future work.

Our algorithm may also lead to faster solutions for \emph{approximate} inference in graphical models. While the stopping criterion of our algorithm \emph{guarantees} that the maximum value is found, it is possible to terminate the algorithm earlier and state that the maximum has \emph{probably} been found. A direction for future work would be to adapt our algorithm to determine the probability that the maximum has been found after a certain number of steps; we could then allow the user to specify an error probability, or a desired running time, and our algorithm could be adapted accordingly.

\section{Conclusion}

We have presented a series of approaches that allow us to improve the performance of exact and approximate max-product message-passing for models with factors smaller than their maximal cliques, and more generally, for models whose factors \emph{that depend upon the observation} contain fewer latent variables than their maximal cliques. We are \emph{always} able to improve the expected computational complexity in any model that exhibits this type of factorization, no matter the size or number of factors. Our improvements increase the class of problems for which inference via max-product belief-propagation is a tractable option.

\subsection*{Acknowledgements}

We would like to thank Pedro Felzenszwalb, Johnicholas Hines, and David Sontag for comments on initial versions of this paper. NICTA is funded by the Australian Government's \emph{Backing Australia's Ability} initiative, and the Australian Research Council's \emph{ICT Centre of Excellence} program.


\appendix
\section{Asymptotic Performance of Algorithm \ref{alg1} and Extensions}
\label{sec:analysis}

In this section we shall determine the expected case running times of Algorithm \ref{alg1} and Algorithm \ref{alg:ext}.
Algorithm \ref{alg1} traverses $\mathbf{v}_a$ and $\mathbf{v}_b$ until it reaches the smallest value of $m$ for which there is some $j \leq m$ for which $m \geq p_b^{-1}[p_a[j]]$. If $M$ is a random variable representing this smallest value of $m$, then we wish to find $E(M)$. While $E(M)$ is the number of `steps' the algorithms take, each step takes $\Theta(K)$ when we have $K$ lists. Thus the expected running time is $\Theta(KE(M))$.

To aid understanding our algorithm, we show the elements being read for specific examples of $\mathbf{v}_a$ and $\mathbf{v}_b$ in Figure \ref{fig:alg1_aistats}. This figure reveals that the actual \emph{values} in $\mathbf{v}_a$ and $\mathbf{v}_b$ are unimportant, and it is only the order-statistics of the two lists that determine the performance of our algorithm. By representing a permutation of the digits $1$ to $N$ as shown in Figure \ref{fig:perms} ((a), (b), and (d)), we observe that $m$ is simply the width of the smallest square (expanding from the top left) that includes an element of the permutation (i.e., it includes $i$ and $p[i]$). 


\begin{figure*}[ht]
\footnotesize
\begin{center}
\begin{tabular}{cc}
\parbox[c]{32.9pt}{\centering\begin{tabular}{c}
                                     \ \\ \includegraphics[scale=0.32]{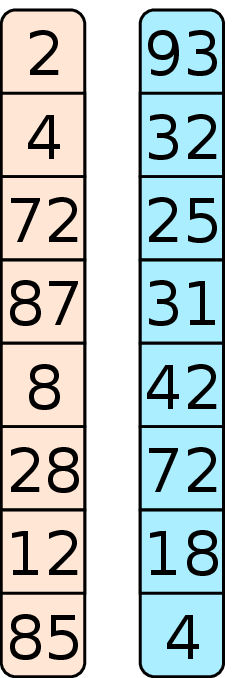}
                                    \end{tabular}
} &
\parbox[c]{326.9pt}
{\centering\begin{tabular}{cccc}
$\mathit{start} = 1$ & $\mathit{start} = 2$ & $\mathit{start} = 3$ & $\mathit{start} = 4$\\
\includegraphics[scale=0.32]{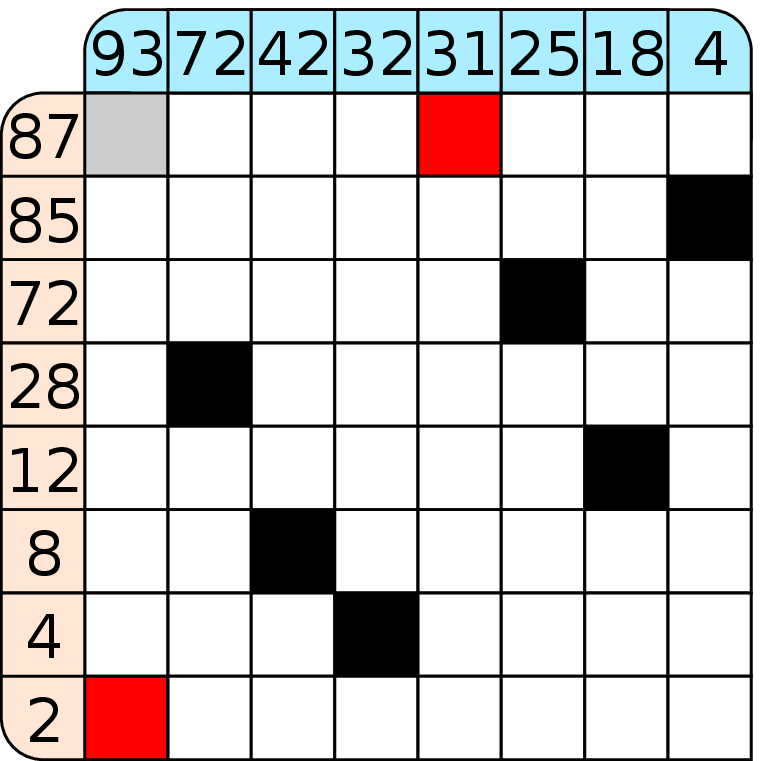} & \includegraphics[scale=0.32]{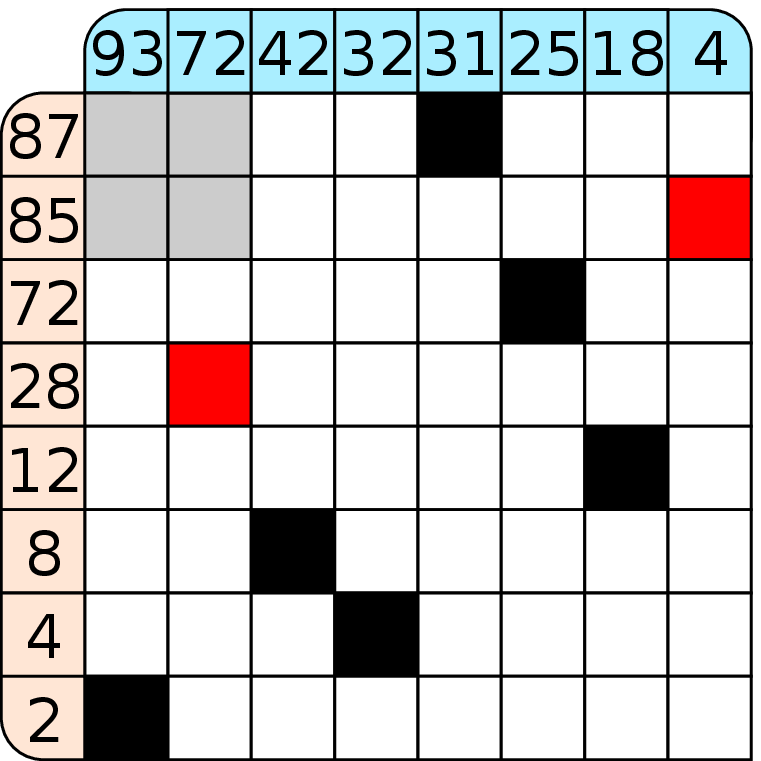} & \includegraphics[scale=0.32]{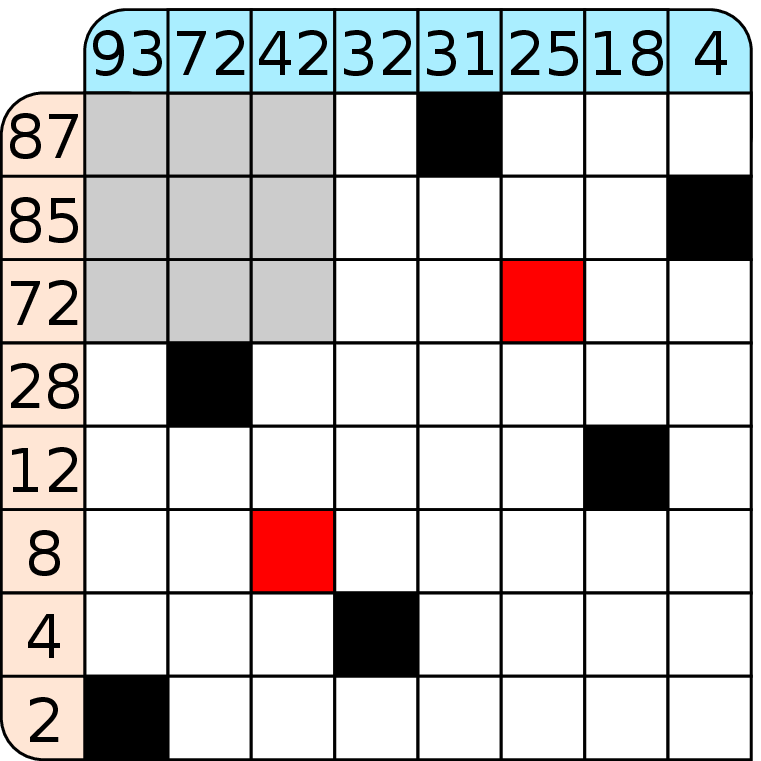} & \includegraphics[scale=0.32]{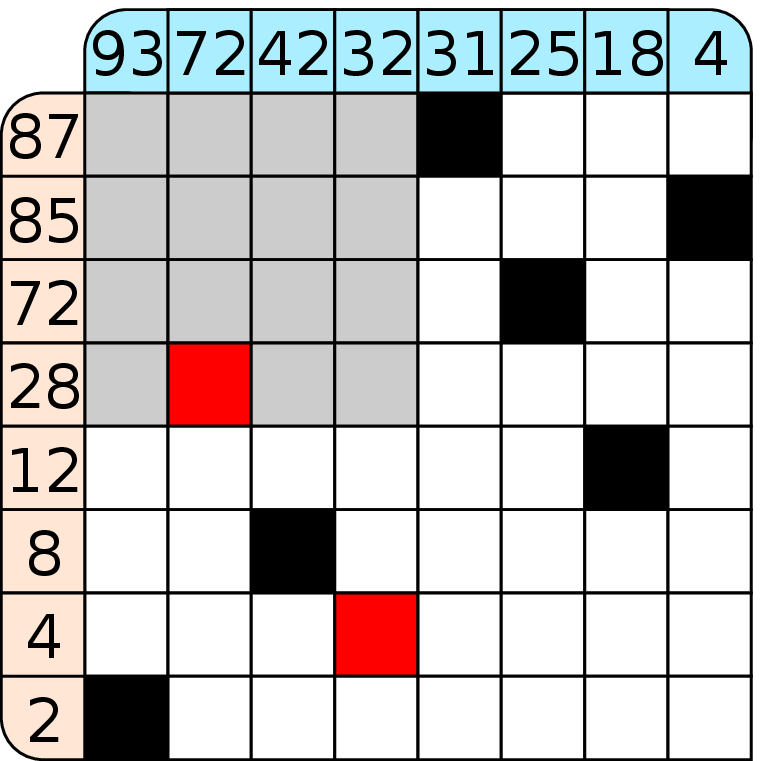}
\end{tabular}} \\ (a) & (b)
\end{tabular}
\end{center}
\vspace{-3mm}
 \caption{(a) The lists $\mathbf{v}_a$ and $\mathbf{v}_b$ before sorting; (b) Black squares show corresponding elements in the sorted lists ($\mathbf{v}_a[p_a[i]]$ and $\mathbf{v}_b[p_b[i]]$); red squares indicate the elements read during each step of the algorithm ($\mathbf{v}_a[p_a[\mathit{start}]]$ and $\mathbf{v}_b[p_b[\mathit{start}]]$). We can imagine expanding a gray box of size $\mathit{start}\times\mathit{start}$ until it contains an entry; note that the maximum is found during the first step.}
\label{fig:alg1_aistats}
\end{figure*}

\begin{figure}
\footnotesize
 \begin{center}
 \begin{tabular}{cccc}
 \includegraphics[scale=0.3]{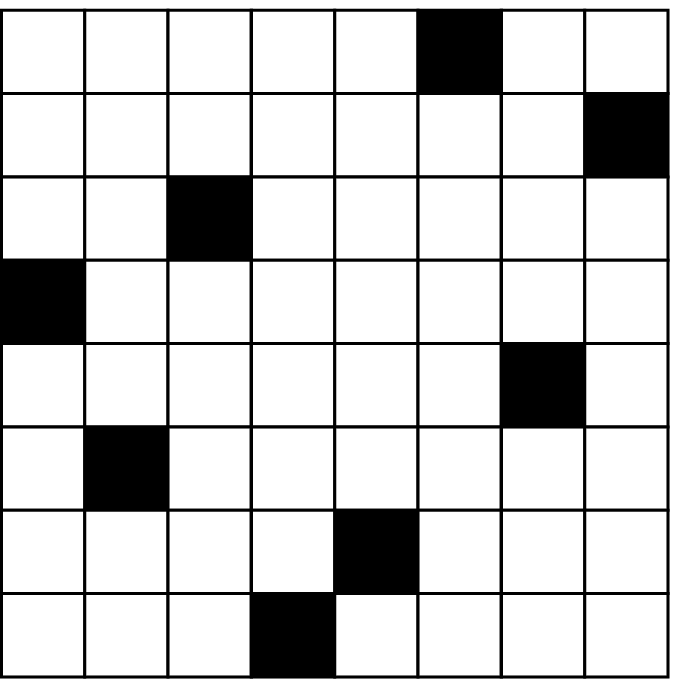} & \includegraphics[scale=0.3]{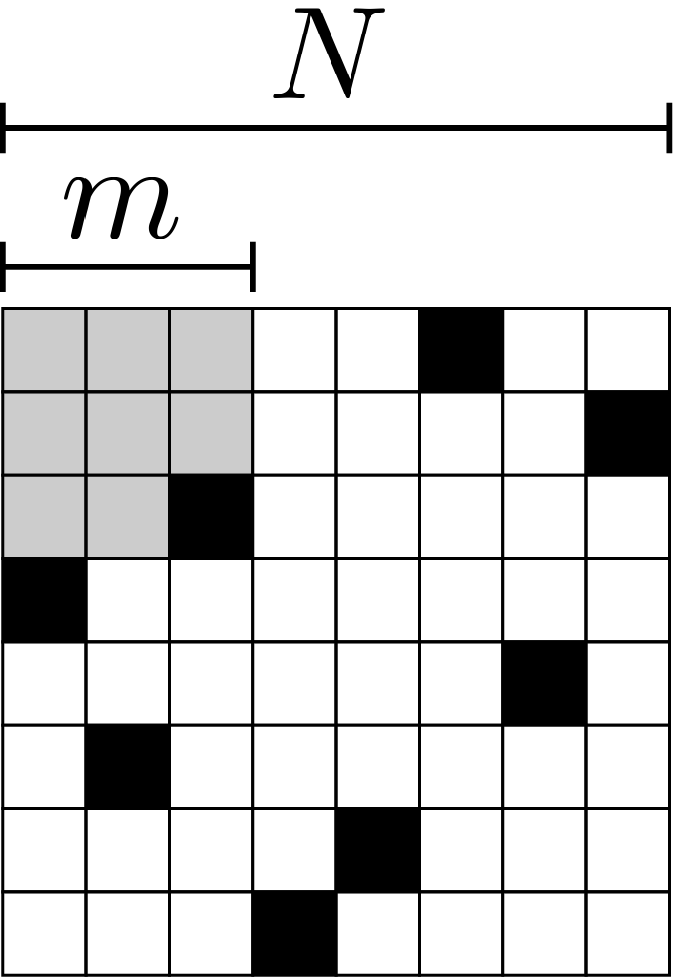} &
 \includegraphics[scale=0.3]{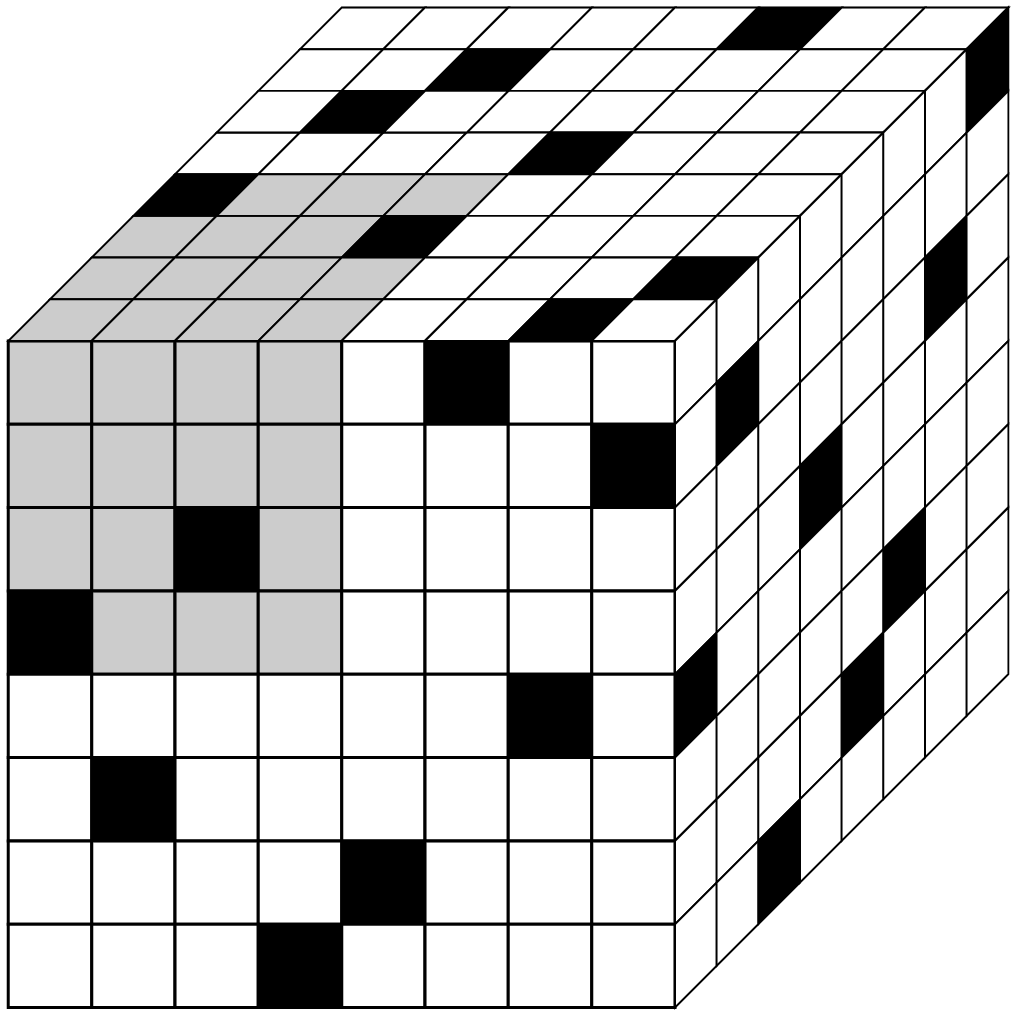} & \includegraphics[scale=0.3]{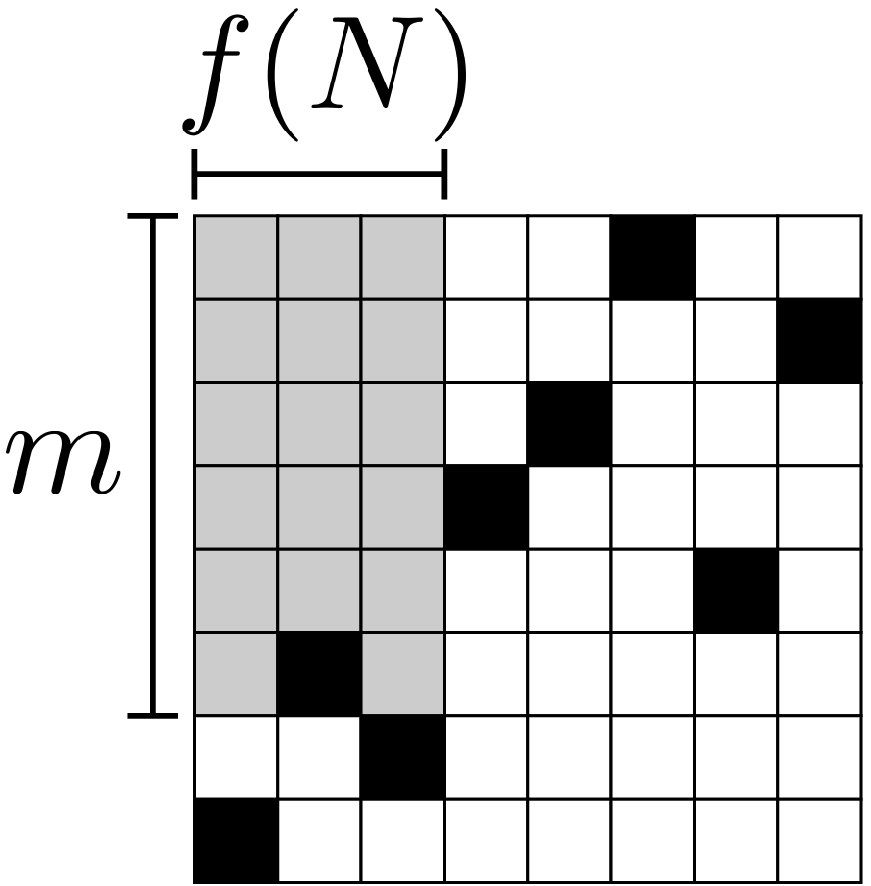}\\
(a) & (b) & (c) & (d)
 \end{tabular}
 \end{center}
\caption{(a) As noted in Figure \ref{fig:alg1_aistats}, a permutation can be represented as an array, where there is exactly one non-zero entry in each row and column; (b) We want to find the smallest value of $m$ such that the grey box includes a non-zero entry; (c) A \emph{pair} of permutations can be thought of as a cube, where every two-dimensional plane contains exactly one non-zero entry; we are now searching for the smallest grey cube that includes a non-zero entry; the faces show the projections of the points onto the exterior of the cube (the third face is determined by the first two); (d) For the sake of establishing an upper-bound, we consider a shaded region of width $f(N)$ and height $m$.}
\label{fig:perms}
\end{figure}

Simple analysis reveals that the probability of choosing a permutation that does not contain a value inside a square of size $m$ is
\begin{equation}
 P(M>m) = \frac{(N-m)!(N-m)!}{(N-2m)!N!}.
\label{eq:factprob}
\end{equation}
This is precisely $1 - F(m)$, where $F(m)$ is the cumulative density function of $M$. It is immediately clear that $1 \leq M \leq \lfloor N/2 \rfloor$, which defines the best and worst-case performance of Algorithm \ref{alg1}.

Using the identity $E(X) = \sum_{x=1}^\infty P(X \geq x)$, we can write down a formula for the expected value of $M$:
\begin{equation}
 E(M) = \sum_{m=0}^{\lfloor N/2 \rfloor} \frac{(N-m)!(N-m)!}{(N-2m)!N!}.
\label{eq:runtimek1}
\end{equation}

The case where we are sampling from multiple permutations simultaneously (i.e., Algorithm \ref{alg:ext}) is analogous. We consider $K-1$ permutations embedded in a $K$-dimensional hypercube, and we wish to find the width of the smallest shaded hypercube that includes exactly one element of the permutations (i.e., $i, p_1[i], \ldots, p_{K-1}[i]$). This is represented in Figure \ref{fig:perms}(c) for $K=3$. Note carefully that $K$ is the number of \emph{lists} in \eq{eq:hatK}; if we have $K$ lists, we require $K-1$ permutations to define a correspondence between them.

Unfortunately, the probability that there is no non-zero entry in a cube of size $m^K$ is not trivial to compute. It is possible to write down an expression that generalizes \eq{eq:factprob}, such as
\begin{equation}
 P^K(M > m) = \frac{1}{N!^{K-1}} \times \sum_{\sigma_1\in S_N}\!\!\!\cdots\!\!\!\sum_{\sigma_{K-1}\in S_N} \bigwedge_{i = 1}^m \left( \max_{k\in\lbrace 1 \ldots K-1 \rbrace} \sigma_k(i) > m \right)
\label{eq:PKm}
\end{equation}
(in which we simply enumerate over all possible permutations and `count' which of them do not fall within a hypercube of size $m^{K}$), and therefore state that
\begin{equation}
 E^K(M) = \sum_{m=0}^{\infty} P^K(M>m).
\label{eq:runtimeexact}
\end{equation}
However, it is very hard to draw any conclusions from \eq{eq:PKm}, and in fact it is intractable even to evaluate it for large values of $N$ and $K$. Hence we shall instead focus our attention on finding an upper-bound on \eq{eq:runtimeexact}. Finding more computationally convenient expressions for \eq{eq:PKm} and \eq{eq:runtimeexact} remains as future work.

\subsection{An Upper-Bound on $E^K(M)$}

Although \eq{eq:runtimek1} and \eq{eq:runtimeexact} precisely define the running times of Algorithm \ref{alg1} and Algorithm \ref{alg:ext}, it is not easy to ascertain the speed improvements they achieve, as the values to which the summations converge for large $N$ are not obvious. Here, we shall try to obtain an upper-bound on their performance, which we assessed experimentally in Section \ref{sec:experiments}. In doing so we shall prove Theorems \ref{the:alg1} and \ref{the:algext}.



\begin{proof}[Proof of Theorem \ref{the:alg1}] (see Algorithm \ref{alg1})
Consider the shaded region in Figure \ref{fig:perms}(d). This region has a width of $f(N)$, and its height $m$ is chosen such that it contains precisely one non-zero entry. Let $\dot{M}$ be a random variable representing the height of the grey region needed in order to include a non-zero entry. We note that
\begin{equation}
E(\dot{M}) \in O(f(N)) ~\Rightarrow~ E(M) \in O(f(N));
\end{equation}
our aim is to find the smallest $f(N)$ such that $E(\dot{M}) \in O(f(N))$. The probability that none of the first $m$ samples appear in the shaded region is
\begin{equation}
 P(\dot{M} > m) = \prod_{i=0}^m \left( 1 - \frac{f(N)}{N - i}\right).
\label{eq:noreplace}
\end{equation}
Next we observe that if the entries in our $N\times N$ grid do not define a permutation, but we instead choose a \emph{random} entry in each row, then the probability (now for $\ddot{M}$) becomes
\begin{equation}
  P(\ddot{M} > m) = \left(1 - \frac{f(N)}{N}\right)^m
\label{eq:replace}
 \end{equation}
(for simplicity we allow $m$ to take arbitrarily large values). We certainly have that $P(\ddot{M}>m) \geq P(\dot{M}>m)$, meaning that $E(\ddot{M})$ is an upper-bound on $E(\dot{M})$, and therefore on $E(M)$. Thus we compute the expected value
\begin{equation}
 E(\ddot{M}) = \sum_{m=0}^\infty \left(1 - \frac{f(N)}{N}\right)^m.
\end{equation}
This is just a geometric progression, which sums to ${N}/{f(N)}$. Thus we need to find $f(N)$ such that
\begin{equation}
 f(N) \in O\left(\frac{N}{f(N)}\right).
\end{equation}
Clearly $f(N) \in O(\sqrt{N})$ will do. Thus we conclude that
\begin{equation}
 E(M) \in O(\sqrt{N}).
\end{equation}
\end{proof}

\begin{proof}[Proof of Theorem \ref{the:algext}] (see Algorithm \ref{alg:ext})
We would like to apply the same reasoning in the case of multiple permutations in order to compute a bound on $E^K(M)$. That is, we would like to consider $K-1$ \emph{random} samples of the digits from $1$ to $N$, rather than $K-1$ permutations, as random samples are easier to work with in practice.

To do so, we begin with some simple corollaries regarding our previous results. We have shown that in a permutation of length $N$, we expect to see a value less than or equal to $f$ after $N/f$ steps. There are now $f-1$ other values that are less than or equal to $f$ amongst the remaining $N - N/f$ values; we note that
\begin{equation}
 \frac{f-1}{N - \frac{N}{f}} = \frac{f}{N}.
\end{equation}
Hence we expect to see the \emph{next} value less than or equal to $f$ in the next $N/f$ steps also. A consequence of this fact is that we not only expect to see the \emph{first} value less than or equal to $f$ earlier in a permutation than in a random sample, but that when we sample $m$ elements, we expect \emph{more} of them to be less than or equal to $f$ in a permutation than in a random sample.

Furthermore, when considering the \emph{maximum} of $K-1$ permutations, we expect the first $m$ elements to contain more values less than or equal to $f$ than the maximum of $K-1$ random samples. \eq{eq:PKm} is concerned with precisely this problem. Therefore, when working in a $K$-dimensional hypercube, we can consider $K-1$ random samples rather than $K-1$ permutations in order to obtain an upper-bound on \eq{eq:runtimeexact}.

Thus we define $\ddot{M}$ as in \eq{eq:replace}, and conclude that
\begin{equation}
 P(\ddot{M} > m) = \left(1 - \frac{f(N,K)^{K-1}}{N^{K-1}}\right)^m.
\end{equation}
Thus the expected value of $\ddot{M}$ is again a geometric progression, which this time sums to $\left({N}/{f(N,K)}\right)^{K-1}$. Thus we need to find $f(N,K)$ such that
\begin{equation}
 f(N,K) \in O\left( \left( \frac{N}{f(N,K)} \right)^{K-1} \right).
\end{equation}
Clearly
\begin{equation}
 f(N,K) \in O\left(N^{\frac{K-1}{K}}\right)
\end{equation}
will do. As mentioned, each step takes $\Theta(K)$, so the final running time is $O(KN^\frac{K-1}{K})$.
\end{proof}

To summarize, for problems decomposable into $K+1$ groups, we will need to find the index that chooses the maximal product amongst $K$ lists; we have shown an upper-bound on the expected number of steps this takes, namely
\begin{equation}
 E^K(M) \in O\left(N^{\frac{K-1}{K}}\right).
\label{eq:bound}
\end{equation}

\end{document}